\documentclass[journal]{IEEEtran}

\usepackage{amsmath,amsfonts,amssymb,mathtools}
\usepackage{amsthm}
\usepackage{algorithm}
\usepackage{algorithmic}
\usepackage{array}
\usepackage[caption=false,font=normalsize,labelfont=sf,textfont=sf]{subfig}
\usepackage{booktabs}
\usepackage{multirow}
\usepackage[table]{xcolor}
\usepackage{colortbl}
\usepackage{makecell}
\usepackage{tabularx}
\usepackage{adjustbox}
\usepackage{graphicx}
\usepackage[percent]{overpic}
\usepackage{pifont}
\usepackage{enumitem}
\usepackage{textcomp}
\usepackage{stfloats}
\usepackage{float}
\usepackage{url}
\usepackage{verbatim}
\usepackage{xspace}
\usepackage{xfp}
\usepackage{siunitx}
\usepackage{tables/tab_style}
\usepackage{microtype}
\usepackage[numbers,sort&compress]{natbib}
\usepackage{bibunits}
\defaultbibliographystyle{IEEEtranN}
\usepackage[hidelinks]{hyperref}
\usepackage{balance}

\newcommand{\method}{\mbox{\normalfont\scshape SACM}\xspace}

\theoremstyle{plain}
\newtheorem{theorem}{Theorem}[section]

\theoremstyle{definition}

\theoremstyle{remark}

\begin{document}

\title{Towards Robust Time Series Learning via Capacity-Centric Modulation}

\author{Siru~Zhong,
        Senzhang~Wang,
        James~T.~Kwok,~\IEEEmembership{Fellow,~IEEE},
        and~Yuxuan~Liang%
\thanks{Siru Zhong and Yuxuan Liang are with Data Science \& Analytics Thrust, The Hong Kong University of Science and Technology (GZ), Guangzhou, China. Senzhang Wang is with School of Computer Science \& Engineering, Central South University, Changsha, China. James T. Kwok is with Dept. of Computer Science \& Engineering, The Hong Kong University of Science and Technology, Hong Kong. Corresponding author: Yuxuan Liang (yuxliang@outlook.com).}}

\makeatletter
\def\leftmark{Preprint}
\def\rightmark{Preprint}
\makeatother

\maketitle

\begin{abstract}
Sample-level reliability heterogeneity is common in deep time series learning. Standard training pipelines apply a uniform regularization setting to all samples, which can under-regularize corrupted samples and over-restrict clean samples. Common robustness approaches filter observations in data space or impose priors on latent representations. We propose Capacity-Centric Modulation (\textbf{CCM}) as a complementary, sample-adaptive regularization principle. Under this principle, we introduce \textbf{\method} (Sample-Adaptive Capacity Modulation), a task-agnostic framework that exploits spectral sparsity to assign sample-wise dropout probabilities along internal activation paths. \method integrates into existing backbones without architectural redesign and preserves the deterministic inference pipeline. Across 301 real-world dataset--backbone pairs covering 9 forecasting, 32 classification, and 4 anomaly-detection datasets, \method reduces forecasting MSE by 6.7\% on average and improves classification accuracy and point-adjusted F1 by 3.04\% and 17.05\%, respectively, relative to unmodified backbones, with zero test-time overhead.
\end{abstract}


\bstctlcite{pageLimitBibStyle}
\section{Introduction}
\label{sec:introduction}

\begin{figure*}[!b]
    \vspace{-1em}
    \centering
    \includegraphics[width=\linewidth]{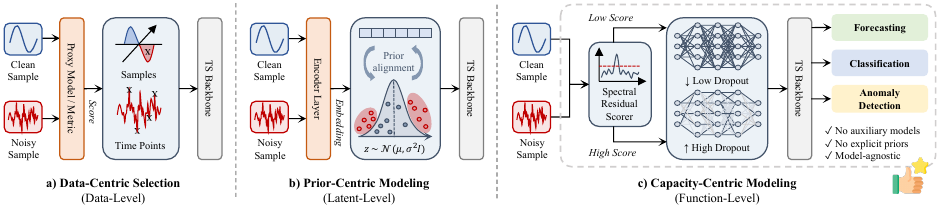}
    \caption{Overview of robustness paradigms and intervention levels. \textbf{a) Data-Centric Selection} removes or reweights observations at the data level; \textbf{b) Prior-Centric Modeling} constrains latent representations; \textbf{c) Capacity-Centric Modulation} modulates effective capacity along internal activation paths for predictive (forecasting), discriminative (classification), and reconstructive (anomaly detection) tasks.}
    \label{fig:motivation}
    \vspace{-1em}
\end{figure*}

\IEEEPARstart{T}{ime} series modeling underpins applications in weather, finance, energy demand, and traffic \cite{li2023longterm,cirstea2022traffic}. Temporal learning spans three primary tasks: \textit{predictive} (forecasting), \textit{discriminative} (classification), and \textit{reconstructive} (anomaly detection). The growing volume and complexity of real-world temporal data have driven rapid innovations in deep temporal architectures, ranging from Recurrent Neural Networks (RNNs) \cite{hochreiter1997long,lai2018modeling} and Convolutional Networks (CNNs) \cite{wu2023timesnet} to Multi-Layer Perceptrons (MLPs) \cite{ekambaram2023tsmixer,wang2024timemixer}, Transformers \cite{zhou2021informer,nie2023patchtst,liu2023itransformer}, and diffusion-based generative models \cite{rasul2021timegrad,tashiro2021csdi}.

Despite these architectural advances, deep time series models face \textit{sample-level reliability heterogeneity} \cite{yamanishi2002unifying,blazquez2021review,yang2025not}. In real-world environments, time series observations often combine diverse signal regimes with unpredictable corruptions. Some samples exhibit clean, stationary periodicity and coherent trends, while others suffer from heavy-tailed measurement noise, transient sensor failures, missing values, or abrupt distributional shifts. Under standard training practices, the model uses a fixed dropout or regularization configuration for every training sample, resulting in the same nominal capacity allocation \cite{hinton2014dropout,wager2013drop,gal2017concrete}. This creates a learning dilemma: a capacity budget expressive enough to capture intricate temporal dynamics in clean samples overfits spurious artifacts in noise-dominated samples, learning incidental fluctuations rather than transferable temporal patterns \cite{cheng2024robusttsf,fu2025selective}.

To alleviate noise overfitting, existing robust time series methods predominantly follow two conventional paradigms:
\begin{itemize}[leftmargin=*]
    \item \textbf{Data-Centric Selection:} Approaches such as RobustTSF \cite{cheng2024robusttsf} and Selective Learning \cite{fu2025selective} frame robustness as a screening problem. They employ auxiliary statistical metrics or proxy networks to identify and exclude corrupted samples or time steps. While effective at filtering outliers, this binary ``keep-or-drop'' screening can cause \textit{information loss}, frequently discarding critical tail events or subtle transitions that mimic noise signatures. Furthermore, such screening mechanisms are typically coupled with task-specific loss heuristics, limiting cross-task applicability.
    \item \textbf{Prior-Centric Modeling:} Approaches such as BayesTSF \cite{Pan2024Bayes} and RSTIB \cite{cheninformation} treat robustness as a representation-disentanglement problem. By incorporating Variational Autoencoders (VAEs) \cite{kingma2013auto} or Bayesian Neural Networks (BNNs), they constrain latent representations to satisfy prior distributions (e.g., Gaussianity). However, real-world temporal dynamics are inherently non-stationary, making \textit{rigid distributional assumptions} brittle while introducing additional inference overhead and optimization complexity.
\end{itemize}

Both paradigms leave sample-adaptive capacity allocation inside the backbone unaddressed. Applying a uniform regularization strength across heterogeneous samples inevitably leads to a capacity mismatch: it under-regularizes corrupted samples (encouraging spurious noise memorization) while over-restricting the expressive fidelity available to clean ones.

\begin{table*}[!t]
    \centering
    \caption{Comparison of robustness paradigms in time series modeling. \method instantiates Capacity-Centric Modulation on the active hypothesis space ($\mathcal{H}$), operating complementarily to conventional data- and prior-centric approaches.}
    \label{tab:robustness_paradigm_comparison}
    \renewcommand{\arraystretch}{1.12}
    \resizebox{\textwidth}{!}{
    \begin{tabular}{l l l l l}
        \toprule
        \rowcolor{TableHeader}
        \textbf{Paradigm} & \textbf{Intervention Space} & \textbf{Core Philosophy} & \textbf{Representative Methods} & \textbf{Key Characteristics \& Limitations} \\
        \midrule

        \textbf{Data-Centric Selection} &
        Data Space ($\mathcal{X}$) &
        \textit{``Filter Noise''} &
        RobustTSF \cite{cheng2024robusttsf} &
        $\bullet$ Discrete keep-or-drop screening \\
        (e.g., Hard Selection) &
        &
        Exclude unreliable observations &
        Selective Learning \cite{fu2025selective} &
        $\bullet$ Risks discarding valid rare events; task-coupled \\
        \midrule

        \textbf{Prior-Centric Modeling} &
        Latent Space ($\mathcal{Z}$) &
        \textit{``Disentangle Noise''} &
        BayesTSF \cite{Pan2024Bayes} &
        $\bullet$ Imposes explicit priors (e.g., Gaussian VAEs/BNNs) \\
        (e.g., Complex Priors) &
        &
        Constrain latent distributions &
        RSTIB \cite{cheninformation} &
        $\bullet$ Sensitive to non-stationarity; heavy inference cost \\
        \midrule

        \rowcolor{TableOurs}
        \textbf{Capacity-Centric Modulation} &
        \textbf{Hypothesis Space ($\mathcal{H}$)} &
        \textbf{\textit{``Coexist with Noise''}} &
        \textbf{\method (Ours)} &
        $\bullet$ \textbf{Sample-adaptive active capacity allocation} \\
        \rowcolor{TableOurs}
        \textbf{(Sample-Adaptive Principle)} &
        &
        \textbf{Match capacity to reliability} &
        &
        $\bullet$ \textbf{Task-agnostic, zero inference overhead, complementary} \\
        \bottomrule
    \end{tabular}}
    \vspace{-1em}
\end{table*}

To address this, we propose \underline{C}apacity-\underline{C}entric \underline{M}odulation (\textbf{CCM}) as a complementary principle for robust time series learning (Figure~\ref{fig:motivation} and Table~\ref{tab:robustness_paradigm_comparison}). Instead of screening observations in the \textit{data space} ($\mathcal{X}$) or enforcing rigid probabilistic priors in the \textit{latent space} ($\mathcal{Z}$), CCM operates on the model's active \textit{hypothesis space} ($\mathcal{H}$) through sample-adaptive regularization.
Along internal activation pathways, CCM scales stochastic regularization conditioned on sample reliability: unreliable samples face stronger pathway dropping to suppress noise memorization, whereas clean samples retain more active capacity to capture intricate temporal dynamics.

Under the CCM principle, we introduce \textbf{\method} (\underline{S}ample-\underline{A}daptive \underline{C}apacity \underline{M}odulation), a unified, task- and model-agnostic framework for robust time series learning. To realize sample-adaptive capacity allocation without expensive noise labels or teacher networks, \method leverages \textit{spectral sparsity} as a label-free inductive bias. In the frequency domain, structured temporal signals often concentrate energy into dominant spectral modes, whereas corruptions and perturbations tend to produce less concentrated or more diffuse spectral residuals \cite{donoho2006compressed,ren2019time}. \method features a differentiable spectral residual scorer that detrends the input, applies a Spectral Flatness Measure (SFM)-anchored filter, and computes the reconstruction residual as an unreliability proxy. A learnable rate mapper translates this residual into sample-specific stochastic retention rates along the network's trainable paths. Via the Straight-Through Estimator (STE), \method optimizes capacity allocation jointly with the backbone end-to-end. During evaluation, capacity modulation is bypassed, preserving the deterministic inference pipeline with zero test-time overhead.

\noindent\textbf{Relationship to Prior Conference Work.} Compared to our preliminary conference paper \textsc{DropoutTS}~\cite{zhong2026dropoutts}, which introduced an empirical spectral dropout heuristic solely for time series forecasting, this journal article systematically advances the formulation into a unified, theoretically motivated framework with three principal extensions:
\begin{enumerate}[leftmargin=*]
    \item \textbf{Capacity-Centric Perspective:} We reframe sample-level dropout from a heuristic into \textit{Capacity-Centric Modulation (CCM)}, establishing it as a complementary principle alongside data- and prior-centric modeling (\S\ref{sec:introduction}, \S\ref{sec:motivation}).
    \item \textbf{Principled Framework \& End-to-End Design:} Moving beyond heuristic dropping, we develop a systematic, label-free pipeline featuring an SFM-anchored spectral scorer, a learnable rate mapper, and STE-based joint optimization, supported by a surrogate risk analysis (\S\ref{sec:methodology}, \S\ref{sec:theory}).
    \item \textbf{Cross-Task \& In-Depth Empirical Generalization:} We extend the framework beyond forecasting to classification across 32 benchmarks and anomaly detection across 4 benchmarks, covering 220+ matched dataset--backbone configurations (\S\ref{sec:rq4}), together with zero-shot transfer and controlled ablations (\S\ref{sec:rq3}, \S\ref{sec:rq5}).
\end{enumerate}

\noindent\textbf{Experimental Roadmap.} We organize the empirical study around five questions: robustness to controlled corruption (RQ1), real-world forecasting and sample-wise capacity allocation (RQ2), data regimes and zero-shot transfer (RQ3), cross-task generality in classification and anomaly detection (RQ4), and ablations, efficiency, and complementarity with data selection (RQ5). Section~\ref{sec:experiments} gives the full protocol and reports each question in a dedicated subsection.

In summary, our contributions are summarized as follows:
\begin{itemize}[leftmargin=*]
    \item \textbf{Capacity-Centric Regularization Principle:} We formulate CCM to address the capacity mismatch in uniform regularization via sample-adaptive capacity allocation.
    \item \textbf{Unified Task-Agnostic Framework:} We develop \method, a label-free, spectral-guided instantiation of CCM that adaptively modulates active capacity per sample with training-only overhead and zero inference overhead.
    \item \textbf{Cross-Task Evaluation:} Across 301 dataset--backbone configurations covering forecasting, classification, and anomaly detection, \method yields widespread empirical gains, faster convergence, and complementarity with data-centric methods.
\end{itemize}

\section{Related Work}
\label{sec:related_work}

\noindent\textbf{Deep Architectures for Time Series Modeling.}
Deep learning has advanced time series modeling across multiple tasks by capturing nonlinear temporal dependencies and learning transferable representations \cite{ma2024survey}. Early sequential architectures based on RNNs \cite{lai2018modeling} and LSTMs \cite{hochreiter1997long} suffered from gradient degradation and limited parallelism. Transformer-based architectures overcame these bottlenecks via self-attention and structural decomposition, including Informer \cite{zhou2021informer}, Autoformer \cite{wu2021autoformer}, FEDformer \cite{zhou2022fedformer}, Crossformer \cite{zhang2022crossformer}, PatchTST \cite{nie2023patchtst}, and iTransformer \cite{liu2023itransformer}. In parallel, linear and MLP-based architectures such as DLinear \cite{zeng2023transformers}, TSMixer \cite{ekambaram2023tsmixer}, and TimeMixer \cite{wang2024timemixer} emerged as competitive lightweight alternatives, while diffusion-based and other specialized models have been developed for forecasting and imputation \cite{rasul2021timegrad,tashiro2021csdi,li2023longterm,cirstea2022traffic}. Beyond forecasting, representation learning underpins time series classification \cite{fawaz2020inceptiontime,dempster2021minirocket,zhang2024logora} and unsupervised anomaly detection \cite{hundman2018detecting,su2019robust,xu2021anomaly,tuli2022tranad,yang2023dcdetector,kieu2022robust,wu2023benchmarks,zhou2024labelfree,blazquez2021review}. Despite these innovations, standard training pipelines apply a fixed regularization configuration across all samples, leaving models vulnerable to sample-level reliability heterogeneity.

\noindent\textbf{Robustness Paradigms for Time Series.}
Existing methods addressing time series corruptions predominantly operate in two spaces: \textit{1) Data-Centric Selection:} Operating in data space $\mathcal{X}$, RobustTSF \cite{cheng2024robusttsf} identifies outlier samples (each a complete input window), while Selective Learning \cite{fu2025selective} excludes high-uncertainty time steps, and data-centric studies reweight or clean training data \cite{wu2025dataopt,liang2024svlearner}. Hard screening suppresses noise but can discard valid rare events and tail dynamics \cite{yamanishi2002unifying}. \textit{2) Prior-Centric Modeling:} Operating in latent space $\mathcal{Z}$, methods such as BayesTSF \cite{Pan2024Bayes} and RSTIB \cite{cheninformation} leverage Variational Autoencoders \cite{kingma2013auto} or Bayesian Neural Networks to constrain representations. Rigid distributional priors (e.g., Gaussianity) struggle with non-stationary real-world distributions and incur heavy inference overhead. In parallel, frequency-domain filtering techniques such as FiLM \cite{zhou2022film}, TimeFilter \cite{hutimefilter}, and multi-resolution time-frequency analysis \cite{yan2024multiresolution} design specialized network layers, but they do not regulate effective capacity conditioned on sample-level temporal reliability.{\parfillskip=0pt\par}

\noindent\textbf{Adaptive Regularization and Capacity Control.}
Dropout \cite{hinton2014dropout} is a standard stochastic regularizer, and Bayesian extensions such as Variational Dropout \cite{Kingma2015VariationalDA} and Concrete Dropout \cite{gal2017concrete} enable learnable dropout rates across layers or parameters. These formulations optimize static, dataset-level parameters and remain sample-agnostic, sampling stochastic masks from fixed retention distributions regardless of input unreliability. \method bridges this gap by formalizing Capacity-Centric Modulation, providing an end-to-end regularizer that adapts the model's active capacity to each sample's reliability across diverse time series tasks. Compared to our preliminary conference version \textsc{DropoutTS} \cite{zhong2026dropoutts}, which introduced a forecasting-specific heuristic, this article formalizes CCM into a unified framework spanning forecasting, classification, and anomaly detection.

\section{Spectral Sparsity as Inductive Bias}
\label{sec:motivation}

\subsection{Problem Formulation and Capacity Allocation Dilemma}
Let $\mathcal{D}=\{(\mathbf{x}_i,\mathbf{y}_i)\}_{i=1}^N$ be a time series dataset, where each sample is an input window $\mathbf{x}_i \in \mathbb{R}^{L \times C}$ and a task-dependent target $\mathbf{y}_i$. $N$ denotes the number of training samples, $L$ is the context length, $C$ is the number of channels, and $\mathcal{L}(\cdot, \cdot)$ is the task-specific objective loss function; for forecasting, $H$ denotes the prediction horizon. The target $\mathbf{y}_i$ covers diverse downstream tasks: future horizons $\mathbf{y}_i \in \mathbb{R}^{H \times C}$ in forecasting, categorical class labels $\mathbf{y}_i \in \{1,\dots,K\}$ in classification, or self-reconstruction targets
$\mathbf{y}_i = \mathbf{x}_i$ in anomaly detection. Contrastive detectors such as DCdetector~\cite{yang2023dcdetector} instead optimize agreement between input-view representations without external target labels. We write the backbone and task head as $f_\theta(\mathbf x;\mathbf m)$, where $\theta$ contains trainable parameters and frozen ones are held fixed. The argument $\mathbf m$ collects dimensionless multipliers applied to internal activations. We abbreviate $f_\theta(\mathbf x;\mathbf 1)$ as $f_\theta(\mathbf x)$.

Conceptually, observed temporal sequences combine dominant structural dynamics with residual variations: $\mathbf{x}_i = \mathbf{x}_i^\star + \mathbf{r}_i$, where $\mathbf{x}_i^\star$ represents structured dynamics (e.g., physical trends, cyclic seasonality, and phase-coherent transitions), and $\mathbf{r}_i$ captures residual variations such as measurement noise, local spikes, sensor glitches, missing values, and irregular fluctuations (Figure~\ref{fig:signal_noise_def}). Let $\sigma_i=\|\mathbf r_i\|$ denote the underlying residual magnitude. Since $\mathbf r_i$ is unobserved, Section~\ref{sec:noise_scoring} constructs a spectral residual score $s_i$ as an empirical proxy. Standard empirical risk minimization with fixed regularization solves
\begin{equation}
\min_{\theta}\sum_{i=1}^{N}\mathcal{L}(f_\theta(\mathbf{x}_i),\mathbf{y}_i)+\mathcal{R}(\theta),
\end{equation}
where $\mathcal{R}(\theta)$ denotes an explicit regularizer (such as an $\ell_2$ penalty), typically paired with fixed stochastic regularization (such as uniform dropout) applied identically across all training samples.
When $\sigma_i$ varies substantially across samples, this static configuration creates an inherent dilemma:

\begin{figure}[!t]
    \centering
    \includegraphics[width=\columnwidth]{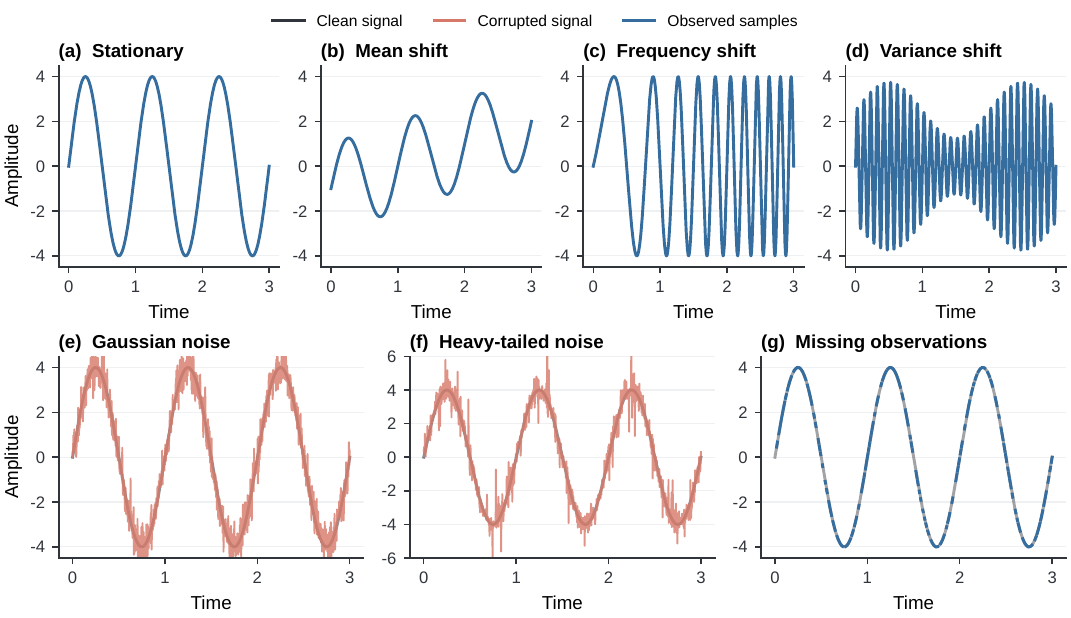}
    \caption{\textbf{Schematic illustration of representative synthetic temporal components and corruption profiles.} \textbf{Top:} Structured dynamics $\mathbf{x}^\star$ from stationary periodicity to non-stationarity in mean (Trend), frequency (Chirp), and amplitude (AM). \textbf{Bottom:} Residual variations $\mathbf{r}$ representing Gaussian noise, heavy-tailed spikes, and randomly scattered missing observations.}
    \label{fig:signal_noise_def}
    \vspace{-1em}
\end{figure}

\begin{itemize}[leftmargin=*]
    \item
    \textbf{Insufficient Model Regularization for Unreliable Samples:} When residual variation $\mathbf{r}_i$ is substantial, an insufficiently regularized model tends to fit transient, non-transferable fluctuations, degrading performance on unseen test samples.
    \item \textbf{Excessive Model Regularization for Reliable Samples:} When a sample exhibits clean, coherent dynamics ($\mathbf{r}_i \approx \mathbf{0}$), aggressive static regularization unnecessarily constrains the model's expressive fidelity, impeding intricate pattern capture.
\end{itemize}

Resolving this conflict requires \textit{modulating the model's effective capacity during training for each sample}. In real-world scenarios, the ground-truth decomposition $\mathbf{r}_i$ and unreliability labels are unavailable, so a label-free proxy is needed to guide sample-level regularization.

\begin{figure}[!t]
    \centering
    \includegraphics[width=\columnwidth]{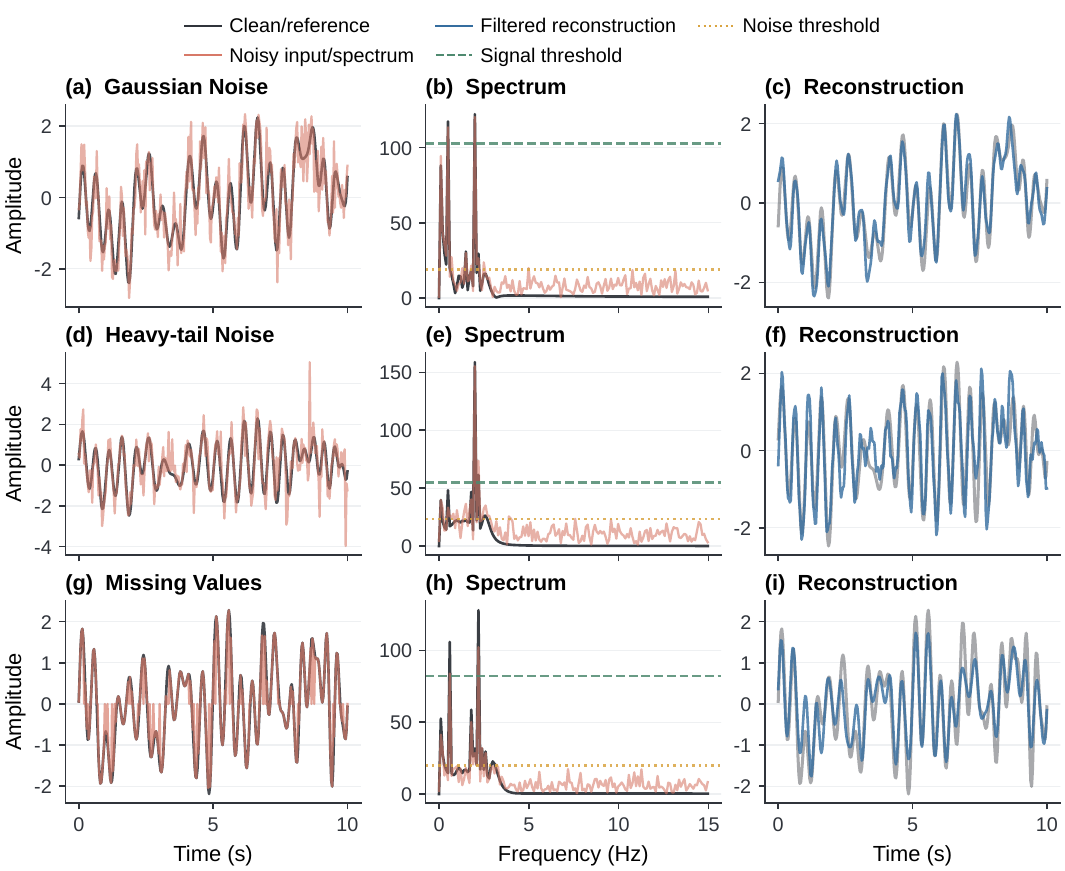}
    \caption{\textbf{Illustration of Spectral Sparsity.} Composite (Periodic + Trend + Chirp + AM) under corruption. \textbf{Left:} Corrupted inputs. \textbf{Middle:} Frequency spectra: structured dynamics concentrate in discrete peaks while corruptions often spread energy across more frequencies. \textbf{Right:} Reconstruction from dominant modes faithfully recovers the underlying ground-truth signal.}
    \label{fig:spectral_analysis}
    \vspace{-1em}
\end{figure}

\subsection{Empirical Verification of Spectral Sparsity}
\label{sec:empirical_verification}

To construct a label-free unreliability proxy, \method exploits the empirical principle of \textit{spectral sparsity} \cite{donoho2006compressed,candes2006nearoptimal}. This inductive bias suits structured temporal dynamics whose energy concentrates in the dominant spectral modes. In the tested settings, stochastic perturbations, localized spikes, and missing-value patterns introduce residual components less concentrated than the dominant structure, motivating our spectral scorer.

By Fourier linearity, the spectrum of an observed sequence satisfies $\mathcal{F}(\mathbf{x}) = \mathcal{F}(\mathbf{x}^\star) + \mathcal{F}(\mathbf{r})$. When the energy of $\mathcal{F}(\mathbf{x}^\star)$ is sparse, retaining the dominant spectral modes while filtering out non-dominant frequencies reconstructs the underlying structure $\mathbf{x}^\star$ while capturing nuisance components in the residual. We empirically validate this across varied temporal regimes:
\begin{itemize}[leftmargin=*]
    \item \noindent\textbf{Evidence across Synthetic and Real-World Data.} As demonstrated in Figure~\ref{fig:spectral_analysis}, retaining only the top \textbf{1\%} dominant Fourier coefficients with the highest energy ($\tau$ at the $99^{\mathrm{th}}$ percentile) reconstructs composite temporal signals with high fidelity across diverse synthetic corruptions; similarly, top-10\% spectral thresholding across extensive real-world benchmarks consistently preserves underlying macro dynamics (detailed in the Supplementary Material).
    \item \noindent\textbf{Spectral Flatness Measure (SFM) as a Concentration Anchor.} To quantify spectral dispersion, we use the SFM, defined as the ratio of the geometric to arithmetic mean of the power spectrum. SFM approaches 1 for flat spectra and decreases as spectral energy becomes concentrated. Figure~\ref{fig:snr_analysis} shows a strong negative correlation ($r = -0.846$, $p < 0.05$)
    between SFM and a proxy SNR (defined as the ratio of spectral energy in the top 10\% highest-amplitude Fourier coefficients to that in the remaining coefficients), supporting SFM as an anchor for automated threshold calibration.
\end{itemize}

\begin{figure}[!b]
    \vspace{-1em}
    \centering
    \includegraphics[width=\columnwidth]{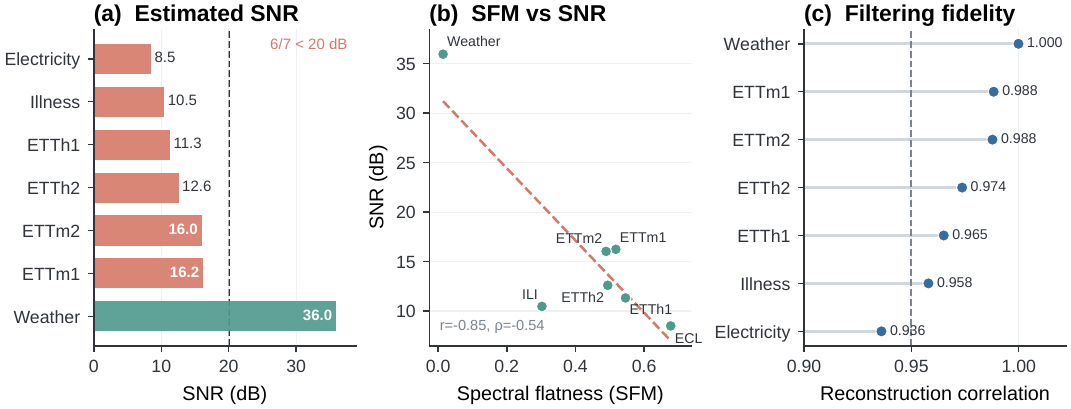}
    \caption{\textbf{Spectral proxy analysis across real-world benchmarks.} \textbf{(a)} Proxy SNR, defined as the ratio of spectral energy in the top 10\% highest-amplitude Fourier coefficients to that in the remaining coefficients; \textbf{(b)} strong negative correlation between SFM and proxy SNR ($r=-0.846$, $p<0.05$); \textbf{(c)} dataset rankings illustrating heterogeneous spectral concentration.}
    \label{fig:snr_analysis}
\end{figure}

\noindent\textbf{Label-Free Spectral Residual Scorer.}
Formally, under the working assumption that temporal structure is spectrally sparse, the dominant component $\hat{\mathbf{x}}$ is approximated by spectral filtering:
\begin{equation}
    \hat{\mathbf{x}} = \mathcal{F}^{-1}(\mathbf{M} \odot \mathcal{F}(\mathbf{x})), \quad \mathbf{M}_k = \mathbb{I}(|\mathcal{F}(\mathbf{x})|_k > \tau),
\end{equation}
where $\mathcal{F}$ and $\mathcal{F}^{-1}$ denote the Discrete Fourier Transform (DFT) and its inverse, $\odot$ the Hadamard product, and $\mathbf{M}_k$ selects modes above $\tau$. The residual $\|\mathbf{x}-\hat{\mathbf{x}}\|_2$ measures the energy outside the retained components, motivating the differentiable soft-mask scorer and mean absolute residual in Section~\ref{sec:noise_scoring}.

\begin{figure*}[!t]
    \centering
    \includegraphics[width=\textwidth]{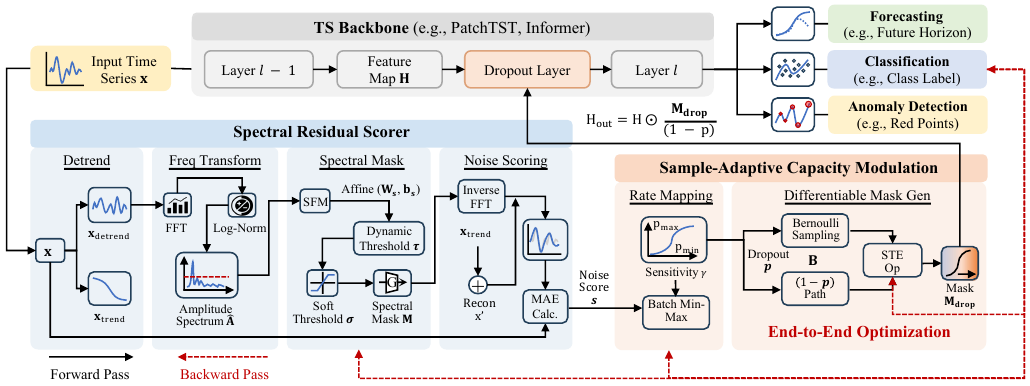}
    \caption{\textbf{Unified cross-task architecture of \method} across \textbf{forecasting, classification, and anomaly detection}. The \textbf{spectral scorer} extracts spectral residual scores via detrending, log-scale normalization, and SFM-anchored soft filtering; the \textbf{rate mapper} maps scores to bounded dropout rates and differentiable activation masks. As a \textbf{task-agnostic framework}, \method applies sample-adaptive regularization during training while leaving the \textbf{inference pipeline} unchanged.}
    \label{fig:framework}
    \vspace{-1em}
\end{figure*}

\section{The \method Framework}
\label{sec:methodology}

Figure~\ref{fig:framework} illustrates \method's dual-path architecture. The upper path represents the backbone and task head (such as forecasting projection, classification head, or reconstruction decoder). The lower path executes capacity modulation:
\begin{enumerate}[leftmargin=*]
    \item \textbf{Spectral Residual Scorer:} Detrends the input, maps it to the frequency domain, applies an SFM-anchored soft spectral filter, and reconstructs the dominant component. The reconstruction error between the original and reconstructed sequences yields a scalar spectral residual score $s_i$.
    \item \textbf{Sample-Adaptive Capacity Mapper:} Normalizes residual scores across the mini-batch and maps them to bounded dropout probabilities $p_i$. Differentiable binary activation masks are generated via the Straight-Through Estimator (STE) \cite{bengio2013estimating} to modulate active feature dimensions along trainable paths during forward passes.
\end{enumerate}
Gradients from the primary task loss $\mathcal{L}$ flow end-to-end through the STE path to optimize the scorer and mapper parameters jointly with the backbone. During evaluation, the capacity modulation path is deactivated, ensuring deterministic inference with no additional scorer or mapper computation or latency.

\subsection{Spectral Residual Scorer}
\label{sec:noise_scoring}
The differentiable spectral residual scorer computes a sample-level unreliability proxy through a four-stage pipeline:

\noindent\textbf{1. Global Linear Detrending.}
Non-stationary trends in finite temporal windows introduce sharp boundary discontinuities when computing the DFT,
causing severe \textit{spectral leakage} that smears energy across frequency bins. To suppress boundary artifacts, for an input sample $\mathbf{x} \in \mathbb{R}^{L \times C}$, we remove linear trends via Ordinary Least Squares (OLS):
\begin{equation}
    \mathbf{x}_{\mathrm{detrend}} = \mathbf{x} - \mathbf{x}_{\mathrm{trend}}, \quad \text{where } \mathbf{x}_{\mathrm{trend}} = \mathbf{T}\mathbf{w}^* + \mathbf{b}^*,
\end{equation}
where $\mathbf{T} = [0, 1, \dots, L-1]^\top \in \mathbb{R}^{L \times 1}$ is the temporal index vector. The slope $\mathbf{w}^* \in \mathbb{R}^{1 \times C}$ and intercept $\mathbf{b}^* \in \mathbb{R}^{1 \times C}$ are obtained independently for each channel by minimizing $\sum_{t=1}^{L}(x_{t,c}-(t-1)w_c-b_c)^2$. This detrending isolates baseline trends and reduces spectral leakage.

\noindent\textbf{2. Log-Scale Spectral Normalisation.}
We compute the DFT of $\mathbf{x}_{\mathrm{detrend}}$ via a real-input fast Fourier transform (FFT), implemented with rFFT. For a real-valued input, conjugate symmetry makes the negative-frequency half redundant, so rFFT returns $K_f = \lfloor L/2 \rfloor + 1$ coefficients $\mathbf{Z} \in \mathbb{C}^{K_f \times C}$ and the amplitude spectrum $\mathbf{A} = |\mathbf{Z}|$. Because spectral amplitudes often span multiple orders of magnitude, we compute the log-amplitude $\mathbf{L} = \log(1 + \mathbf{A})$ and perform min-max normalisation across frequency bins for each channel:
\begin{equation}
    \hat{A}_{k,c} = \frac{L_{k,c} - \min_j L_{j,c}}{\max\left(\max_j L_{j,c} - \min_j L_{j,c}, \, \epsilon\right)},
\end{equation}
where $k \in \{1,\dots,K_f\}$ indexes frequency bins and $c \in \{1,\dots,C\}$ channels, and $\epsilon = 10^{-8}$ prevents division by zero.

\noindent\textbf{3. Learnable SFM-Anchored Spectral Filter.}
To separate dominant spectral modes from diffuse residual energy without relying on hard, hand-tuned frequency cutoffs, we design a learnable soft filter anchored by SFM. For channel $c$, we define the spectral power proxy $P_{k,c} = L_{k,c}^2 + \epsilon$ and compute:
\begin{equation}
    \mathrm{SFM}_{c} = \frac{\exp\left(\frac{1}{K_f} \sum_{k=1}^{K_f} \ln P_{k,c}\right)}{\frac{1}{K_f} \sum_{k=1}^{K_f} P_{k,c} + \epsilon}.
\end{equation}
An $\mathrm{SFM}_c$ value close to 1 signifies a flat, noise-like power spectrum, whereas smaller values indicate sharp spectral concentration. We parameterize a channel-adaptive threshold $\tau_c$ as a learnable affine transformation of $\mathrm{SFM}_c$:
\begin{equation}
    \tau_c = \operatorname{sigmoid}\left(w_{s,c} \cdot \mathrm{SFM}_c + b_{s,c}\right),
\end{equation}
where $w_{s,c}$ and $b_{s,c}$ are trainable parameters. The continuous spectral retention mask $\mathbf{M} \in [0, 1]^{K_f \times C}$ is then computed as:
\begin{equation}
    M_{k,c} = \operatorname{sigmoid}\left( \operatorname{softplus}(\alpha) \cdot (\hat{A}_{k,c} - \tau_c) \right),
\end{equation}
where the learnable scalar $\alpha$ controls transition steepness via a strictly positive $\operatorname{softplus}$ transform. This soft mask smoothly attenuates non-dominant frequency components while preserving dominant structural harmonics.

\noindent\textbf{4. Residual Scoring.}
We reconstruct the dominant temporal signal by performing the inverse real-input FFT ($\mathcal{F}^{-1}$) on the filtered spectrum and restoring the previously extracted trend:
\begin{equation}
    \mathbf{x}' = \mathcal{F}^{-1}(\mathbf{Z} \odot \mathbf{M}) + \mathbf{x}_{\mathrm{trend}}.
\end{equation}
We first compute a channel-level residual score and then aggregate it into one sample-level score:
\begin{equation}
    s_c = \frac{1}{L} \sum_{t=1}^L |x_{t,c} - x'_{t,c}|,\qquad
    s = \frac{1}{C}\sum_{c=1}^{C}s_c.
\end{equation}
The mean $s$ aggregates channel-wise residuals into a label-free proxy for sample unreliability. It yields a shared dropout rate across adaptive layers, allowing sample-level modulation of joint multivariate representations without requiring a correspondence between input channels and hidden features.

\subsection{Sample-Adaptive Capacity Modulation}
\label{sec:adaptive_dropout}

Given sample residual scores $\{s_i\}_{i=1}^{|\mathcal{B}|}$ in mini-batch $\mathcal{B}$, the second module dynamically maps $s_i$ to a sample-specific dropout rate $p_i$, modulating the backbone's active capacity.

\noindent\textbf{Batch-Aware Relative Mapping.}
Because raw residual magnitudes vary across datasets and mini-batches,
we normalize scores relative to the current mini-batch distribution:
\begin{equation}
    \hat{s}_i = \begin{cases}
    \dfrac{s_i - \min_{j} s_j}{\max_{j} s_j - \min_{j} s_j}, & \text{if } \max_{j} s_j - \min_{j} s_j > \epsilon, \\[6pt]
    0.5, & \text{otherwise}.
    \end{cases}
\end{equation}
If score dispersion within a batch falls below $\epsilon$, assigning $\hat{s}_i = 0.5$ provides a neutral default. We then map score $\hat{s}_i \in [0, 1]$ to a dropout probability $p_i \in [p_{\min}, p_{\max}]$ through a learnable nonlinear sensitivity curve:
\begin{equation}
    \tilde{s}_i = \tanh\left(\hat{s}_i \cdot \operatorname{softplus}(\gamma)\right),
\end{equation}
\begin{equation}
    \label{eq:rate_mapper}
    p_i = p_{\min} + (p_{\max} - p_{\min}) \cdot \tilde{s}_i,
\end{equation}
where the shared trainable scalar $\gamma$ governs sensitivity to relative unreliability via $a = \operatorname{softplus}(\gamma) > 0$. Lower-score samples ($\hat{s}_i \to 0$) receive a base dropout rate $p_{\min}$ to retain more active capacity, while higher-score samples ($\hat{s}_i \to 1$) receive elevated rates up to $p_{\min} + (p_{\max}-p_{\min})\tanh(a) \le p_{\max}$ to restrict capacity and suppress noise memorization.

\noindent\textbf{Differentiable Mask Generation via STE.}
Stochastic dropout mask sampling is discrete and non-differentiable. To enable end-to-end backpropagation from the task loss $\mathcal{L}$ to the scorer parameters $(\alpha, \mathbf{w}_s, \mathbf{b}_s, \gamma)$, we employ STE~\cite{bengio2013estimating}. For sample $i$, the forward pass samples a binary mask
$\mathbf{B}_i \sim \operatorname{Bernoulli}(1 - p_i)$. A larger $p_i$ zeros out more dimensions, while surviving activations are scaled by inverted-dropout normalization. The differentiable activation mask $\mathbf{M}_{\mathrm{drop},i}$ is defined as:
\begin{equation}
    \mathbf{M}_{\mathrm{drop},i} = \mathbf{B}_i + q_i - \operatorname{sg}(q_i), \qquad q_i=1-p_i,
\end{equation}
where $\operatorname{sg}(\cdot)$ denotes the stop-gradient operator. In the forward pass, $\mathbf{M}_{\mathrm{drop},i} = \mathbf{B}_i$, while in the backward pass, its surrogate gradient satisfies $\partial \mathbf{M}_{\mathrm{drop},i} / \partial p_i = -1$. For an activation tensor $\mathbf{H}_i$, inverted dropout scaling produces:
\begin{equation}
    \mathbf{H}_{i,\mathrm{out}} = \mathbf{H}_i \odot \mathbf{m}_i, \qquad \mathbf{m}_i = \frac{\mathbf{M}_{\mathrm{drop},i}}{q_i}.
\end{equation}
This operation directly modulates sample-wise active capacity during training while preserving the unbiased conditional expectation of feature representations ($\mathbb{E}[\mathbf{H}_{i,\mathrm{out}} \mid \mathbf{H}_i] = \mathbf{H}_i$).

\noindent\textbf{Task-Agnostic End-to-End Optimization.}
The parameter count added by \method is minimal ($2C + 2$ scalar parameters for $C$ channels: $w_{s,c}, b_{s,c}, \alpha, \gamma$). During training, the scorer and mapper parameters $\phi=\{\alpha,\mathbf{w}_s,\mathbf{b}_s,\gamma\}$ are optimized jointly with the trainable backbone parameters $\theta$:
\begin{equation}
    \min_{\theta,\phi}\;\mathbb{E}_{\mathcal B}\!\left[
    \frac{1}{|\mathcal B|}\sum_{i\in\mathcal B}
    \mathbb{E}_{\mathbf B_i}\!\left[
    \mathcal L\bigl(f_\theta(\mathbf x_i;\mathbf m_i),\mathbf y_i\bigr)
    \right]\right].
\end{equation}
Here, $\mathbf m_i$ collects the activation multipliers across all adaptive dropout sites for sample $i$, with rates shared across sites and Bernoulli masks drawn independently. The outer expectation reflects mini-batch sampling, as relative unreliability scores are evaluated across $\mathcal B$. Gradients backpropagate end-to-end through the STE to update scorer parameters $\phi$ alongside backbone parameters $\theta$. Because \method acts strictly on internal activation paths without requiring auxiliary reliability labels, it seamlessly integrates with any differentiable task loss $\mathcal{L}$.

\subsection{Inference Procedure}
\label{sec:inference}
During inference, the capacity modulation path is completely bypassed: no spectral decomposition, residual scoring, or rate mapping is executed. The backbone and task heads operate deterministically with all dropout operations disabled ($\mathbf{m}_i \equiv \mathbf{1}$). Consequently, \method preserves the deterministic inference graph of the underlying backbone, introducing \textbf{zero} additional parameters, memory overhead, or latency at test time.

\begin{table*}[!b]
    \caption{Definitions of clean and corrupted signals. Overview of the four clean regimes and three corruption profiles ($\tilde{x}_t$).}
    \label{tab:synthetic_signal_corruption_definitions}
    \centering
    \small
    \renewcommand{\arraystretch}{1.08}

    \resizebox{\textwidth}{!}{%
        \begin{tabular}{clcll}
            \toprule
            \rowcolor{TableHeader}
            \textbf{Type} &
            \textbf{Category} &
            \textbf{Mathematical Formulation} &
            \textbf{Physical Interpretation} &
            \textbf{Example} \\
            \midrule

            \multirow{4}{*}{\rotatebox{90}{\textbf{Signal}}}
            & Stationary (Periodic)
            & $x_t = \sum_{k} A_k \sin(2\pi f_k t + \phi_k)$
            & Stable equilibrium; constant freq. \& amp.
            & Power grid voltage \\
            \addlinespace[0.4em]

            & Non-stat. (Mean)
            & $x_t = \alpha t + \beta + \sum_{k} A_k \sin(2\pi f_k t)$
            & Trend \& Seasonality; drifts with fluctuations
            & Macroeconomic growth (GDP) \\
            \addlinespace[0.4em]

            & Non-stat. (Freq.)
            & $x_t = A \sin(2\pi (f_0 t + \frac{1}{2}kt^2))$
            & Spectral Drift; time-varying frequency
            & Doppler effects (Radar/Sonar) \\
            \addlinespace[0.4em]

            & Non-stat. (Var.)
            & $x_t = (1 + \mu \sin(2\pi f_m t)) \sin(2\pi f_c t)$
            & Time-varying amplitude envelope
            & Vibration or communication signals \\

            \midrule

            \multirow{3}{*}{\rotatebox{90}{\textbf{Noise}}}

            & Gaussian Noise
            & $\tilde{x}_t = x_t + \epsilon_t, \enspace \epsilon_t \sim \mathcal{N}(0, \sigma^2)$
            & Random measurement noise
            & Sensor thermal noise \\
            \addlinespace[0.4em]

            & Heavy-tail (Student-t)
            & $\tilde{x}_t = x_t + \epsilon_t, \enspace \epsilon_t \sim t_\nu \,(\nu=2.5)$
            & Heavy-tailed random corruption
            & Impulsive sensor disturbances \\
            \addlinespace[0.4em]

            & Missing Values
            & $\tilde{x}_t = x_t \odot m_t, \enspace m_t \sim \mathcal{B}(1-p)$
            & Observation Failures; random data loss
            & Wireless packet loss \\

            \bottomrule
        \end{tabular}%
    }
    \vspace{-1em}
\end{table*}

\section{Theoretical Capacity Analysis}
\label{sec:theory}

We relate sample-adaptive dropout to local regularization, then analyze how residual heterogeneity affects capacity allocation under a bias--variance surrogate (Appendix~E).

\subsection{From Dynamic Dropout to Sample-Adaptive Regularization}
\label{sec:theory_regularization}
Fix the backbone parameters, sample $i$, and its mini-batch. Analyze one adaptive layer with all other activation multipliers set to one. Let $\mathbf h_i\in\mathbb R^d$ be its pre-mask activations, vectorised here, where $d$ is the number of coordinates and $j\in\{1,\ldots,d\}$ indexes them. Conditional on the sample and mini-batch, the retention rate is $q_i=1-p_i$, and independent mask entries satisfy $B_{ij}\sim\operatorname{Bernoulli}(q_i)$. Following Sections~\ref{sec:motivation} and~\ref{sec:adaptive_dropout}, $\mathbf m_i$ collects the multipliers $\mathbf B_i/q_i$ at this layer and ones elsewhere. Thus $f_\theta(\mathbf x_i;\mathbf m_i)$ is the stochastic forward mapping, whereas $f_\theta(\mathbf x_i)=f_\theta(\mathbf x_i;\mathbf1)$ is deterministic inference.

The inverted multiplier has mean one and variance
\begin{equation}
    \label{eq:multiplier_variance}
    \mathbb E[B_{ij}/q_i]=1,\qquad
    \operatorname{Var}(B_{ij}/q_i)=\frac{p_i}{1-p_i}=: \lambda(p_i).
\end{equation}
Consequently, the activation perturbation $\boldsymbol\delta_i=\mathbf h_i\odot(\mathbf B_i/q_i-\mathbf1)$ has zero mean and covariance $\lambda(p_i)\operatorname{diag}(h_{i1}^2,\ldots,h_{id}^2)$. Let $\ell_i(\mathbf h)$ be the downstream task loss with the target fixed, so $\ell_i(\mathbf h_i)=\mathcal L(f_\theta(\mathbf x_i),\mathbf y_i)$. All mask expectations and variances in this subsection are conditional on the fixed input and mini-batch. A second-order Taylor expansion of $\ell_i$ around $\mathbf h_i$ gives~\cite{wager2013drop}
\begin{align}
    \mathbb{E}_{\mathbf B_i}\!\left[\mathcal L\bigl(f_\theta(\mathbf x_i;\mathbf m_i),\mathbf y_i\bigr)\right]
      &\approx \mathcal L\bigl(f_\theta(\mathbf x_i),\mathbf y_i\bigr)
       +\tfrac{1}{2}\lambda(p_i)\,\Omega(\mathbf x_i;\theta),\notag\\[-2pt]
    \Omega(\mathbf x_i;\theta)
      &=\sum_{j=1}^{d} h_{ij}^2\,\frac{\partial^{2}\ell_i}{\partial h_j^{2}}(\mathbf h_i).
\end{align}
The zero-mean perturbation eliminates the first-order term. $\Omega$ is the activation-weighted loss curvature. When $\Omega\ge0$, the correction penalizes sensitivity to activation perturbations with coefficient $\lambda(p_i)/2$, which increases with $p_i$ because $\mathrm d\lambda/\mathrm dp=(1-p)^{-2}>0$. The approximation requires a smooth loss and a small Taylor remainder (Appendix~E). For fixed mapper parameters, Equation~\eqref{eq:rate_mapper} assigns larger $p_i$ and hence larger $\lambda(p_i)$ to higher normalized residual scores $\hat s_i$ within a mini-batch. Define the layer's expected active width as
\begin{equation}
    \mathcal C_i^{\mathrm{train}}
       :=\mathbb E\!\left[\sum_{j=1}^d B_{ij}\right]=(1-p_i)d.
\end{equation}
This counts retained activation coordinates before inverted scaling, with no change in parameter count. Equivalently, $\mathcal C_i^{\mathrm{train}}=d/(1+\lambda(p_i))$, so a larger coefficient yields a smaller active width. At inference, all activation multipliers are one and the scorer and mapper are bypassed.

\subsection{A Bias--Variance Surrogate for Capacity Allocation}
\label{sec:theory_excess_risk}
For dropout rate $p$, write $\lambda=\lambda(p)=p/(1-p)$ for the corresponding regularization strength. The sample-specific value is $\lambda(p_i)$. The bounds $0\le p_{\min}<p_{\max}<1$ define $\Lambda=[\lambda_{\min},\lambda_{\max}]$, with $\lambda_{\min}=p_{\min}/(1-p_{\min})$ and $\lambda_{\max}=p_{\max}/(1-p_{\max})$. Let $\sigma(\mathbf x)\ge0$ denote the residual magnitude, with $\sigma(\mathbf x_i)=\sigma_i=\|\mathbf r_i\|$ as defined in Section~\ref{sec:motivation}. The learned mapper may realize a narrower range within $\Lambda$.

We make two assumptions on a scalar prediction error. First, a linear bias response $b\lambda$ models regularization-induced systematic error, with zero baseline bias and shared slope $b\ne0$. This first-order model gives squared bias $C_1\lambda^2$, where $C_1=b^2>0$. Second, model the fitted residual response by a scalar quadratic penalty. Let $\varepsilon$ be a random perturbation with $\mathbb E[\varepsilon\mid\sigma]=0$ and $\operatorname{Var}(\varepsilon\mid\sigma)=1$. For residual input $\sigma\varepsilon$, define the fitted scalar response $v_\lambda$ by
\begin{equation}
    v_\lambda=\arg\min_v\left\{\tfrac12(v-\sigma\varepsilon)^2
        +\tfrac{\lambda}{2}v^2\right\}
        =\frac{\sigma\varepsilon}{1+\lambda}
\end{equation}
The stationarity condition $(v-\sigma\varepsilon)+\lambda v=0$ gives the displayed solution. With a fixed output scale $\sqrt{C_2}$, where $C_2>0$, its conditional variance is $C_2\operatorname{Var}(v_\lambda\mid\sigma)=C_2\sigma^2/(1+\lambda)^2$. This assumed fitted-response variance decreases with $\lambda$, whereas the forward mask variance in Equation~\eqref{eq:multiplier_variance} increases. The two describe distinct effects.

Combining the assumed bias and residual response gives $e_\lambda=b\lambda+\sqrt{C_2}v_\lambda$. Its conditional mean squared error is
\begin{equation}
\label{eq:surrogate_risk}
\mathcal E(\lambda,\sigma)
  :=\mathbb E[e_\lambda^2\mid\sigma]
  =C_1\lambda^2+C_2\frac{\sigma^2}{(1+\lambda)^2}.
\end{equation}
The zero conditional mean of $v_\lambda$ removes the cross term, leaving squared bias and variance. Shared constants $C_1,C_2$ isolate residual magnitude as the source of heterogeneity. 

The surrogate is strictly convex for $\lambda\ge0$, since $\partial_\lambda^2\mathcal E=2C_1+6C_2\sigma^2/(1+\lambda)^4>0$. Its minimizer $\lambda^\star(\sigma)$ over $[0,\infty)$ satisfies $C_1\lambda^\star(1+\lambda^\star)^3=C_2\sigma^2$. The left side increases strictly with $\lambda^\star$, so the minimizer increases with $\sigma$. The constrained oracle allocation is $g(\sigma)=\lambda^\star_\Lambda(\sigma)=\Pi_\Lambda(\lambda^\star(\sigma))$, where $\Pi_\Lambda(u)=\min\{\lambda_{\max},\max\{\lambda_{\min},u\}\}$ clips a scalar to $\Lambda$.

\begin{theorem}[Excess Surrogate Risk of Uniform Capacity Allocation]
\label{thm:fixed_capacity_excess_risk}
Let $\mathbf X$ be a random input window from the population under study. Under Equation~\eqref{eq:surrogate_risk}, assume $\mathbb E[\sigma(\mathbf X)^2]<\infty$ and $\operatorname{Var}[g(\sigma(\mathbf X))]>0$. Every sample-independent strength $\lambda_{\mathrm{fix}}\in\Lambda$ incurs strictly positive expected excess surrogate risk relative to the oracle allocation:
\begin{equation}
\begin{aligned}
 &\mathbb E_{\mathbf X}\!\left[\mathcal E(\lambda_{\mathrm{fix}},\sigma(\mathbf X))
       -\mathcal E(g(\sigma(\mathbf X)),\sigma(\mathbf X))\right]\\
 &\qquad\ge C_1\operatorname{Var}[g(\sigma(\mathbf X))]>0.
\end{aligned}
\end{equation}
\end{theorem}

\noindent\textit{Proof sketch.} The bound $\partial_\lambda^2\mathcal E\ge2C_1$ gives strong convexity. At the minimizer $g$, the first-order optimality condition yields
\[
\mathcal E(\lambda,\sigma)-\mathcal E(g(\sigma),\sigma)
\ge C_1(\lambda-g(\sigma))^2.
\]
Taking expectations with $\mathbb E[(\lambda_{\mathrm{fix}}-g)^2]\ge\operatorname{Var}(g)$ and $g=g(\sigma(\mathbf X))$ proves it. Appendix~E covers boundary optima.

Theorem~\ref{thm:fixed_capacity_excess_risk} assumes distinct oracle rates after clipping, excluding constant allocation at a shared bound. Under this model, larger residual magnitudes favor stronger regularization and smaller expected active widths. This surrogate result motivates \method's monotone mapper, which uses spectral residual scores to guide sample-wise regularization.

\section{Experiments}
\label{sec:experiments}

To validate Capacity-Centric Modulation and \method, our experiments address five research questions (RQs):
\begin{itemize}[leftmargin=*]
    \item \textbf{RQ1 (Corruption Robustness):} Does \method reduce forecasting error across corruption scales? (\S\ref{sec:rq1})
    \item \textbf{RQ2 (Real-World Forecasting):} Does \method improve forecasting accuracy on public benchmarks? (\S\ref{sec:rq2})
    \item \textbf{RQ3 (Generalization):} Does sample-adaptive capacity improve transfer across data regimes and domains? (\S\ref{sec:rq3})
    \item \textbf{RQ4 (Cross-Task Generality):} Does \method generalize to both classification and anomaly detection tasks? (\S\ref{sec:rq4})
    \item \textbf{RQ5 (Ablations \& Efficiency):} Do modules help, is \method efficient, and does it complement data selection? (\S\ref{sec:rq5})
\end{itemize}

\begin{table*}[!t]
    \small
    \centering
    \caption{Forecasting results on Synth-12 ($L=96$, averaged across $H \in \{96, 192, 336, 720\}$). +\method denotes backbones augmented with \method; bold marks the lower error per pair. Avg. Gain macro-averages relative improvements across noise levels $\sigma \in [0.1, 0.9]$.}
    \label{tab:synthetic_forecasting_summary}

    \setlength{\tabcolsep}{2.8pt}
    \renewcommand{\arraystretch}{1}

    \providecommand{\gc}{\tableourscell}
    \providecommand{\win}[1]{\textbf{#1}}
    \providecommand{\secbest}[1]{\underline{#1}}
    \providecommand{\dtwin}[2]{\gc\textbf{#1}\tableforecastchange{#2}{#1}}
    \providecommand{\dtsecond}[2]{\gc\underline{#1}\tableforecastchange{#2}{#1}}
    \providecommand{\dttie}[2]{\gc\textbf{#1}\tableforecastchange{#2}{#1}}

    \resizebox{\textwidth}{!}{%
    \begin{tabular}{@{}c c *{18}{c}@{}}
    \toprule
    \rowcolor{TableHeader}
    \multirow{2}{*}{\textbf{$\sigma$}} & \multirow{2}{*}{\textbf{Metric}} &
    \multicolumn{2}{c}{Informer (\citeyear{zhou2021informer})} &
    \multicolumn{2}{c}{Crossformer (\citeyear{zhang2022crossformer})} &
    \multicolumn{2}{c}{PatchTST (\citeyear{nie2023patchtst})} &
    \multicolumn{2}{c}{TimesNet (\citeyear{wu2023timesnet})} &
    \multicolumn{2}{c}{iTransformer (\citeyear{liu2023itransformer})} &
    \multicolumn{2}{c}{TimeMixer (\citeyear{wang2024timemixer})} &
    \multicolumn{2}{c}{WPMixer (\citeyear{murad2025wpmixer})} &
    \multicolumn{2}{c}{TimeFilter (\citeyear{hutimefilter})} &
    \multicolumn{2}{c}{MultiPatchFormer (\citeyear{naghashi2025multiscale})} \\
    \cmidrule(lr){3-4} \cmidrule(lr){5-6} \cmidrule(lr){7-8} \cmidrule(lr){9-10} \cmidrule(lr){11-12} \cmidrule(lr){13-14} \cmidrule(lr){15-16} \cmidrule(lr){17-18} \cmidrule(lr){19-20}
    \rowcolor{TableSubheader}
    & & Raw & \tableoursheader +\method & Raw & \tableoursheader +\method & Raw & \tableoursheader +\method & Raw & \tableoursheader +\method & Raw & \tableoursheader +\method & Raw & \tableoursheader +\method & Raw & \tableoursheader +\method & Raw & \tableoursheader +\method & Raw & \tableoursheader +\method \\
    \midrule

    \multirow{2}{*}{0.1} & MSE & \secbest{0.966} & \dtwin{0.514}{0.966} & \secbest{0.450} & \dtwin{0.386}{0.450} & \secbest{0.542} & \dtwin{0.529}{0.542} & \secbest{0.902} & \dtwin{0.848}{0.902} & \secbest{0.523} & \dtwin{0.519}{0.523} & \secbest{0.545} & \dtwin{0.540}{0.545} & \secbest{0.539} & \dtwin{0.519}{0.539} & \secbest{0.543} & \dtwin{0.530}{0.543} & \secbest{0.547} & \dtwin{0.523}{0.547} \\
     & MAE & \secbest{0.775} & \dtwin{0.581}{0.775} & \secbest{0.502} & \dtwin{0.471}{0.502} & \secbest{0.561} & \dtwin{0.554}{0.561} & \secbest{0.763} & \dtwin{0.739}{0.763} & \secbest{0.544} & \dtwin{0.541}{0.544} & \secbest{0.555} & \dtwin{0.554}{0.555} & \secbest{0.552} & \dtwin{0.539}{0.552} & \secbest{0.558} & \dtwin{0.550}{0.558} & \secbest{0.558} & \dtwin{0.546}{0.558} \\
    \addlinespace[0.25em]

    \multirow{2}{*}{0.3} & MSE & \secbest{0.978} & \dtwin{0.507}{0.978} & \secbest{0.439} & \dtwin{0.378}{0.439} & \secbest{0.571} & \dtwin{0.543}{0.571} & \secbest{0.847} & \dtwin{0.821}{0.847} & \secbest{0.554} & \dtwin{0.548}{0.554} & \secbest{0.553} & \dtwin{0.551}{0.553} & \secbest{0.556} & \dtwin{0.543}{0.556} & \secbest{0.577} & \dtwin{0.572}{0.577} & \secbest{0.559} & \dtwin{0.550}{0.559} \\
     & MAE & \secbest{0.779} & \dtwin{0.577}{0.779} & \secbest{0.495} & \dtwin{0.464}{0.495} & \secbest{0.580} & \dtwin{0.564}{0.580} & \secbest{0.743} & \dtwin{0.730}{0.743} & \secbest{0.566} & \dtwin{0.560}{0.566} & \secbest{0.565} & \dtwin{0.561}{0.565} & \secbest{0.563} & \dtwin{0.554}{0.563} & \secbest{0.583} & \dtwin{0.581}{0.583} & \secbest{0.568} & \dtwin{0.562}{0.568} \\
    \addlinespace[0.25em]

    \multirow{2}{*}{0.5} & MSE & \secbest{0.936} & \dtwin{0.494}{0.936} & \secbest{0.431} & \dtwin{0.391}{0.431} & \secbest{0.573} & \dtwin{0.556}{0.573} & \secbest{0.816} & \dtwin{0.794}{0.816} & \secbest{0.561} & \dtwin{0.554}{0.561} & \secbest{0.558} & \dtwin{0.552}{0.558} & \secbest{0.556} & \dtwin{0.550}{0.556} & \secbest{0.585} & \dtwin{0.578}{0.585} & \secbest{0.581} & \dtwin{0.550}{0.581} \\
     & MAE & \secbest{0.762} & \dtwin{0.571}{0.762} & \secbest{0.492} & \dtwin{0.477}{0.492} & \secbest{0.584} & \dtwin{0.575}{0.584} & \secbest{0.729} & \dtwin{0.721}{0.729} & \secbest{0.575} & \dtwin{0.570}{0.575} & \secbest{0.570} & \dtwin{0.566}{0.570} & \secbest{0.563} & \dtwin{0.560}{0.563} & \secbest{0.594} & \dtwin{0.589}{0.594} & \secbest{0.588} & \dtwin{0.569}{0.588} \\
    \addlinespace[0.25em]

    \multirow{2}{*}{0.7} & MSE & \secbest{0.841} & \dtwin{0.470}{0.841} & \secbest{0.411} & \dtwin{0.370}{0.411} & \secbest{0.571} & \dtwin{0.561}{0.571} & \secbest{0.785} & \dtwin{0.763}{0.785} & \secbest{0.556} & \dtwin{0.549}{0.556} & \secbest{0.552} & \dtwin{0.548}{0.552} & \secbest{0.542} & \dtwin{0.537}{0.542} & \secbest{0.576} & \dtwin{0.571}{0.576} & \secbest{0.573} & \dtwin{0.550}{0.573} \\
     & MAE & \secbest{0.726} & \dtwin{0.558}{0.726} & \secbest{0.481} & \dtwin{0.435}{0.481} & \secbest{0.584} & \dtwin{0.577}{0.584} & \secbest{0.721} & \dtwin{0.707}{0.721} & \secbest{0.578} & \dtwin{0.572}{0.578} & \secbest{0.567} & \dtwin{0.564}{0.567} & \secbest{0.559} & \dtwin{0.556}{0.559} & \secbest{0.591} & \dtwin{0.588}{0.591} & \secbest{0.585} & \dtwin{0.573}{0.585} \\
    \addlinespace[0.25em]

    \multirow{2}{*}{0.9} & MSE & \secbest{0.828} & \dtwin{0.464}{0.828} & \secbest{0.409} & \dtwin{0.360}{0.409} & \secbest{0.552} & \dtwin{0.539}{0.552} & \secbest{0.742} & \dtwin{0.717}{0.742} & \secbest{0.535} & \dtwin{0.532}{0.535} & \secbest{0.536} & \dtwin{0.531}{0.536} & \secbest{0.518} & \dtwin{0.516}{0.518} & \secbest{0.555} & \dtwin{0.550}{0.555} & \secbest{0.555} & \dtwin{0.541}{0.555} \\
     & MAE & \secbest{0.716} & \dtwin{0.550}{0.716} & \secbest{0.479} & \dtwin{0.457}{0.479} & \secbest{0.575} & \dtwin{0.570}{0.575} & \secbest{0.701} & \dtwin{0.689}{0.701} & \secbest{0.570} & \dtwin{0.568}{0.570} & \secbest{0.560} & \dtwin{0.557}{0.560} & \secbest{0.547} & \dtwin{0.546}{0.547} & \secbest{0.581} & \dtwin{0.577}{0.581} & \secbest{0.580} & \dtwin{0.569}{0.580} \\

    \midrule
    \rowcolor{TableSummary}
    \multicolumn{2}{c}{\textbf{Avg. Gain}} &
    \multicolumn{2}{c}{\textbf{46.0\% / 24.5\%}} &
    \multicolumn{2}{c}{\textbf{11.9\% / 5.9\%}} &
    \multicolumn{2}{c}{\textbf{2.9\% / 1.5\%}} &
    \multicolumn{2}{c}{\textbf{3.6\% / 1.9\%}} &
    \multicolumn{2}{c}{\textbf{1.0\% / 0.9\%}} &
    \multicolumn{2}{c}{\textbf{0.8\% / 0.5\%}} &
    \multicolumn{2}{c}{\textbf{1.7\% / 1.0\%}} &
    \multicolumn{2}{c}{\textbf{1.2\% / 0.8\%}} &
    \multicolumn{2}{c}{\textbf{3.6\% / 2.1\%}} \\

    \bottomrule
    \end{tabular}%
    }
    \vspace{-1em}
\end{table*}

\subsection{Experimental Settings}
\label{sec:settings}

\noindent\textbf{Datasets \& Benchmark Coverage.}
Our benchmarks span multiple time series regimes. For \textit{predictive modeling (forecasting)}, we incorporate nine real-world benchmarks: \textit{ETTh1/ETTh2}, \textit{ETTm1/ETTm2} \cite{zhou2021informer}, \textit{Electricity}, \textit{Exchange Rate} \cite{lai2018modeling}, \textit{Weather}, \textit{ILI} \cite{wu2021autoformer}, and \textit{ECG} waveforms from PhysioNet \cite{goldberger2000physiobank}. To evaluate robustness under controlled reliability heterogeneity, we introduce \textit{Synth-12}, which combines structured temporal signals with layered corruptions at five scales (Tables~\ref{tab:synthetic_signal_corruption_definitions} and~\ref{tab:synthetic_corruption_statistics}). For \textit{discriminative modeling (classification)}, we evaluate 30 UEA datasets \cite{bagnall2018uea}, \textit{HAR} \cite{anguita2013public}, and \textit{Sleep-EDF} \cite{kemp2000analysis}. For \textit{reconstructive modeling (anomaly detection)}, we evaluate \textit{SMD} \cite{su2019robust}, \textit{MSL}/\textit{SMAP} \cite{hundman2018detecting}, and \textit{PSM} \cite{xu2021anomaly}. Benchmark statistics and protocols are provided in the Supplementary Material.

\begin{table}[H]
    \centering
    \caption{Synth-12 corruption gradient statistics across noise levels $\sigma \in \{0.1, \dots, 0.9\}$ ($N = 33,600$, clean SFM $= 0.003$).}
    \label{tab:synthetic_corruption_statistics}
    \renewcommand{\arraystretch}{0.92}
    \setlength{\arrayrulewidth}{0.3pt}
    \begin{tabular*}{\columnwidth}{@{\extracolsep{\fill}} c | c | c | c | c | l @{}}
    \toprule
    \rowcolor{TableHeader}
    $\boldsymbol{\sigma}$ & \textbf{SNR (dB)} & \textbf{SFM} & $\boldsymbol{\Delta}$\textbf{SFM} & \textbf{MSE} & \textbf{Stress Regime} \\
    \midrule
    0.1 & 23.77 & 0.008 & 0.005 & 0.004 & Mild \\
    0.3 & 16.57 & 0.023 & 0.020 & 0.022 & Moderate \\
    0.5 & 12.39 & 0.046 & 0.043 & 0.058 & High \\
    0.7 & 9.54  & 0.076 & 0.073 & 0.111 & Severe \\
    0.9 & 7.39  & 0.109 & 0.106 & 0.182 & Extreme \\
    \bottomrule
    \end{tabular*}
\end{table}

\noindent\textbf{Backbone Architectures.}
We evaluate \method across diverse deep temporal architectures. Forecasting backbones include Transformer-based (Informer \cite{zhou2021informer}, Crossformer \cite{zhang2022crossformer}, PatchTST \cite{nie2023patchtst}, iTransformer \cite{liu2023itransformer}, MultiPatchFormer \cite{naghashi2025multiscale}), MLP-based (TimeMixer \cite{wang2024timemixer}, WPMixer \cite{murad2025wpmixer}), CNN-based (TimesNet \cite{wu2023timesnet}), and graph-based (TimeFilter \cite{hutimefilter}). Classification backbones include iTransformer, PatchTST, NSFormer \cite{liu2022non}, TimesNet, InceptionTime \cite{fawaz2020inceptiontime}, and MiniROCKET \cite{dempster2021minirocket}. Anomaly detection backbones include PatchTST, iTransformer, NSFormer, AnomTrans \cite{xu2021anomaly}, TranAD \cite{tuli2022tranad}, DCdetector \cite{yang2023dcdetector}, and TimesNet.

\begin{table*}[!t]
    \small
    \centering
    \caption{Real-world forecasting averaged across four prediction horizons per dataset (settings in Sec.~\ref{sec:settings}). +\method denotes backbones augmented with \method; bold marks the lower error per pair. Avg. Gain macro-averages relative improvements across datasets.}
    \label{tab:real_world_forecasting_summary}

    \setlength{\tabcolsep}{2.8pt}
    \renewcommand{\arraystretch}{1}

    \providecommand{\gc}{\tableourscell}
    \providecommand{\win}[1]{\textbf{#1}}
    \providecommand{\secbest}[1]{\underline{#1}}
    \providecommand{\dtwin}[2]{\gc\textbf{#1}\tableforecastchange{#2}{#1}}
    \providecommand{\dtsecond}[2]{\gc\underline{#1}\tableforecastchange{#2}{#1}}
    \providecommand{\dttie}[2]{\gc\textbf{#1}\tableforecastchange{#2}{#1}}

    \resizebox{\textwidth}{!}{%
    \begin{tabular}{@{}c c *{18}{c}@{}}
    \toprule
    \rowcolor{TableHeader}
    \multirow{2}{*}{\textbf{Dataset}} & \multirow{2}{*}{\textbf{Metric}} &
    \multicolumn{2}{c}{Informer (\citeyear{zhou2021informer})} &
    \multicolumn{2}{c}{Crossformer (\citeyear{zhang2022crossformer})} &
    \multicolumn{2}{c}{PatchTST (\citeyear{nie2023patchtst})} &
    \multicolumn{2}{c}{TimesNet (\citeyear{wu2023timesnet})} &
    \multicolumn{2}{c}{iTransformer (\citeyear{liu2023itransformer})} &
    \multicolumn{2}{c}{TimeMixer (\citeyear{wang2024timemixer})} &
    \multicolumn{2}{c}{WPMixer (\citeyear{murad2025wpmixer})} &
    \multicolumn{2}{c}{TimeFilter (\citeyear{hutimefilter})} &
    \multicolumn{2}{c}{MultiPatchFormer (\citeyear{naghashi2025multiscale})} \\
    \cmidrule(lr){3-4} \cmidrule(lr){5-6} \cmidrule(lr){7-8} \cmidrule(lr){9-10} \cmidrule(lr){11-12} \cmidrule(lr){13-14} \cmidrule(lr){15-16} \cmidrule(lr){17-18} \cmidrule(lr){19-20}
    \rowcolor{TableSubheader}
    & & Raw & \tableoursheader +\method & Raw & \tableoursheader +\method & Raw & \tableoursheader +\method & Raw & \tableoursheader +\method & Raw & \tableoursheader +\method & Raw & \tableoursheader +\method & Raw & \tableoursheader +\method & Raw & \tableoursheader +\method & Raw & \tableoursheader +\method \\
    \midrule

    \multirow{2}{*}{\textit{ETTh1}} & MSE & \secbest{1.337} & \dtwin{1.077}{1.337} & \secbest{0.449} & \dtwin{0.441}{0.449} & \secbest{0.459} & \dtwin{0.446}{0.459} & \secbest{0.534} & \dtwin{0.520}{0.534} & \secbest{0.448} & \dtwin{0.445}{0.448} & \secbest{0.458} & \dtwin{0.453}{0.458} & \win{0.431} & \dttie{0.431}{0.431} & \secbest{0.428} & \dtwin{0.427}{0.428} & \win{0.429} & \dttie{0.429}{0.429} \\
     & MAE & \secbest{0.823} & \dtwin{0.743}{0.823} & \secbest{0.445} & \dtwin{0.439}{0.445} & \secbest{0.432} & \dtwin{0.429}{0.432} & \secbest{0.492} & \dtwin{0.483}{0.492} & \win{0.431} & \dttie{0.431}{0.431} & \secbest{0.429} & \dtwin{0.427}{0.429} & \secbest{0.455} & \dtwin{0.454}{0.455} & \secbest{0.452} & \dtwin{0.449}{0.452} & \secbest{0.441} & \dtwin{0.439}{0.441} \\
    \addlinespace[0.5em]

    \multirow{2}{*}{\textit{ETTh2}} & MSE & \secbest{2.657} & \dtwin{1.392}{2.657} & \secbest{0.795} & \dtwin{0.626}{0.795} & \secbest{0.384} & \dtwin{0.378}{0.384} & \secbest{0.480} & \dtwin{0.456}{0.480} & \secbest{0.384} & \dtwin{0.381}{0.384} & \secbest{0.386} & \dtwin{0.380}{0.386} & \secbest{0.377} & \dtwin{0.373}{0.377} & \secbest{0.381} & \dtwin{0.377}{0.381} & \secbest{0.395} & \dtwin{0.384}{0.395} \\
     & MAE & \secbest{1.120} & \dtwin{0.827}{1.120} & \secbest{0.599} & \dtwin{0.521}{0.599} & \secbest{0.403} & \dtwin{0.398}{0.403} & \secbest{0.460} & \dtwin{0.447}{0.460} & \secbest{0.400} & \dtwin{0.399}{0.400} & \secbest{0.402} & \dtwin{0.399}{0.402} & \secbest{0.398} & \dtwin{0.396}{0.398} & \secbest{0.401} & \dtwin{0.399}{0.401} & \secbest{0.411} & \dtwin{0.405}{0.411} \\
    \addlinespace[0.5em]

    \multirow{2}{*}{\textit{ETTm1}} & MSE & \secbest{1.480} & \dtwin{1.194}{1.480} & \secbest{0.413} & \dtwin{0.394}{0.413} & \secbest{0.396} & \dtwin{0.393}{0.396} & \secbest{0.519} & \dtwin{0.512}{0.519} & \win{0.399} & \dttie{0.399}{0.399} & \secbest{0.393} & \dtwin{0.392}{0.393} & \secbest{0.391} & \dtwin{0.388}{0.391} & \secbest{0.394} & \dtwin{0.393}{0.394} & \secbest{0.402} & \dtwin{0.399}{0.402} \\
     & MAE & \secbest{0.855} & \dtwin{0.760}{0.855} & \secbest{0.404} & \dtwin{0.392}{0.404} & \secbest{0.387} & \dtwin{0.385}{0.387} & \secbest{0.474} & \dtwin{0.470}{0.474} & \win{0.390} & \dttie{0.390}{0.390} & \secbest{0.385} & \dtwin{0.383}{0.385} & \secbest{0.385} & \dtwin{0.384}{0.385} & \secbest{0.386} & \dtwin{0.385}{0.386} & \secbest{0.397} & \dtwin{0.395}{0.397} \\
    \addlinespace[0.5em]

    \multirow{2}{*}{\textit{ETTm2}} & MSE & \secbest{3.357} & \dtwin{1.634}{3.357} & \secbest{0.379} & \dtwin{0.351}{0.379} & \win{0.279} & \dttie{0.279}{0.279} & \secbest{0.348} & \dtwin{0.341}{0.348} & \win{0.284} & \dttie{0.284}{0.284} & \win{0.277} & \dttie{0.277}{0.277} & \secbest{0.275} & \dtwin{0.274}{0.275} & \win{0.283} & \dttie{0.283}{0.283} & \secbest{0.284} & \dtwin{0.282}{0.284} \\
     & MAE & \secbest{1.128} & \dtwin{0.738}{1.128} & \secbest{0.395} & \dtwin{0.378}{0.395} & \win{0.319} & \dttie{0.319}{0.319} & \secbest{0.364} & \dtwin{0.360}{0.364} & \win{0.321} & \dttie{0.321}{0.321} & \win{0.317} & \dttie{0.317}{0.317} & \secbest{0.317} & \dtwin{0.316}{0.317} & \win{0.321} & \dttie{0.321}{0.321} & \secbest{0.325} & \dtwin{0.324}{0.325} \\
    \addlinespace[0.5em]

    \multirow{2}{*}{\textit{Weather}} & MSE & \secbest{0.484} & \dtwin{0.381}{0.484} & \secbest{0.241} & \dtwin{0.239}{0.241} & \secbest{0.251} & \dtwin{0.247}{0.251} & \secbest{0.289} & \dtwin{0.284}{0.289} & \secbest{0.268} & \dtwin{0.265}{0.268} & \secbest{0.282} & \dtwin{0.244}{0.282} & \secbest{0.243} & \dtwin{0.242}{0.243} & \secbest{0.246} & \dtwin{0.244}{0.246} & \secbest{0.252} & \dtwin{0.250}{0.252} \\
     & MAE & \secbest{0.432} & \dtwin{0.373}{0.432} & \secbest{0.268} & \dtwin{0.266}{0.268} & \secbest{0.270} & \dtwin{0.266}{0.270} & \secbest{0.305} & \dtwin{0.302}{0.305} & \secbest{0.280} & \dtwin{0.276}{0.280} & \secbest{0.301} & \dtwin{0.264}{0.301} & \secbest{0.264} & \dtwin{0.263}{0.264} & \secbest{0.267} & \dtwin{0.264}{0.267} & \secbest{0.270} & \dtwin{0.268}{0.270} \\
    \addlinespace[0.5em]

    \multirow{2}{*}{\textit{Electricity}} & MSE & \secbest{1.528} & \dtwin{0.489}{1.528} & \secbest{0.213} & \dtwin{0.211}{0.213} & \secbest{0.215} & \dtwin{0.210}{0.215} & \secbest{0.206} & \dtwin{0.203}{0.206} & \secbest{0.215} & \dtwin{0.213}{0.215} & \win{0.211} & \dtsecond{0.212}{0.211} & \secbest{0.208} & \dtwin{0.198}{0.208} & \win{0.220} & \dttie{0.220}{0.220} & \secbest{0.320} & \dtwin{0.297}{0.320} \\
     & MAE & \secbest{0.989} & \dtwin{0.459}{0.989} & \secbest{0.290} & \dtwin{0.287}{0.290} & \secbest{0.290} & \dtwin{0.286}{0.290} & \secbest{0.304} & \dtwin{0.301}{0.304} & \secbest{0.290} & \dtwin{0.273}{0.290} & \win{0.282} & \dttie{0.282}{0.282} & \secbest{0.284} & \dtwin{0.278}{0.284} & \win{0.294} & \dttie{0.294}{0.294} & \secbest{0.377} & \dtwin{0.364}{0.377} \\
    \addlinespace[0.5em]

    \multirow{2}{*}{\textit{ILI}} & MSE & \secbest{7.130} & \dtwin{6.235}{7.130} & \secbest{5.035} & \dtwin{4.321}{5.035} & \secbest{3.321} & \dtwin{2.942}{3.321} & \secbest{6.045} & \dtwin{4.872}{6.045} & \secbest{3.163} & \dtwin{2.994}{3.163} & \secbest{3.182} & \dtwin{3.151}{3.182} & \secbest{3.081} & \dtwin{2.977}{3.081} & \secbest{2.407} & \dtwin{2.223}{2.407} & \secbest{2.804} & \dtwin{2.550}{2.804} \\
     & MAE & \secbest{1.901} & \dtwin{1.772}{1.901} & \secbest{1.542} & \dtwin{1.406}{1.542} & \secbest{1.110} & \dtwin{1.064}{1.110} & \secbest{1.303} & \dtwin{1.223}{1.303} & \secbest{1.069} & \dtwin{1.048}{1.069} & \secbest{1.148} & \dtwin{1.142}{1.148} & \secbest{1.065} & \dtwin{1.038}{1.065} & \secbest{0.965} & \dtwin{0.931}{0.965} & \win{0.991} & \dtsecond{1.001}{0.991} \\
    \addlinespace[0.5em]

    \multirow{2}{*}{\textit{Exchange Rate}} & MSE & \secbest{3.533} & \dtwin{1.434}{3.533} & \secbest{0.933} & \dtwin{0.851}{0.933} & \secbest{0.458} & \dtwin{0.444}{0.458} & \secbest{0.478} & \dtwin{0.477}{0.478} & \secbest{0.451} & \dtwin{0.438}{0.451} & \secbest{0.445} & \dtwin{0.441}{0.445} & \secbest{0.445} & \dtwin{0.441}{0.445} & \secbest{0.446} & \dtwin{0.444}{0.446} & \secbest{0.526} & \dtwin{0.469}{0.526} \\
     & MAE & \secbest{1.410} & \dtwin{0.875}{1.410} & \secbest{0.657} & \dtwin{0.626}{0.657} & \secbest{0.457} & \dtwin{0.452}{0.457} & \win{0.478} & \dtsecond{0.479}{0.478} & \secbest{0.455} & \dtwin{0.448}{0.455} & \secbest{0.446} & \dtwin{0.445}{0.446} & \secbest{0.449} & \dtwin{0.445}{0.449} & \secbest{0.452} & \dtwin{0.448}{0.452} & \secbest{0.486} & \dtwin{0.463}{0.486} \\
    \addlinespace[0.5em]

    \multirow{2}{*}{\textit{ECG}} & MSE & \secbest{1.306} & \dtwin{1.052}{1.306} & \secbest{0.711} & \dtwin{0.682}{0.711} & \secbest{0.750} & \dtwin{0.722}{0.750} & \secbest{0.745} & \dtwin{0.740}{0.745} & \secbest{0.762} & \dtwin{0.740}{0.762} & \secbest{0.715} & \dtwin{0.710}{0.715} & \secbest{0.725} & \dtwin{0.694}{0.725} & \secbest{0.737} & \dtwin{0.730}{0.737} & \secbest{0.750} & \dtwin{0.733}{0.750} \\
     & MAE & \secbest{0.826} & \dtwin{0.693}{0.826} & \secbest{0.518} & \dtwin{0.505}{0.518} & \secbest{0.533} & \dtwin{0.517}{0.533} & \secbest{0.532} & \dtwin{0.529}{0.532} & \secbest{0.537} & \dtwin{0.528}{0.537} & \secbest{0.515} & \dtwin{0.512}{0.515} & \secbest{0.518} & \dtwin{0.502}{0.518} & \secbest{0.524} & \dtwin{0.519}{0.524} & \secbest{0.532} & \dtwin{0.524}{0.532} \\

    \midrule
    \rowcolor{TableSummary}
    \multicolumn{2}{c}{\textbf{Avg. Gain}} &
    \multicolumn{2}{c}{\textbf{35.4\% / 23.3\%}} &
    \multicolumn{2}{c}{\textbf{7.1\% / 4.4\%}} &
    \multicolumn{2}{c}{\textbf{3.0\% / 1.5\%}} &
    \multicolumn{2}{c}{\textbf{3.8\% / 1.7\%}} &
    \multicolumn{2}{c}{\textbf{1.6\% / 1.4\%}} &
    \multicolumn{2}{c}{\textbf{2.0\% / 1.7\%}} &
    \multicolumn{2}{c}{\textbf{1.8\% / 1.1\%}} &
    \multicolumn{2}{c}{\textbf{1.3\% / 0.9\%}} &
    \multicolumn{2}{c}{\textbf{3.8\% / 1.3\%}} \\

    \bottomrule
    \end{tabular}%
    }
\end{table*}

\noindent\textbf{Implementation Details.}
All models retain their published backbone architecture, default task head, and default optimization hyperparameters. We implement the paired Raw and +\method variants in PyTorch, optimize with Adam~\cite{Adam}, and run experiments on NVIDIA A800 GPUs. Unless stated otherwise, \method fixes $p_{\min}=0.05$ and $p_{\max}=0.50$, and initializes the learnable parameters at $\alpha=10$, $\mathbf{b}_s=0$, and $\gamma=1$. The scorer and mapper are active only during training; evaluation disables them and all stochastic dropout, so inference follows the deterministic backbone graph. Dataset-specific splits, window construction, corruption generation, decoder usage, controlled baselines, and random-seed protocols are stated with the corresponding RQ below and tabulated in Supplementary Appendices~A--B. Evaluation reports MSE/MAE for forecasting, top-1 accuracy for classification, and point-adjusted precision, recall, and F1 for anomaly detection.

\subsection{RQ1: Robustness under Reliability Heterogeneity}
\label{sec:rq1}

\noindent\textbf{Performance under Controlled Corruptions.}
Table~\ref{tab:synthetic_forecasting_summary} summarizes the forecasting performance on \textit{Synth-12} under the layered Gaussian, heavy-tailed, and point-wise missing corruptions of Table~\ref{tab:synthetic_signal_corruption_definitions}. Each corrupted window is paired with its clean target, so the model trains on the corrupted observation but is evaluated against the clean future. The benchmark contains 33,600 windows with $\sigma\in\{0.1,0.3,0.5,0.7,0.9\}$ and horizons $H\in\{96,192,336,720\}$. Across all backbones and corruption scales, \method lowers horizon-averaged MSE; Informer \cite{zhou2021informer} shows the largest reduction, 46.0\% lower average MSE and up to 48.2\% at $\sigma=0.3$. Crossformer \cite{zhang2022crossformer} achieves an 11.9\% gain, while TimesNet \cite{wu2023timesnet} improves by 3.6\%. PatchTST, iTransformer, and TimeMixer achieve gains of 0.8\%--2.9\%. Newer backbones follow: WPMixer \cite{murad2025wpmixer}, TimeFilter \cite{hutimefilter}, and MultiPatchFormer \cite{naghashi2025multiscale} reduce average MSE by 1.7\%, 1.2\%, and 3.6\%, improving at all scales. These gains across backbone designs support a separate capacity-modulation path: sample-adaptive regularization complements temporal feature modeling without changing the backbone's inference architecture.

\noindent\textbf{Non-Monotonic Corruption Response.}
Figure~\ref{fig:paradox_comparison} plots the horizon-averaged MSE of Informer and Crossformer against the corruption scale $\sigma \in [0.1, 0.9]$. Error does not grow monotonically with noise for either baseline: Informer rises to a mid-noise maximum and then \emph{decreases} at larger $\sigma$, while Crossformer decreases throughout. These different baseline trajectories motivate comparing Raw and +\method at each matched corruption scale. \method (blue) remains below the corresponding baseline at every plotted level, flattening the Informer mid-noise peak and reducing Crossformer error throughout. The matched gains across these different trajectories provide a more consistent robustness indicator than the slope of either error curve alone.

\begin{figure}[!ht]
    \centering
    \includegraphics[width=\linewidth]{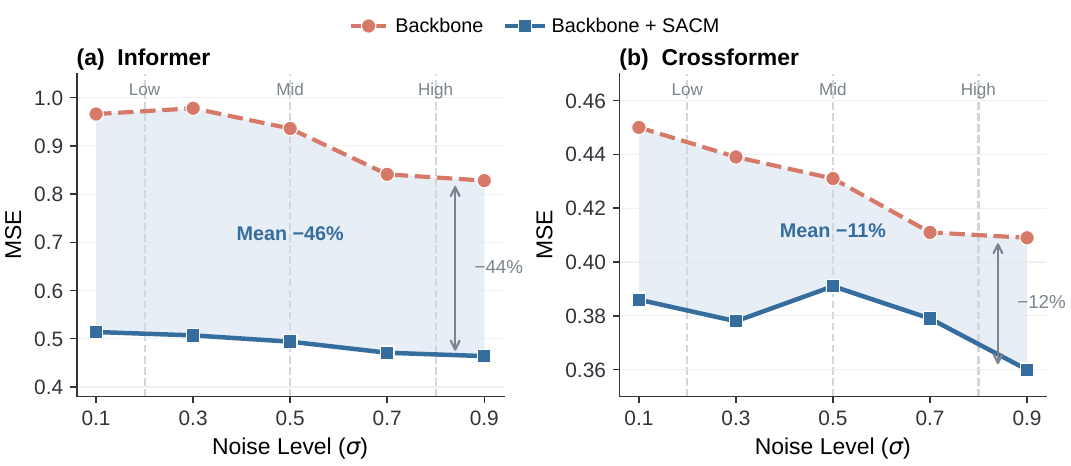}
    \caption{\textbf{Error trajectories under increasing corruption.} (a) Informer exhibits a mid-noise error peak under fixed capacity (red), which is absent with \method (blue). (b) Crossformer shows decreasing baseline error across the plotted corruption levels; \method remains lower and smoother.}
    \label{fig:paradox_comparison}
\end{figure}

\subsection{RQ2: Real-World Forecasting \& Capacity Allocation}
\label{sec:rq2}

\noindent\textbf{Long-Term Forecasting Performance.}
Table~\ref{tab:real_world_forecasting_summary} compares nine backbones on nine real-world benchmarks using chronological train/validation/test splits and the forecasting settings of Supplementary Appendix~A. \method achieves gains across many backbones and horizons, with smaller or tied changes on ETTm2. Informer reduces average MSE by 68.0\% on \textit{Electricity}, 59.4\% on \textit{Exchange Rate}, and 47.6\% on \textit{ETTh2}. TimeMixer improves by 13.7\% on \textit{Weather}, while MultiPatchFormer gains 7.2\% on \textit{Electricity}. On \textit{ILI}, WPMixer, TimeFilter, and MultiPatchFormer reduce MSE by 3.4\%, 7.6\%, and 9.1\%, respectively. On \textit{ECG} waveforms, Informer and PatchTST improve MSE by 19.4\% and 3.8\%. Larger gains on Informer and smaller gains on newer backbones indicate different regularization sensitivity, while improvements across domains support its use with diverse temporal representations.

\noindent\textbf{Sample-Wise Capacity Allocation Analysis.}
Table~\ref{tab:dropout_rate_distribution} reports observed sample-wise dropout rates $p_i$ from the final convergence epochs on the training windows used for forecasting. \textit{ILI} has a lower mean dropout rate than \textit{ETTm2} ($0.267$ vs. $0.329$), but its variance is fourfold larger ($0.0194$ vs. $0.0049$). Thus, \method retains more activations on average on \textit{ILI} and varies retention more across its windows. The distributions illustrate both regularization strength and sample allocation.

\begin{table}[H]
    \centering
    \small
    \setlength{\tabcolsep}{2.5pt}
    \caption{Distribution of sample-adaptive dropout rates during final epochs. Statistics after batch-relative rate mapping.}
    \label{tab:dropout_rate_distribution}
    \renewcommand{\arraystretch}{1.12}
    \begin{tabular*}{\columnwidth}{@{\extracolsep{\fill}} l c c c l @{}}
        \toprule
        \rowcolor{TableHeader}
        \textbf{Dataset} & \textbf{Mean $\pm$ Std.} & \textbf{Variance} & \textbf{Rel. Var.} & \textbf{Allocation} \\
        \midrule
        \textit{ETTm2} & $0.329 \pm 0.070$ & 0.0049 & $1.0\times$ & Lower variation \\
        \textit{ILI}   & $0.267 \pm 0.139$ & 0.0194 & $4.0\times$ & Higher variation \\
        \bottomrule
    \end{tabular*}
\end{table}

\subsection{RQ3: Data-Regime and Cross-Domain Generalization}
\label{sec:rq3}

{\parfillskip=0pt\relax
\noindent\textbf{Varying Training-Data Fractions.}
Table~\ref{tab:generalization_results} evaluates PatchTST on \textit{ILI} using 10\%, 25\%, 50\%, 75\%, and 100\% of the chronological training split, with validation and test windows unchanged. \method lowers both errors at every fraction and improves monotonically as data grows (MAE $1.295 \to 1.064$, MSE $4.074 \to 2.942$), whereas the unregularized model stalls and then regresses: its MSE holds at $3.169$ from 50\% to 75\% and rises to $3.321$ at 100\%. \method's relative gain therefore peaks at full data ($4.14\%$ MAE, $11.41\%$ MSE), so more training data strengthens its advantage; the MSE-dominant reduction indicates suppression of the large deviations that dominate squared error. Notably, \method at 50\% data already beats the raw model at 100\% (MSE $3.143$ vs.\ $3.321$).\par}

{\parfillskip=0pt\relax
\noindent\textbf{Zero-Shot Cross-Domain Transfer.}
We test whether source-domain residual regularization aids cross-domain generalization. We apply a source-trained PatchTST to the target. Table~\ref{tab:generalization_results} lists eight ETT transfer directions. No target-domain fine-tuning or label selection is used. \method lowers both errors in all eight directions and never degrades a metric (average MAE $0.402 \to 0.395$, MSE $0.407 \to 0.404$), with its largest gain (MAE $0.519 \to 0.501$) on \textit{ETTh2}~$\rightarrow$~\textit{ETTh1}, the highest-error direction. The scorer and dropout are disabled at evaluation, so these gains add no inference cost. Because the target pass is identical in both settings, the gains stem from source training alone.\par}

\begin{table}[H]
    \centering
    \footnotesize
    \caption{PatchTST performance across training-data fractions (ILI, top) and zero-shot OOD transfer (ETT, bottom). Shaded cells denote \method; arrows indicate relative changes.}
    \label{tab:generalization_results}
    \setlength{\tabcolsep}{3.2pt}
    \renewcommand{\arraystretch}{1.00}
    \setlength{\aboverulesep}{1pt}
    \setlength{\belowrulesep}{1pt}
    \providecommand{\avgimp}[2]{\tableforecastchange{#1}{#2}}

    \resizebox{0.94\columnwidth}{!}{%
        \begin{tabular}{@{}l cc cc@{}}
            \toprule
            \rowcolor{TableHeader}
            \multicolumn{5}{c}{\textbf{Training-Data Fractions on \textit{ILI}}} \\
            \cmidrule(lr){1-5}
            \rowcolor{TableHeader}
            \textbf{Scenario}
            & \multicolumn{2}{c}{\textbf{Raw}}
            & \multicolumn{2}{c}{\tableoursheader\textbf{+\method}} \\
            \cmidrule(lr){2-3}\cmidrule(lr){4-5}
            \rowcolor{TableSubheader}
            & \textbf{MAE $\downarrow$} & \textbf{MSE $\downarrow$}
            & \tableoursheader\textbf{MAE $\downarrow$} & \tableoursheader\textbf{MSE $\downarrow$} \\
            \midrule
            10\% Data
            & \underline{1.333} & \underline{4.185}
            & \tableourscell\textbf{1.295}\avgimp{1.333}{1.295} & \tableourscell\textbf{4.074}\avgimp{4.185}{4.074} \\
            25\% Data
            & \underline{1.156} & \underline{3.201}
            & \tableourscell\textbf{1.144}\avgimp{1.156}{1.144} & \tableourscell\textbf{3.195}\avgimp{3.201}{3.195} \\
            50\% Data
            & \underline{1.116} & \underline{3.169}
            & \tableourscell\textbf{1.106}\avgimp{1.116}{1.106} & \tableourscell\textbf{3.143}\avgimp{3.169}{3.143} \\
            75\% Data
            & \underline{1.111} & \underline{3.169}
            & \tableourscell\textbf{1.096}\avgimp{1.111}{1.096} & \tableourscell\textbf{3.128}\avgimp{3.169}{3.128} \\
            100\% Data
            & \underline{1.110} & \underline{3.321}
            & \tableourscell\textbf{1.064}\avgimp{1.110}{1.064} & \tableourscell\textbf{2.942}\avgimp{3.321}{2.942} \\
            \midrule
            \textbf{Average}
            & \underline{1.165} & \underline{3.409}
            & \textbf{1.141}\avgimp{1.165}{1.141} & \textbf{3.296}\avgimp{3.409}{3.296} \\
            \bottomrule
            \noalign{\vskip 4pt}
            \toprule
            \rowcolor{TableHeader}
            \multicolumn{5}{c}{\textbf{Zero-Shot Generalization on \textit{ETT}}} \\
            \cmidrule(lr){1-5}
            \rowcolor{TableHeader}
            \textbf{Scenario}
            & \multicolumn{2}{c}{\textbf{Raw}}
            & \multicolumn{2}{c}{\tableoursheader\textbf{+\method}} \\
            \cmidrule(lr){2-3}\cmidrule(lr){4-5}
            \rowcolor{TableSubheader}
            & \textbf{MAE $\downarrow$} & \textbf{MSE $\downarrow$}
            & \tableoursheader\textbf{MAE $\downarrow$} & \tableoursheader\textbf{MSE $\downarrow$} \\
            \midrule
            \textit{ETTh1} $\rightarrow$ \textit{ETTh2}
            & \underline{0.391} & \underline{0.406}
            & \tableourscell\textbf{0.387}\avgimp{0.391}{0.387} & \tableourscell\textbf{0.404}\avgimp{0.406}{0.404} \\
            \textit{ETTh1} $\rightarrow$ \textit{ETTm2}
            & \underline{0.319} & \underline{0.359}
            & \tableourscell\textbf{0.318}\avgimp{0.319}{0.318} & \tableourscell\textbf{0.358}\avgimp{0.359}{0.358} \\
            \textit{ETTh2} $\rightarrow$ \textit{ETTh1}
            & \underline{0.519} & \underline{0.480}
            & \tableourscell\textbf{0.501}\avgimp{0.519}{0.501} & \tableourscell\textbf{0.473}\avgimp{0.480}{0.473} \\
            \textit{ETTh2} $\rightarrow$ \textit{ETTm2}
            & \underline{0.325} & \underline{0.364}
            & \tableourscell\textbf{0.317}\avgimp{0.325}{0.317} & \tableourscell\textbf{0.358}\avgimp{0.364}{0.358} \\
            \textit{ETTm1} $\rightarrow$ \textit{ETTh2}
            & \underline{0.457} & \underline{0.442}
            & \tableourscell\textbf{0.448}\avgimp{0.457}{0.448} & \tableourscell\textbf{0.439}\avgimp{0.442}{0.439} \\
            \textit{ETTm1} $\rightarrow$ \textit{ETTm2}
            & \underline{0.305} & \underline{0.331}
            & \tableourscell\textbf{0.302}\avgimp{0.305}{0.302} & \tableourscell\textbf{0.330}\avgimp{0.331}{0.330} \\
            \textit{ETTm2} $\rightarrow$ \textit{ETTh2}
            & \underline{0.424} & \underline{0.428}
            & \tableourscell\textbf{0.420}\avgimp{0.424}{0.420} & \tableourscell\textbf{0.427}\avgimp{0.428}{0.427} \\
            \textit{ETTm2} $\rightarrow$ \textit{ETTm1}
            & \underline{0.476} & \underline{0.444}
            & \tableourscell\textbf{0.470}\avgimp{0.476}{0.470} & \tableourscell\textbf{0.442}\avgimp{0.444}{0.442} \\
            \midrule
            \textbf{Average}
            & \underline{0.402} & \underline{0.407}
            & \textbf{0.395}\avgimp{0.402}{0.395} & \textbf{0.404}\avgimp{0.407}{0.404} \\
            \bottomrule
        \end{tabular}%
    }
\end{table}

\begin{table*}[!t]
    \centering
    \scriptsize
    \caption{Classification accuracy (\%) on 32 datasets and six backbones. Shaded +\method cells denote \method and show changes from Raw in percentage points (pp); bold marks the higher accuracy within each pair.}
    \label{tab:classification_results}
    \setlength{\tabcolsep}{3.8pt}
    \renewcommand{\arraystretch}{1.00}
    \providecommand{\gc}{\tableourscell}
    \providecommand{\clsraw}[2]{%
        \ifdim#1pt>#2pt \textbf{#1}\else
            \ifdim#1pt<#2pt \underline{#1}\else \textbf{#1}\fi
        \fi
    }
    \providecommand{\clsdt}[2]{%
        \gc
        \ifdim#2pt>#1pt \textbf{#2}\else
            \ifdim#2pt<#1pt \underline{#2}\else \textbf{#2}\fi
        \fi
        \tableppchange{#1}{#2}%
    }
    \resizebox{\textwidth}{!}{%
    \begin{tabular}{@{}l *{12}{c}@{}}
        \toprule
        \rowcolor{TableHeader}
        \multirow{2}{*}{\textbf{Dataset}} & \multicolumn{2}{c}{\textbf{iTransformer}} & \multicolumn{2}{c}{\textbf{PatchTST}} & \multicolumn{2}{c}{\textbf{NSFormer}} & \multicolumn{2}{c}{\textbf{TimesNet}} & \multicolumn{2}{c}{\textbf{InceptionTime}} & \multicolumn{2}{c}{\textbf{MiniROCKET}} \\
        \cmidrule(lr){2-3} \cmidrule(lr){4-5} \cmidrule(lr){6-7} \cmidrule(lr){8-9} \cmidrule(lr){10-11} \cmidrule(lr){12-13}
        \rowcolor{TableSubheader}
        & Raw & \tableoursheader +\method & Raw & \tableoursheader +\method & Raw & \tableoursheader +\method & Raw & \tableoursheader +\method & Raw & \tableoursheader +\method & Raw & \tableoursheader +\method \\
        \midrule
        \textit{ArticularyWordRecognition} & \clsraw{97.67}{98.33} & \clsdt{97.67}{98.33} & \clsraw{97.33}{97.67} & \clsdt{97.33}{97.67} & \clsraw{97.00}{98.00} & \clsdt{97.00}{98.00} & \clsraw{97.33}{97.67} & \clsdt{97.33}{97.67} & \clsraw{99.00}{99.00} & \clsdt{99.00}{99.00} & \clsraw{96.67}{97.67} & \clsdt{96.67}{97.67} \\
        \textit{AtrialFibrillation} & \clsraw{33.33}{40.00} & \clsdt{33.33}{40.00} & \clsraw{46.67}{53.33} & \clsdt{46.67}{53.33} & \clsraw{33.33}{40.00} & \clsdt{33.33}{40.00} & \clsraw{33.33}{33.33} & \clsdt{33.33}{33.33} & \clsraw{33.33}{33.33} & \clsdt{33.33}{33.33} & \clsraw{40.00}{40.00} & \clsdt{40.00}{40.00} \\
        \textit{BasicMotions} & \clsraw{92.50}{95.00} & \clsdt{92.50}{95.00} & \clsraw{67.50}{70.00} & \clsdt{67.50}{70.00} & \clsraw{85.00}{87.50} & \clsdt{85.00}{87.50} & \clsraw{90.00}{90.00} & \clsdt{90.00}{90.00} & \clsraw{25.00}{25.00} & \clsdt{25.00}{25.00} & \clsraw{85.00}{92.50} & \clsdt{85.00}{92.50} \\
        \textit{CharacterTrajectories} & \clsraw{98.40}{98.61} & \clsdt{98.40}{98.61} & \clsraw{97.84}{97.98} & \clsdt{97.84}{97.98} & \clsraw{98.33}{98.33} & \clsdt{98.33}{98.33} & \clsraw{98.33}{98.54} & \clsdt{98.33}{98.54} & \clsraw{99.79}{99.79} & \clsdt{99.79}{99.79} & \clsraw{96.45}{97.49} & \clsdt{96.45}{97.49} \\
        \textit{Cricket} & \clsraw{87.50}{88.89} & \clsdt{87.50}{88.89} & \clsraw{84.72}{90.28} & \clsdt{84.72}{90.28} & \clsraw{90.28}{98.61} & \clsdt{90.28}{98.61} & \clsraw{79.17}{84.72} & \clsdt{79.17}{84.72} & \clsraw{98.61}{98.61} & \clsdt{98.61}{98.61} & \clsraw{93.06}{93.06} & \clsdt{93.06}{93.06} \\
        \textit{DuckDuckGeese} & \clsraw{38.00}{42.00} & \clsdt{38.00}{42.00} & \clsraw{24.00}{26.00} & \clsdt{24.00}{26.00} & \clsraw{24.00}{28.00} & \clsdt{24.00}{28.00} & \clsraw{56.00}{56.00} & \clsdt{56.00}{56.00} & \clsraw{54.00}{60.00} & \clsdt{54.00}{60.00} & \clsraw{24.00}{30.00} & \clsdt{24.00}{30.00} \\
        \textit{EigenWorms} & \clsraw{47.33}{49.62} & \clsdt{47.33}{49.62} & \clsraw{50.38}{55.73} & \clsdt{50.38}{55.73} & \clsraw{52.67}{56.49} & \clsdt{52.67}{56.49} & \clsraw{38.17}{45.04} & \clsdt{38.17}{45.04} & \clsraw{41.98}{61.83} & \clsdt{41.98}{61.83} & \clsraw{73.28}{73.28} & \clsdt{73.28}{73.28} \\
        \textit{Epilepsy} & \clsraw{76.81}{78.99} & \clsdt{76.81}{78.99} & \clsraw{83.33}{89.13} & \clsdt{83.33}{89.13} & \clsraw{49.28}{56.52} & \clsdt{49.28}{56.52} & \clsraw{69.57}{73.91} & \clsdt{69.57}{73.91} & \clsraw{95.65}{96.38} & \clsdt{95.65}{96.38} & \clsraw{92.03}{92.03} & \clsdt{92.03}{92.03} \\
        \textit{EthanolConcentration} & \clsraw{31.56}{34.22} & \clsdt{31.56}{34.22} & \clsraw{28.14}{29.66} & \clsdt{28.14}{29.66} & \clsraw{32.32}{32.32} & \clsdt{32.32}{32.32} & \clsraw{29.66}{31.94} & \clsdt{29.66}{31.94} & \clsraw{33.08}{35.36} & \clsdt{33.08}{35.36} & \clsraw{27.38}{30.42} & \clsdt{27.38}{30.42} \\
        \textit{ERing} & \clsraw{93.70}{94.07} & \clsdt{93.70}{94.07} & \clsraw{94.07}{94.07} & \clsdt{94.07}{94.07} & \clsraw{93.70}{94.44} & \clsdt{93.70}{94.44} & \clsraw{92.96}{93.33} & \clsdt{92.96}{93.33} & \clsraw{16.67}{33.33} & \clsdt{16.67}{33.33} & \clsraw{82.22}{85.19} & \clsdt{82.22}{85.19} \\
        \textit{FaceDetection} & \clsraw{67.51}{68.10} & \clsdt{67.51}{68.10} & \clsraw{65.21}{66.06} & \clsdt{65.21}{66.06} & \clsraw{68.67}{68.70} & \clsdt{68.67}{68.70} & \clsraw{66.97}{68.59} & \clsdt{66.97}{68.59} & \clsraw{65.72}{68.08} & \clsdt{65.72}{68.08} & \clsraw{50.28}{51.53} & \clsdt{50.28}{51.53} \\
        \textit{FingerMovements} & \clsraw{57.00}{62.00} & \clsdt{57.00}{62.00} & \clsraw{55.00}{57.00} & \clsdt{55.00}{57.00} & \clsraw{56.00}{58.00} & \clsdt{56.00}{58.00} & \clsraw{55.00}{61.00} & \clsdt{55.00}{61.00} & \clsraw{60.00}{63.00} & \clsdt{60.00}{63.00} & \clsraw{54.00}{58.00} & \clsdt{54.00}{58.00} \\
        \textit{HandMovementDirection} & \clsraw{47.30}{52.70} & \clsdt{47.30}{52.70} & \clsraw{56.76}{56.76} & \clsdt{56.76}{56.76} & \clsraw{58.11}{62.16} & \clsdt{58.11}{62.16} & \clsraw{58.11}{59.46} & \clsdt{58.11}{59.46} & \clsraw{43.24}{48.65} & \clsdt{43.24}{48.65} & \clsraw{37.84}{43.24} & \clsdt{37.84}{43.24} \\
        \textit{Handwriting} & \clsraw{26.59}{27.88} & \clsdt{26.59}{27.88} & \clsraw{25.41}{28.35} & \clsdt{25.41}{28.35} & \clsraw{33.53}{34.71} & \clsdt{33.53}{34.71} & \clsraw{23.29}{24.82} & \clsdt{23.29}{24.82} & \clsraw{64.94}{65.65} & \clsdt{64.94}{65.65} & \clsraw{30.59}{30.59} & \clsdt{30.59}{30.59} \\
        \textit{Heartbeat} & \clsraw{75.12}{76.10} & \clsdt{75.12}{76.10} & \clsraw{72.20}{72.68} & \clsdt{72.20}{72.68} & \clsraw{72.20}{72.68} & \clsdt{72.20}{72.68} & \clsraw{72.20}{72.68} & \clsdt{72.20}{72.68} & \clsraw{77.07}{78.05} & \clsdt{77.07}{78.05} & \clsraw{72.20}{72.20} & \clsdt{72.20}{72.20} \\
        \textit{InsectWingbeat} & \clsraw{70.09}{70.92} & \clsdt{70.09}{70.92} & \clsraw{53.07}{53.56} & \clsdt{53.07}{53.56} & \clsraw{56.11}{58.26} & \clsdt{56.11}{58.26} & \clsraw{63.10}{65.31} & \clsdt{63.10}{65.31} & \clsraw{70.51}{70.54} & \clsdt{70.51}{70.54} & \clsraw{48.13}{48.58} & \clsdt{48.13}{48.58} \\
        \midrule
        \textit{JapaneseVowels} & \clsraw{98.38}{98.38} & \clsdt{98.38}{98.38} & \clsraw{91.89}{91.89} & \clsdt{91.89}{91.89} & \clsraw{94.86}{94.86} & \clsdt{94.86}{94.86} & \clsraw{97.03}{97.30} & \clsdt{97.03}{97.30} & \clsraw{97.03}{97.57} & \clsdt{97.03}{97.57} & \clsraw{97.03}{97.03} & \clsdt{97.03}{97.03} \\
        \textit{Libras} & \clsraw{88.33}{88.33} & \clsdt{88.33}{88.33} & \clsraw{76.67}{78.89} & \clsdt{76.67}{78.89} & \clsraw{77.22}{77.78} & \clsdt{77.22}{77.78} & \clsraw{82.22}{85.56} & \clsdt{82.22}{85.56} & \clsraw{76.67}{78.33} & \clsdt{76.67}{78.33} & \clsraw{75.00}{75.56} & \clsdt{75.00}{75.56} \\
        \textit{LSST} & \clsraw{58.80}{60.02} & \clsdt{58.80}{60.02} & \clsraw{55.56}{56.49} & \clsdt{55.56}{56.49} & \clsraw{59.41}{61.11} & \clsdt{59.41}{61.11} & \clsraw{56.20}{57.30} & \clsdt{56.20}{57.30} & \clsraw{54.34}{54.34} & \clsdt{54.34}{54.34} & \clsraw{56.20}{56.49} & \clsdt{56.20}{56.49} \\
        \textit{MotorImagery} & \clsraw{61.00}{67.00} & \clsdt{61.00}{67.00} & \clsraw{61.00}{63.00} & \clsdt{61.00}{63.00} & \clsraw{59.00}{65.00} & \clsdt{59.00}{65.00} & \clsraw{62.00}{62.00} & \clsdt{62.00}{62.00} & \clsraw{66.00}{66.00} & \clsdt{66.00}{66.00} & \clsraw{52.00}{53.00} & \clsdt{52.00}{53.00} \\
        \textit{NATOPS} & \clsraw{90.00}{91.67} & \clsdt{90.00}{91.67} & \clsraw{74.44}{77.22} & \clsdt{74.44}{77.22} & \clsraw{82.78}{83.33} & \clsdt{82.78}{83.33} & \clsraw{91.67}{93.33} & \clsdt{91.67}{93.33} & \clsraw{95.56}{95.56} & \clsdt{95.56}{95.56} & \clsraw{82.22}{87.78} & \clsdt{82.22}{87.78} \\
        \textit{PenDigits} & \clsraw{97.97}{98.17} & \clsdt{97.97}{98.17} & \clsraw{97.14}{97.60} & \clsdt{97.14}{97.60} & \clsraw{98.60}{98.77} & \clsdt{98.60}{98.77} & \clsraw{98.08}{98.17} & \clsdt{98.08}{98.17} & \clsraw{98.54}{99.03} & \clsdt{98.54}{99.03} & \clsraw{98.14}{98.17} & \clsdt{98.14}{98.17} \\
        \textit{PEMS-SF} & \clsraw{87.86}{89.60} & \clsdt{87.86}{89.60} & \clsraw{84.97}{88.44} & \clsdt{84.97}{88.44} & \clsraw{87.28}{90.75} & \clsdt{87.28}{90.75} & \clsraw{80.35}{83.82} & \clsdt{80.35}{83.82} & \clsraw{74.45}{75.72} & \clsdt{74.45}{75.72} & \clsraw{64.16}{68.79} & \clsdt{64.16}{68.79} \\
        \textit{PhonemeSpectra} & \clsraw{11.03}{11.27} & \clsdt{11.03}{11.27} & \clsraw{11.15}{11.84} & \clsdt{11.15}{11.84} & \clsraw{13.42}{13.81} & \clsdt{13.42}{13.81} & \clsraw{10.83}{11.90} & \clsdt{10.83}{11.90} & \clsraw{30.96}{32.12} & \clsdt{30.96}{32.12} & \clsraw{16.22}{19.48} & \clsdt{16.22}{19.48} \\
        \textit{RacketSports} & \clsraw{80.92}{82.24} & \clsdt{80.92}{82.24} & \clsraw{78.95}{78.95} & \clsdt{78.95}{78.95} & \clsraw{83.55}{84.21} & \clsdt{83.55}{84.21} & \clsraw{76.97}{77.63} & \clsdt{76.97}{77.63} & \clsraw{90.79}{90.79} & \clsdt{90.79}{90.79} & \clsraw{75.00}{76.32} & \clsdt{75.00}{76.32} \\
        \textit{SelfRegulationSCP1} & \clsraw{92.15}{92.83} & \clsdt{92.15}{92.83} & \clsraw{82.25}{83.96} & \clsdt{82.25}{83.96} & \clsraw{80.20}{82.59} & \clsdt{80.20}{82.59} & \clsraw{91.13}{92.83} & \clsdt{91.13}{92.83} & \clsraw{87.71}{87.71} & \clsdt{87.71}{87.71} & \clsraw{54.95}{56.66} & \clsdt{54.95}{56.66} \\
        \textit{SelfRegulationSCP2} & \clsraw{58.33}{60.56} & \clsdt{58.33}{60.56} & \clsraw{54.44}{55.00} & \clsdt{54.44}{55.00} & \clsraw{53.89}{53.89} & \clsdt{53.89}{53.89} & \clsraw{55.56}{56.11} & \clsdt{55.56}{56.11} & \clsraw{58.33}{59.44} & \clsdt{58.33}{59.44} & \clsraw{58.33}{58.89} & \clsdt{58.33}{58.89} \\
        \textit{SpokenArabicDigits} & \clsraw{98.45}{98.64} & \clsdt{98.45}{98.64} & \clsraw{97.36}{97.50} & \clsdt{97.36}{97.50} & \clsraw{99.45}{99.50} & \clsdt{99.45}{99.50} & \clsraw{98.41}{98.73} & \clsdt{98.41}{98.73} & \clsraw{99.73}{99.82} & \clsdt{99.73}{99.82} & \clsraw{98.27}{98.73} & \clsdt{98.27}{98.73} \\
        \textit{StandWalkJump} & \clsraw{40.00}{40.00} & \clsdt{40.00}{40.00} & \clsraw{46.67}{53.33} & \clsdt{46.67}{53.33} & \clsraw{40.00}{73.33} & \clsdt{40.00}{73.33} & \clsraw{46.67}{46.67} & \clsdt{46.67}{46.67} & \clsraw{40.00}{60.00} & \clsdt{40.00}{60.00} & \clsraw{46.67}{46.67} & \clsdt{46.67}{46.67} \\
        \textit{UWaveGestureLibrary} & \clsraw{87.19}{87.50} & \clsdt{87.19}{87.50} & \clsraw{85.31}{85.62} & \clsdt{85.31}{85.62} & \clsraw{84.38}{84.69} & \clsdt{84.38}{84.69} & \clsraw{84.69}{85.62} & \clsdt{84.69}{85.62} & \clsraw{90.00}{90.94} & \clsdt{90.00}{90.94} & \clsraw{80.94}{84.38} & \clsdt{80.94}{84.38} \\
        \textit{HAR} & \clsraw{94.54}{95.22} & \clsdt{94.54}{95.22} & \clsraw{83.64}{83.85} & \clsdt{83.64}{83.85} & \clsraw{86.87}{89.38} & \clsdt{86.87}{89.38} & \clsraw{93.55}{94.10} & \clsdt{93.55}{94.10} & \clsraw{94.44}{95.01} & \clsdt{94.44}{95.01} & \clsraw{93.04}{93.08} & \clsdt{93.04}{93.08} \\
        \textit{Sleep-EDF} & \clsraw{80.19}{80.43} & \clsdt{80.19}{80.43} & \clsraw{87.37}{87.54} & \clsdt{87.37}{87.54} & \clsraw{86.09}{86.40} & \clsdt{86.09}{86.40} & \clsraw{82.35}{83.17} & \clsdt{82.35}{83.17} & \clsraw{89.42}{89.58} & \clsdt{89.42}{89.58} & \clsraw{85.60}{86.21} & \clsdt{85.60}{86.21} \\
        \midrule
        \rowcolor{TableSummary}
        \multicolumn{13}{c}{\textbf{Overall (192 pairs):} Accuracy $68.82\% \rightarrow$ \tablegain{70.91\%} (\tablegain{$2.09\,\mathrm{pp}\,\uparrow$}); W/T/L = \textbf{158/34/0}} \\
        \bottomrule
    \end{tabular}%
    }
    \vspace{-1em}
\end{table*}

\begin{table*}[!t]
    \centering
    \caption{Point-adjusted anomaly-detection precision (P), recall (R), and F1 scores across four datasets (\%). Shaded +\method columns denote \method; bold marks higher values per pair. Avg. $\Delta$F1 reports macro-averaged F1 gains in percentage points (pp).}
    \label{tab:anomaly_detection_results}
    \scriptsize
    \setlength{\tabcolsep}{3.0pt}
    \renewcommand{\arraystretch}{1.05}
    \providecommand{\adfraw}[2]{%
        \ifdim#1pt>#2pt \textbf{#1}\else
            \ifdim#1pt<#2pt \underline{#1}\else \textbf{#1}\fi
        \fi
    }
    \providecommand{\adfdt}[2]{%
        \ifdim#2pt>#1pt \textbf{#2}\else
            \ifdim#2pt<#1pt \underline{#2}\else \textbf{#2}\fi
        \fi
        \tableppchange{#1}{#2}%
    }
    \resizebox{\textwidth}{!}{%
    \begin{tabular}{@{}l c *{7}{c >{\columncolor{TableOurs}}c}@{}}
        \toprule
        \rowcolor{TableHeader}
        \multirow{2}{*}{\textbf{Dataset}} &
        \multirow{2}{*}{\textbf{Metric}} &
        \multicolumn{2}{c}{\textbf{PatchTST}} &
        \multicolumn{2}{c}{\textbf{iTransformer}} &
        \multicolumn{2}{c}{\textbf{NSFormer}} &
        \multicolumn{2}{c}{\textbf{AnomTrans}} &
        \multicolumn{2}{c}{\textbf{TranAD}} &
        \multicolumn{2}{c}{\textbf{DCdetector}} &
        \multicolumn{2}{c}{\textbf{TimesNet}} \\
        \cmidrule(lr){3-4} \cmidrule(lr){5-6} \cmidrule(lr){7-8} \cmidrule(lr){9-10}
        \cmidrule(lr){11-12} \cmidrule(lr){13-14} \cmidrule(lr){15-16}
        \rowcolor{TableSubheader}
        & & Raw & \tableoursheader +\method & Raw & \tableoursheader +\method & Raw & \tableoursheader +\method &
        Raw & \tableoursheader +\method & Raw & \tableoursheader +\method & Raw & \tableoursheader +\method & Raw & \tableoursheader +\method \\
        \midrule
        \multirow{3}{*}{\textit{SMD} (\%)}
        & P  & \adfraw{44.74}{70.27} & \adfdt{44.74}{70.27} & \adfraw{53.71}{79.98} & \adfdt{53.71}{79.98} & \adfraw{63.71}{67.10} & \adfdt{63.71}{67.10} & \adfraw{66.01}{66.06} & \adfdt{66.01}{66.06} & \adfraw{60.78}{70.86} & \adfdt{60.78}{70.86} & \adfraw{66.63}{66.17} & \adfdt{66.63}{66.17} & \adfraw{44.78}{45.66} & \adfdt{44.78}{45.66} \\
        & R  & \adfraw{99.04}{96.80} & \adfdt{99.04}{96.80} & \adfraw{99.35}{86.84} & \adfdt{99.35}{86.84} & \adfraw{78.52}{79.61} & \adfdt{78.52}{79.61} & \adfraw{84.97}{85.16} & \adfdt{84.97}{85.16} & \adfraw{84.99}{83.63} & \adfdt{84.99}{83.63} & \adfraw{83.15}{84.46} & \adfdt{83.15}{84.46} & \adfraw{99.92}{99.92} & \adfdt{99.92}{99.92} \\
        & F1 & \adfraw{61.64}{81.43} & \adfdt{61.64}{81.43} & \adfraw{69.73}{83.27} & \adfdt{69.73}{83.27} & \adfraw{70.34}{72.82} & \adfdt{70.34}{72.82} & \adfraw{74.30}{74.40} & \adfdt{74.30}{74.40} & \adfraw{70.87}{76.72} & \adfdt{70.87}{76.72} & \adfraw{73.98}{74.20} & \adfdt{73.98}{74.20} & \adfraw{61.85}{62.68} & \adfdt{61.85}{62.68} \\
        \midrule
        \multirow{3}{*}{\textit{MSL} (\%)}
        & P  & \adfraw{72.93}{88.19} & \adfdt{72.93}{88.19} & \adfraw{83.01}{88.22} & \adfdt{83.01}{88.22} & \adfraw{55.96}{70.44} & \adfdt{55.96}{70.44} & \adfraw{45.14}{59.26} & \adfdt{45.14}{59.26} & \adfraw{66.08}{67.18} & \adfdt{66.08}{67.18} & \adfraw{74.01}{74.69} & \adfdt{74.01}{74.69} & \adfraw{86.48}{88.72} & \adfdt{86.48}{88.72} \\
        & R  & \adfraw{91.09}{77.61} & \adfdt{91.09}{77.61} & \adfraw{38.93}{76.18} & \adfdt{38.93}{76.18} & \adfraw{45.72}{86.64} & \adfdt{45.72}{86.64} & \adfraw{4.27}{96.77} & \adfdt{4.27}{96.77} & \adfraw{34.02}{72.33} & \adfdt{34.02}{72.33} & \adfraw{82.65}{92.64} & \adfdt{82.65}{92.64} & \adfraw{53.70}{94.94} & \adfdt{53.70}{94.94} \\
        & F1 & \adfraw{81.00}{82.56} & \adfdt{81.00}{82.56} & \adfraw{53.00}{81.76} & \adfdt{53.00}{81.76} & \adfraw{50.32}{77.70} & \adfdt{50.32}{77.70} & \adfraw{7.80}{73.50} & \adfdt{7.80}{73.50} & \adfraw{44.92}{69.66} & \adfdt{44.92}{69.66} & \adfraw{78.09}{82.70} & \adfdt{78.09}{82.70} & \adfraw{66.26}{91.73} & \adfdt{66.26}{91.73} \\
        \midrule
        \multirow{3}{*}{\textit{SMAP} (\%)}
        & P  & \adfraw{77.96}{61.93} & \adfdt{77.96}{61.93} & \adfraw{59.78}{65.70} & \adfdt{59.78}{65.70} & \adfraw{75.35}{78.72} & \adfdt{75.35}{78.72} & \adfraw{74.52}{67.77} & \adfdt{74.52}{67.77} & \adfraw{73.69}{60.67} & \adfdt{73.69}{60.67} & \adfraw{62.34}{76.64} & \adfdt{62.34}{76.64} & \adfraw{52.26}{72.53} & \adfdt{52.26}{72.53} \\
        & R  & \adfraw{60.89}{85.92} & \adfdt{60.89}{85.92} & \adfraw{57.70}{76.54} & \adfdt{57.70}{76.54} & \adfraw{44.06}{84.76} & \adfdt{44.06}{84.76} & \adfraw{81.12}{99.93} & \adfdt{81.12}{99.93} & \adfraw{62.93}{91.14} & \adfdt{62.93}{91.14} & \adfraw{92.43}{97.10} & \adfdt{92.43}{97.10} & \adfraw{94.06}{94.06} & \adfdt{94.06}{94.06} \\
        & F1 & \adfraw{68.38}{71.98} & \adfdt{68.38}{71.98} & \adfraw{58.72}{70.71} & \adfdt{58.72}{70.71} & \adfraw{55.61}{81.63} & \adfdt{55.61}{81.63} & \adfraw{77.68}{80.77} & \adfdt{77.68}{80.77} & \adfraw{67.88}{72.85} & \adfdt{67.88}{72.85} & \adfraw{74.46}{85.67} & \adfdt{74.46}{85.67} & \adfraw{67.19}{81.90} & \adfdt{67.19}{81.90} \\
        \midrule
        \multirow{3}{*}{\textit{PSM} (\%)}
        & P  & \adfraw{97.08}{97.93} & \adfdt{97.08}{97.93} & \adfraw{99.65}{98.60} & \adfdt{99.65}{98.60} & \adfraw{99.18}{99.68} & \adfdt{99.18}{99.68} & \adfraw{100.00}{99.98} & \adfdt{100.00}{99.98} & \adfraw{99.97}{99.88} & \adfdt{99.97}{99.88} & \adfraw{100.00}{100.00} & \adfdt{100.00}{100.00} & \adfraw{93.12}{94.71} & \adfdt{93.12}{94.71} \\
        & R  & \adfraw{90.92}{98.23} & \adfdt{90.92}{98.23} & \adfraw{83.24}{93.77} & \adfdt{83.24}{93.77} & \adfraw{71.54}{84.44} & \adfdt{71.54}{84.44} & \adfraw{71.55}{76.59} & \adfdt{71.55}{76.59} & \adfraw{65.61}{69.71} & \adfdt{65.61}{69.71} & \adfraw{63.63}{76.56} & \adfdt{63.63}{76.56} & \adfraw{99.94}{99.94} & \adfdt{99.94}{99.94} \\
        & F1 & \adfraw{93.90}{98.08} & \adfdt{93.90}{98.08} & \adfraw{90.71}{96.13} & \adfdt{90.71}{96.13} & \adfraw{83.12}{91.43} & \adfdt{83.12}{91.43} & \adfraw{83.42}{86.73} & \adfdt{83.42}{86.73} & \adfraw{79.23}{82.12} & \adfdt{79.23}{82.12} & \adfraw{77.77}{86.72} & \adfdt{77.77}{86.72} & \adfraw{96.41}{97.25} & \adfdt{96.41}{97.25} \\
        \midrule
        \multicolumn{2}{l}{\textbf{Avg. $\Delta$F1}} &
        \multicolumn{2}{c}{\tablegain{$7.28\,\mathrm{pp}\,\uparrow$}} &
        \multicolumn{2}{c}{\tablegain{$14.93\,\mathrm{pp}\,\uparrow$}} &
        \multicolumn{2}{c}{\tablegain{$16.05\,\mathrm{pp}\,\uparrow$}} &
        \multicolumn{2}{c}{\tablegain{$18.05\,\mathrm{pp}\,\uparrow$}} &
        \multicolumn{2}{c}{\tablegain{$9.61\,\mathrm{pp}\,\uparrow$}} &
        \multicolumn{2}{c}{\tablegain{$6.25\,\mathrm{pp}\,\uparrow$}} &
        \multicolumn{2}{c}{\tablegain{$10.46\,\mathrm{pp}\,\uparrow$}} \\
        \rowcolor{TableSummary}
        \multicolumn{16}{c}{\textbf{Overall (28 dataset--backbone pairs):} Macro F1 $69.24\% \rightarrow$ \tablegain{81.04\%};
        \tablegain{$11.80\,\mathrm{pp}\,\uparrow$}; W/T/L = \textbf{28/0/0}} \\
        \bottomrule
    \end{tabular}%
    }
    \vspace{-1em}
\end{table*}

\subsection{RQ4: Cross-Task Generality}
\label{sec:rq4}

\noindent\textbf{Classification.}
Classification requires learning temporal features that distinguish classes across samples with varying measurement quality and local irregularity. By assigning stronger dropout to windows with higher residual scores, \method aims to reduce reliance on sample-specific fluctuations while retaining more activations for lower-score windows. Table~\ref{tab:classification_results} evaluates this regularization strategy in 192 matched pairs across six backbones and 32 datasets, with each task trained independently using its train/test split. Macro accuracy increases from $68.82\%$ to $70.91\%$, a gain of $2.09$ percentage points (pp), with 158 wins, 34 ties, and no losses at the reported precision. Notable gains span architectures: InceptionTime improves from $41.98\%$ to $61.83\%$ on \textit{EigenWorms} ($+19.85$ pp), NSFormer from $90.28\%$ to $98.61\%$ on \textit{Cricket} ($+8.33$ pp), and PatchTST from $83.33\%$ to $89.13\%$ on \textit{Epilepsy} ($+5.80$ pp). On \textit{FingerMovements}, all six backbones improve, with gains of $2.00$--$6.00$ pp, showing the benefit extends across architectures on the same task. Changes are smaller on \textit{SpokenArabicDigits}, where Raw accuracy exceeds $97\%$ for every backbone.

{\parfillskip=0pt\relax
\noindent\textbf{Anomaly Detection.}
Spectral residual scores vary even among nominally normal training windows. We test whether this label-free cue improves detection through adaptive regularization. We evaluate \textit{SMD}, \textit{MSL}, \textit{SMAP}, and \textit{PSM} using matched splits and the detector's task objective and anomaly score. For reconstruction-based models, the observed input window is the reconstruction target. Adaptive dropout acts on trainable activation paths and is disabled at evaluation. Raw and +\method share the benchmark point-adjusted (PA) protocol~\cite{xu2021anomaly,yang2023dcdetector} for paired comparison; Supplementary Appendix~A defines PA. Table~\ref{tab:anomaly_detection_results} reports precision, recall, and F1 for all seven backbones on four datasets. All 28 pairs improve under PA, with mean $\Delta\mathrm{F1}=+11.80$ pp (macro F1: $69.24\%$ to $81.04\%$). AnomTrans, NSFormer, and TranAD improve by $18.05$, $16.05$, and $9.61$ pp. On \textit{MSL}, iTransformer F1 rises from $53.00\%$ to $81.76\%$, while NSFormer improves from $50.32\%$ to $77.70\%$, showing substantial gains for multiple detectors on the same dataset. Macro means and changes are computed before rounding. The gains across detection objectives support the task-independent role of the scorer: it supplies a training-time regularization signal while each detector retains its own scoring rule.\par}

\raggedbottom
\setlength{\intextsep}{6pt}
\subsection{RQ5: Ablations, Complementarity \& Efficiency}
\label{sec:rq5}

\noindent\textbf{Sample-Adaptive vs. Fixed Capacity Control.}
Table~\ref{tab:dropout_strategy_comparison} compares the original backbone, validation-selected fixed dropout, learnable global dropout, and \method. We use four backbones on \textit{ILI} ($H=48$) and \textit{Exchange Rate} ($H=96$), with matched splits, initialization, and stopping rules. Validation MSE selects the fixed rate from $\{0,0.05,\ldots,0.50\}$; the global variant learns one rate shared by all samples within $[0.05,0.50]$. \method improves MSE over fixed dropout in all eight settings and over global dropout in seven, tying TimeMixer on \textit{Exchange Rate} at $0.101$. Its MAE is lower than the learned-global variant in all eight settings. For TimeMixer on \textit{ILI}, learning a global rate raises MSE from $3.055$ to $3.217$, whereas sample-adaptive rates yield $3.064$. This contrast supports conditioning regularization on individual windows beyond adjusting its strength. Supplementary Appendix~B also reports five-seed comparisons on three datasets and four backbones.

\begin{table*}[!b]
    \centering
    \small
    \caption{Dropout strategy comparison on forecasting benchmarks ($H=96$ for Exchange Rate, $H=48$ for ILI). Entries report test MSE/MAE. $p^\star$ is validation-selected; bold marks lowest MSE.}
    \label{tab:dropout_strategy_comparison}
    \setlength{\tabcolsep}{3pt}
    \renewcommand{\arraystretch}{1.08}

    \resizebox{\textwidth}{!}{%
    \begin{tabular}{@{}l
        c c c c c
        c c c c c@{}}
        \toprule
        \rowcolor{TableHeader}
        & \multicolumn{5}{c}{\textbf{\textit{ILI}} ($H{=}48$)}
        & \multicolumn{5}{c}{\textbf{\textit{Exchange Rate}} ($H{=}96$)} \\
        \cmidrule(lr){2-6}\cmidrule(lr){7-11}
        \rowcolor{TableSubheader}
        \textbf{Backbone} & $\boldsymbol{p^\star}$
        & \textbf{Raw} & \textbf{Best Fixed} & \textbf{Learned Global}
        & \tableoursheader\textbf{\method}
        & $\boldsymbol{p^\star}$
        & \textbf{Raw} & \textbf{Best Fixed} & \textbf{Learned Global}
        & \tableoursheader\textbf{\method} \\
        \midrule
        TimeFilter
        & 0.100 & 2.466/0.974 & 2.619/1.012 & 2.469/0.980 & \tableourscell\textbf{2.461}/\textbf{0.968}
        & 0.000 & 0.108/0.232 & 0.108/0.232 & 0.106/0.230 & \tableourscell\textbf{0.105}/\textbf{0.229} \\
        PatchTST
        & 0.000 & 2.939/1.099 & 2.956/1.111 & 2.879/1.095 & \tableourscell\textbf{2.874}/\textbf{1.093}
        & 0.100 & \textbf{0.107}/0.230 & 0.108/0.231 & 0.109/0.230 & \tableourscell\textbf{0.107}/\textbf{0.229} \\
        TimeMixer
        & 0.000 & \textbf{3.055}/1.130 & 3.227/1.207 & 3.217/1.205 & \tableourscell 3.064/\textbf{1.129}
        & 0.000 & 0.102/\textbf{0.224} & 0.102/\textbf{0.224} & \textbf{0.101}/0.225 & \tableourscell\textbf{0.101}/\textbf{0.224} \\
        Informer
        & 0.450 & 7.173/1.920 & 6.569/1.807 & 6.650/1.936 & \tableourscell\textbf{6.348}/\textbf{1.784}
        & 0.050 & 4.411/1.676 & 2.754/1.169 & 2.777/1.143 & \tableourscell\textbf{2.078}/\textbf{1.050} \\
        \bottomrule
    \end{tabular}%
    }
\end{table*}

\noindent\textbf{Component Attribution Analysis.}
We use the same \textit{Synth-12} benchmark as RQ1 to examine component effects under controlled corruption. All ablation variants are evaluated across its five corruption scales using the same signal construction and forecasting horizons. Table~\ref{tab:component_ablation_summary} summarizes performance across these scales, while Supplementary Appendix~B reports the complete per-scale results. The full configuration has MSE/MAE $1.076/0.797$. Removing global linear detrending, the SFM anchor, or log-amplitude normalization raises MSE by $52.6\%$, $41.3\%$, and $21.6\%$, respectively, relative to the full configuration. All listed ablations also perform worse than fixed dropout ($p=0.1$) on average. Detrending and spectral anchoring have the largest effects, supporting their roles in separating the residual from broad temporal structure and setting the filtering threshold. Their performance below fixed dropout highlights the importance of the scoring pipeline when adapting regularization. Endpoint-based detrending also trails OLS (MSE $1.574$ vs.\ $1.076$) with other components fixed.

\begin{table}[H]
    \centering
    \footnotesize
    \caption{Component ablation on Synth-12 ($\sigma \in [0.1, 0.9]$). Gain denotes relative change vs. baseline ($p=0.1$).}
    \label{tab:component_ablation_summary}
    \setlength{\tabcolsep}{2.4pt}
    \resizebox{\linewidth}{!}{
    \begin{tabular}{@{}l ccc cc cc@{}}
        \toprule
        \rowcolor{TableHeader}
        \multirow{2}{*}{\textbf{Method}} & \multicolumn{3}{c}{\textbf{Components}} & \multicolumn{2}{c}{\textbf{MSE $\downarrow$}} & \multicolumn{2}{c}{\textbf{MAE $\downarrow$}} \\
        \cmidrule(lr){2-4} \cmidrule(lr){5-6} \cmidrule(lr){7-8}
        \rowcolor{TableSubheader}
         & \textbf{Detrend} & \textbf{Norm} & \textbf{log-SFM} & \textbf{Avg.} & \textbf{Gain (\%)} & \textbf{Avg.} & \textbf{Gain (\%)} \\
        \midrule
        Baseline & - & - & - & 1.159 & - & 0.836 & - \\
        \midrule
        Minimal Model & None & \ding{55} & \ding{55} & 1.514 & \tableloss{$\downarrow$\,30.6} & 0.960 & \tableloss{$\downarrow$\,14.8} \\
        w/o Detrend+Norm & None & \ding{55} & \ding{51} & 1.468 & \tableloss{$\downarrow$\,26.7} & 0.942 & \tableloss{$\downarrow$\,12.7} \\
        w/o Detrend & None & \ding{51} & \ding{51} & 1.642 & \tableloss{$\downarrow$\,41.7} & 1.014 & \tableloss{$\downarrow$\,21.3} \\
        Simple Detrend & Simple & \ding{51} & \ding{51} & 1.574 & \tableloss{$\downarrow$\,35.8} & 0.987 & \tableloss{$\downarrow$\,18.1} \\
        w/o Spectral Norm & OLS & \ding{55} & \ding{51} & 1.308 & \tableloss{$\downarrow$\,12.9} & 0.896 & \tableloss{$\downarrow$\,7.2} \\
        w/o log-SFM Anchor & OLS & \ding{51} & \ding{55} & 1.520 & \tableloss{$\downarrow$\,31.1} & 0.970 & \tableloss{$\downarrow$\,16.0} \\
        \midrule
        \rowcolor{TableOurs}
        \textbf{\method (Ours)} & OLS & \ding{51} & \ding{51} & \textbf{1.076} & \tablegain{$\uparrow$\,7.2} & \textbf{0.797} & \tablegain{$\uparrow$\,4.7} \\
        \bottomrule
    \end{tabular}}
    \par\vspace{3pt}
    {\scriptsize\raggedright
    OLS: global ordinary least-squares linear detrending; Simple: endpoint-based linear detrending; Norm: log-amplitude normalization; log-SFM: spectral flatness measure anchoring. The Minimal Model omits all three listed components. The baseline uses a fixed dropout rate ($p=0.1$) without adaptive scoring. ``w/o'' denotes ``without''; \ding{51}/\ding{55} indicate enabled/disabled components, respectively. Bold marks the lowest average error.\par}
\end{table}

\noindent\textbf{Hyperparameter Sensitivity.}
A larger $\gamma$ increases dropout at a fixed normalized residual score and moves the rate mapper toward saturation. Figure~\ref{fig:hyperparam_sensitivity} compares $\gamma\in\{1,5,10\}$ across five corruption scales. Panel (a) shows non-monotonic responses to corruption, including a sharp MSE reduction for $\gamma=10$ at $\sigma=0.7$. Panel (b) isolates the sensitivity effect: MSE peaks at $\gamma=5$ for $\sigma\in\{0.5,0.7,0.9\}$ but increases across settings for $\sigma\in\{0.1,0.3\}$. Thus, stronger masking affects corruption scales differently. The $\gamma=1$ setting yields the lowest plotted MSE at four scales; the remaining experiments initialize $\gamma=1$.

\begin{figure}[!b]
    \centering
    \includegraphics[width=\linewidth]{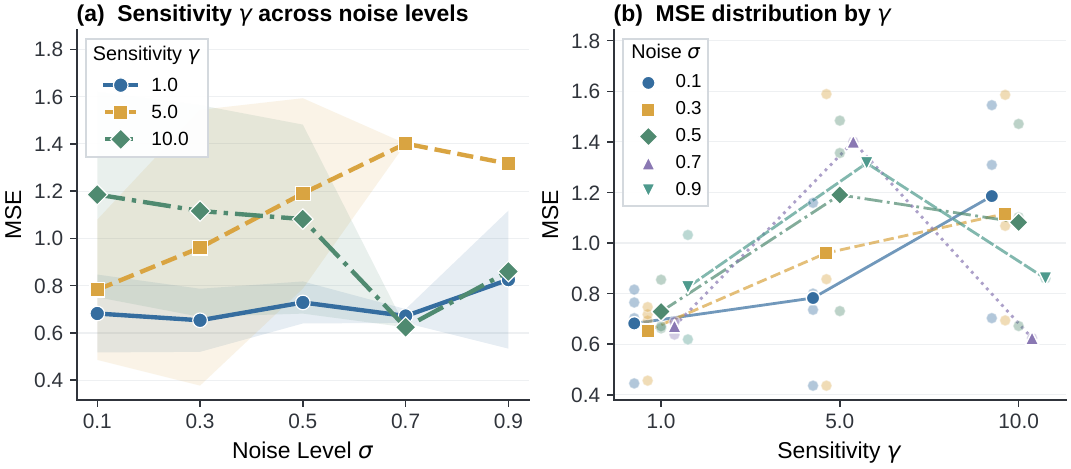}
    \caption{\textbf{Hyperparameter sensitivity analysis.} (a) Effect of sensitivity parameter $\gamma$ across corruption scales $\sigma \in [0.1, 0.9]$; (b) MSE grouped by sensitivity setting, with separate markers for the five corruption scales.}
    \label{fig:hyperparam_sensitivity}
\end{figure}

\noindent\textbf{Computational Efficiency.}
Table~\ref{tab:efficiency_analysis} reports added parameters, latency, epochs to stopping, and total training time for matched Raw and +\method runs. The scorer adds only 16 parameters in these seven-channel configurations. The reported time per epoch increases by $11.1\%$ for TimeMixer and $29.3\%$ for Informer. Both +\method runs stop after 16 epochs, compared with 20 and 30 for their Raw counterparts. In these runs, the reduction in epoch count offsets the additional per-epoch cost, lowering total training time by $11.0\%$ for TimeMixer ($1.12\times$) and $31.0\%$ for Informer ($1.45\times$). The auxiliary path incurs this cost only during training; inference retains the deterministic graph.

\begin{table}[H]
\centering
\caption{Training cost for the paired runs (Raw $\to$ +\method). Epochs denotes epochs to stopping; total time includes \method overhead.}
\label{tab:efficiency_analysis}

\renewcommand{\arraystretch}{1.25}
\setlength{\tabcolsep}{5pt}

\resizebox{\linewidth}{!}{
\begin{tabular}{l cc cc cc}
\toprule
\rowcolor{TableHeader}
 & \multicolumn{2}{c}{\textbf{Overhead (Cost)}} & \multicolumn{2}{c}{\textbf{Training Cost}} & \multicolumn{2}{c}{\textbf{Net Gain}} \\
\cmidrule(lr){2-3} \cmidrule(lr){4-5} \cmidrule(lr){6-7}

\rowcolor{TableSubheader}
\textbf{Model} & \textbf{Params} & \textbf{s/Epoch} & \textbf{Epochs} & \textbf{Total (s)} & \textbf{Saved} & \textbf{Speedup} \\
\midrule

\textbf{TimeMixer} & +16 & 102.3 $\to$ 113.7 & 20 $\to$ \tablegain{16} & 2045 $\to$ \tablegain{1819} & \tablegain{11.0\%} & \tablegain{1.12$\times$} \\

\textbf{Informer} & +16 & \phantom{0}95.8 $\to$ 123.9 & 30 $\to$ \tablegain{16} & 2873 $\to$ \tablegain{1982} & \tablegain{31.0\%} & \tablegain{1.45$\times$} \\

\bottomrule
\end{tabular}
}
\end{table}

\noindent\textbf{Complementarity with Data-Centric Selection.}
Table~\ref{tab:data_selection_compatibility_cost} combines Selective Learning (SL)~\cite{fu2025selective} with \method on \textit{ILI} using Informer, $L=24$, and $H=60$. Adding \method to SL reduces MSE from $6.461$ to $6.336$ ($1.93\%$) and MAE from $1.902$ to $1.774$ ($6.73\%$). Both variants incur the same $29.0$ s selection-pretraining cost. Backbone training takes $15.1$ s with SL and $6.5$ s with SL+\method, totaling $44.1$ s and $35.5$ s. The combined run saves $8.6$ s versus SL alone, while remaining above the Raw baseline's $20.0$ s. Accuracy gains after SL support complementary roles: selection determines which time steps enter training, while \method adjusts activation retention.

\begin{table}[!ht]
    \centering
    \footnotesize
    \caption{Compatibility with Selective Learning (SL) and computational cost on ILI ($L=24, H=60$, Informer). Rel. Gain denotes relative MSE reduction from Raw.}
    \label{tab:data_selection_compatibility_cost}
    \setlength{\tabcolsep}{2pt}
    \renewcommand{\arraystretch}{1.00}

    \resizebox{\linewidth}{!}{%
    \begin{tabular}{@{}l l c c c@{}}
        \toprule
        \rowcolor{TableSection}
        \multicolumn{5}{@{}l}{\textbf{Compatibility with SL}} \\
        \midrule
        \rowcolor{TableHeader}
        \multirow{2}{*}{\textbf{Method Strategy}} & \multirow{2}{*}{\textbf{Mechanism}} & \multicolumn{2}{c}{\textbf{Error Metric} $\downarrow$} & \textbf{Rel.} \\
        \cmidrule(lr){3-4}
        \rowcolor{TableSubheader}
         & & \textbf{MSE} & \textbf{MAE} & $\boldsymbol{\Delta_{\mathrm{rel}}}$ (\%) \\
        \midrule
        \textbf{Baseline} (Raw) & - & 7.140 & 1.916 & - \\
        \textbf{\method} & Capacity Modulation & 6.429 & 1.846 & $\uparrow$\,10.0 \\
        \textbf{SL} & Data Selection & 6.461 & 1.902 & $\uparrow$\,9.5 \\
        \midrule
        \rowcolor{TableOurs} \textbf{SL + \method} (Ours) & \textbf{Combined} & \textbf{6.336} & \textbf{1.774} & \tablegain{$\uparrow$\,11.3} \\
        \midrule

        \rowcolor{TableSection}
        \multicolumn{5}{@{}l}{\textbf{Parameter Overhead}} \\
        \midrule
        \rowcolor{TableHeader}
        \textbf{Method} & \textbf{Base Params} & \textbf{Additional} & \textbf{Total} & \textbf{Overhead (\%)} \\
        \midrule
        Raw Baseline & 11,799,047 & - & 11,799,047 & - \\
        \rowcolor{TableOurs}\textbf{\method (Ours)} & 11,799,047 & \textbf{+16} & \textbf{11,799,063} & \tablegain{$1.36\!\times\!10^{-4}$} \\
        SL & 11,799,047 & +3,000 & 11,802,047 & 0.025 \\
        \midrule

        \rowcolor{TableSection}
        \multicolumn{5}{@{}l}{\textbf{Training Time}} \\
        \midrule
        \rowcolor{TableHeader}
        \textbf{Method} & \textbf{Train Time} & \textbf{Pre-train} & \textbf{Total} & \textbf{Time Saved (\%)} \\
        \midrule
        Raw Baseline & 20.0s & 0s & 20.0s & - \\
        \rowcolor{TableOurs}\textbf{\method} & \textbf{16.8s} & \textbf{0s} & \textbf{16.8s} & \tablegain{$\uparrow$\,16.0} \\
        SL & 15.1s & 29.0s & 44.1s & \tableloss{$\downarrow$\,120.5} \\
        \method + SL & 6.5s & 29.0s & 35.5s & \tableloss{$\downarrow$\,77.5} \\
        \bottomrule
    \end{tabular}%
    }
\end{table}

\noindent\textbf{Comparison with Data-Space Hard Denoising.}
Table~\ref{tab:hard_denoising_comparison} compares spectral processing as input filtering versus a regularization cue. On \textit{ILI}, hard denoising (HD) applies low-pass filtering with matched horizons, split, and backbones. HD lowers Informer MAE from $1.901$ to $1.869$ but raises PatchTST and TimeMixer MAE to $1.162$ and $1.197$. \method instead achieves $1.772$, $1.064$, and $1.142$, improving over Raw on all three backbones. These results suggest that filtering can discard useful information, supporting separate scoring and prediction paths: spectral residuals guide training-time masking while the predictor retains the observed input.

\begin{table}[!t]
    \centering
    \small
    \caption{MAE comparison with hard denoising (HD) on ILI. $\Delta$ is the relative MAE change from Raw; bold marks the lower error and $\uparrow$ denotes improvement.}
    \label{tab:hard_denoising_comparison}
    \setlength{\tabcolsep}{7pt}
    \renewcommand{\arraystretch}{1.0}
    
    \newcommand{\gc}{\tableourscell}
    \newcommand{\win}[1]{\textbf{#1}}
    
    \adjustbox{max width=\linewidth}{%
    \begin{tabular}{l c c c c c}
        \toprule
        \rowcolor{TableHeader}
        \textbf{Backbone} & \textbf{Raw} & \textbf{HD} & \textbf{\method} & \textbf{$\Delta$ HD} & \textbf{$\Delta$ Ours} \\
        \midrule
        Informer & 1.901 & 1.869 & \gc \win{1.772} & \tablegain{$\uparrow$1.7\%} & \tablegain{$\uparrow$6.8\%} \\
        PatchTST & 1.110 & 1.162 & \gc \win{1.064} & \tableloss{$\downarrow$4.7\%} & \tablegain{$\uparrow$4.1\%} \\
        TimeMixer & 1.148 & 1.197 & \gc \win{1.142} & \tableloss{$\downarrow$4.3\%} & \tablegain{$\uparrow$0.5\%} \\
        \midrule
        \rowcolor{TableSummary}
        \textbf{Overall} & - & - & - & \tableloss{$\downarrow$2.4\%} & \tablegain{$\uparrow$3.8\%} \\
        \bottomrule
    \end{tabular}%
    }
\end{table}

\noindent\textbf{Qualitative Forecast Behavior.}
Figure~\ref{fig:qualitative_comparison} shows four test windows from the 96-to-720 SyntheticTS noise-0.1 experiment using WPMixer. The windows were selected from the middle range of baseline errors, rather than the largest improvements. The relative MAE reductions are 7.0\%, 8.8\%, 17.1\%, and 21.3\% for panels (a)--(d), respectively. Across these examples, \method reduces level and amplitude errors while following the broad temporal pattern of the clean target; residual differences remain around sharp spikes, trough depth, and fast oscillations.

\begin{figure}[H]
    \centering
    \includegraphics[width=0.94\columnwidth]{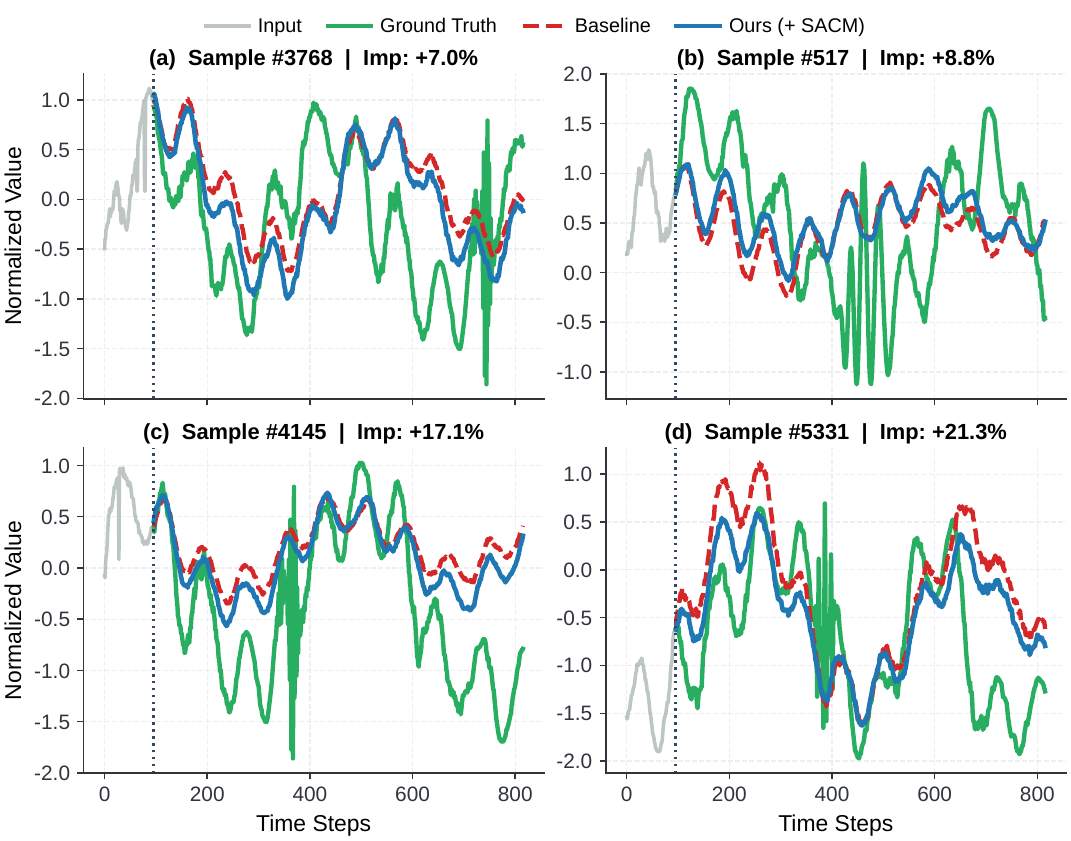}
    \caption{\textbf{Qualitative comparison of trajectories.} Input context (gray), clean ground truth (green), Raw WPMixer (dashed red), and \method (blue).}
    \label{fig:qualitative_comparison}
\end{figure}

\section{Conclusion}
\label{sec:conclusion}

We present \textbf{Capacity-Centric Modulation (CCM)}, which addresses sample-level reliability differences through adaptive regularization of activation paths. Its task-agnostic implementation, \textbf{\method}, calibrates sample-wise dropout using label-free spectral residuals. A bias--variance surrogate characterizes the excess risk of fixed allocation relative to an oracle with heterogeneous optimal rates. Across 301 dataset--backbone pairs, \method improves forecasting, classification, and anomaly detection, with shorter training times in the profiled runs and no inference overhead. The gains persist across data regimes and transfer to unseen domains without target adaptation.

\noindent\textbf{Limitations and Future Work.}
Spectral residuals may conflate informative broadband or impulsive dynamics with corruption (Appendix~C). Future work includes hybrid time--frequency filtering for such non-harmonic signals, as well as batch-stable calibration and channel-specific modulation.

\bibliographystyle{IEEEtranN}
\bibliography{main}

\makeatletter
\def\@IEEEBIOskipN{0.5\baselineskip}
\def\@IEEEBIOphotowidth{0.80in}
\def\@IEEEBIOphotodepth{1.00in}
\def\@IEEEBIOhangwidth{0.91in}
\def\@IEEEBIOhangdepth{1.00in}
\makeatother

\begin{IEEEbiography}[{\includegraphics[width=0.80in,height=1.00in,clip,keepaspectratio]{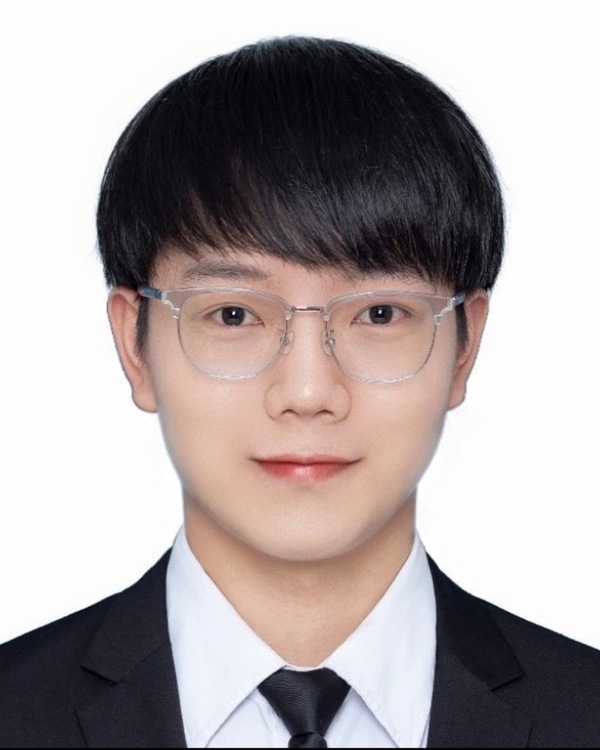}}]{Siru Zhong}
is a Ph.D.\ candidate with the Data Science and Analytics Thrust, HKUST (GZ), Guangzhou, China. His research interests include spatio-temporal intelligence and time series modeling. He has authored or coauthored more than 20 papers in ICML, NeurIPS, KDD, ACM MM, and AAAI. He gave tutorials at ACM MM 2025, ICME 2026, and AAAI 2026, and received the DSA Excellent Research Award in 2025 and 2026.
\end{IEEEbiography}

\begin{IEEEbiography}[{\includegraphics[width=0.80in,height=1.00in,clip,keepaspectratio]{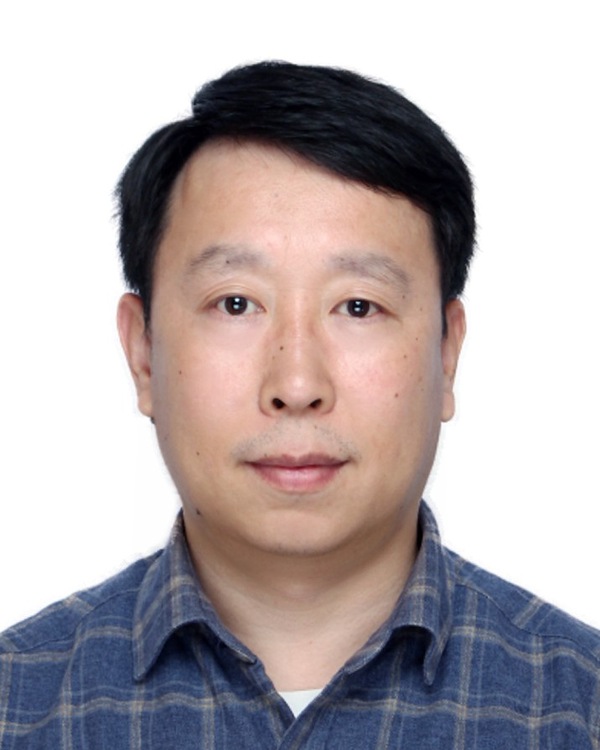}}]{Senzhang Wang}
(Member, IEEE) received the Ph.D.\ degree from Beihang University, Beijing, China, in 2015. He is currently a Professor with the School of Computer Science and Engineering, Central South University, Changsha. He has authored or coauthored more than 100 papers in top journals and conferences, such as IEEE TKDE, ACM TOIS, KDD, NeurIPS, ICLR, IJCAI, and AAAI. His research interests include graph mining, spatio-temporal data mining, and time series analysis.
\end{IEEEbiography}

\begin{IEEEbiography}[{\includegraphics[width=0.80in,height=1.00in,clip,keepaspectratio]{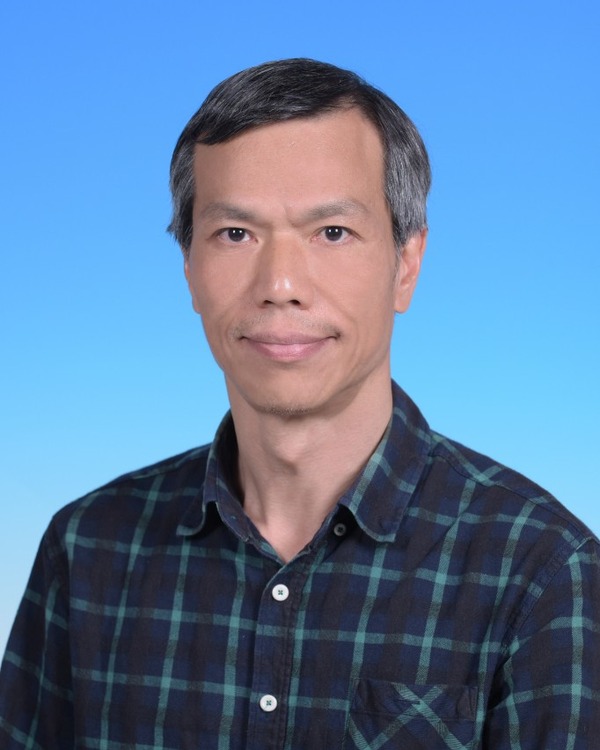}}]{James T. Kwok}
(Fellow, IEEE) is currently a Professor with the Department of Computer Science and Engineering at HKUST, Hong Kong. His research interests include machine learning and artificial intelligence. He has served as an Associate Editor of several journals including the IEEE Transactions on Neural Networks and Learning Systems, Neural Networks, Neurocomputing, and Artificial Intelligence. He has also served as a Senior Area Chair of major conferences including NeurIPS, ICML, ICLR, and IJCAI. He is a member of the IJCAI Board of Trustees, and served as the Program Chair of IJCAI 2025 and the General Chair of PAKDD 2026.
\end{IEEEbiography}

\begin{IEEEbiography}[{\includegraphics[width=0.80in,height=1.00in,clip,keepaspectratio]{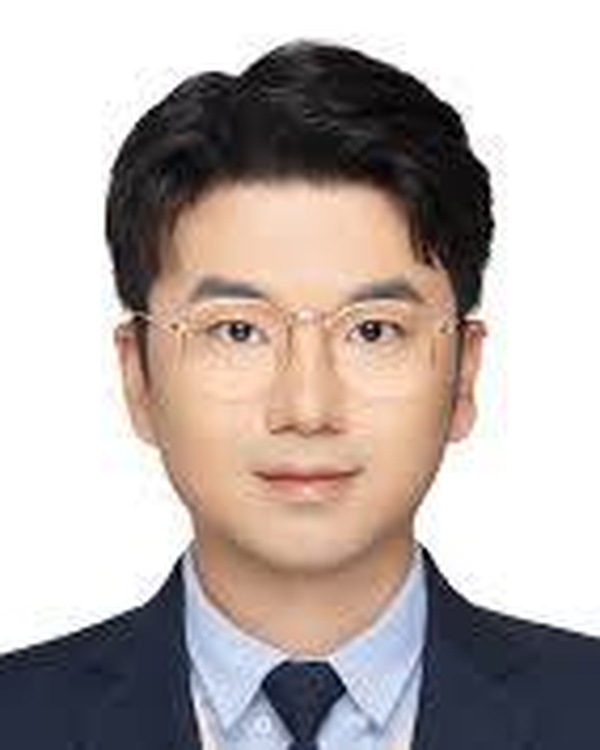}}]{Yuxuan Liang}
is an Assistant Professor at HKUST (GZ). He obtained his Ph.D. from the National University of Singapore. His research focuses on spatio-temporal AI and foundation models, and he has authored over 100 papers with 16,000 citations. He serves as the Vice Chair of the IEEE CIS NNTC Task Force on Spatio-Temporal Data and Time Series, Associate Editor of Neurocomputing, and Area Chair of major AI conferences including NeurIPS, ICML, ICLR, KDD, ACL, and MM. He received Best Paper (or Runner-up) Awards, e.g., ICDM 2025, the ACM SIGSPATIAL China Rising Star Award, and the SDSC Dissertation Research Fellowship.
\end{IEEEbiography}

\clearpage
\onecolumn
\raggedbottom

\makeatletter
\setlength{\@fptop}{0pt}
\setlength{\@fpsep}{12pt}
\setlength{\@fpbot}{0pt plus 1fil}
\makeatother
\renewcommand{\topfraction}{0.95}
\renewcommand{\bottomfraction}{0.95}
\renewcommand{\textfraction}{0.05}
\renewcommand{\floatpagefraction}{0.80}
\setlength{\textfloatsep}{14pt plus 3pt minus 3pt}
\setlength{\floatsep}{12pt plus 3pt minus 3pt}

\begin{bibunit}
\appendices

\section*{Appendices Overview}
\noindent
The appendices provide the experimental specifications, complete forecasting results, spectral diagnostics, detrending analysis, and theoretical proofs supporting the main paper. The entries below are clickable in the electronic version.

\vspace{0.4em}
\begingroup
\small
\renewcommand{\arraystretch}{1.18}
\setlength{\tabcolsep}{4pt}
\noindent\begin{tabularx}{\linewidth}{>{\bfseries}p{0.10\linewidth} >{\bfseries}p{0.35\linewidth} X r}
\toprule
\rowcolor{TableHeader}
\textbf{Part} & \textbf{Section} & \textbf{Contents} & \textbf{Page} \\
\midrule
Appendix A & \hyperref[app:experimental_details]{Experimental Details} & Benchmarks, protocols, notation, and evaluation metrics & \pageref{app:experimental_details} \\
Appendix B & \hyperref[app:complete_results]{Complete Results} & Per-horizon forecasting results, fixed-rate sweep, ablations, and batch-size sensitivity & \pageref{app:complete_results} \\
Appendix C & \hyperref[appx:visualizations]{Spectral Diagnostics and Real-world Sparsity} & Component-wise diagnostics and top-coefficient reconstructions & \pageref{appx:visualizations} \\
Appendix D & \hyperref[sec:detrend_ablation]{Detrending Analysis} & Spectral leakage, Edge-MAE, and failure modes & \pageref{sec:detrend_ablation} \\
Appendix E & \hyperref[app:theory_proofs]{Theory and Proofs} & Surrogate-risk assumptions and excess-risk theorem proof & \pageref{app:theory_proofs} \\
\bottomrule
\end{tabularx}
\endgroup

\vspace{0.8em}

\suppressfloats[t]

\section{Experimental Details}
\label{app:experimental_details}

\subsection{Benchmark Details}
\label{app:benchmark_details}

We evaluate forecasting, classification, and anomaly detection benchmarks spanning energy, meteorology, finance, healthcare, human activity, speech, and operational telemetry, covering both large-scale real-world series and synthetic stress tests.

\noindent\textbf{Forecasting benchmarks.}
The four \textit{ETT} variants contain oil temperature and six power-load variables at hourly or 15-minute resolution. \textit{Weather}, \textit{Electricity}, and \textit{Exchange Rate} represent meteorological, high-dimensional demand, and non-stationary financial series, respectively. \textit{ILI} contains weekly influenza-like-illness ratios, while \textit{ECG} concatenates both leads from MIT-BIH records 100--105 \cite{goldberger2000physiobank} into a 12-channel series, downsampled to 36 Hz, truncated to 65,000 points. Table~\ref{tab:forecasting_benchmark_details} summarizes dimensions, splits, frequencies, and horizons. Forecasting datasets use the chronological train/validation/test splits from \texttt{BasicTS} \cite{liang2022basicts}.

\begin{table}[t]
  \vspace{-1em}
  \caption{Statistics of the nine real-world forecasting benchmarks. Columns report dimensionality, time steps, chronological train/val/test split, sampling frequency, prediction horizons, and domain; ILI and ECG use task-specific shorter horizons.}
  \label{tab:forecasting_benchmark_details}
  \centering
  \small
  \setlength{\tabcolsep}{4pt}
  \begin{tabular*}{\textwidth}{@{\extracolsep{\fill}}lcrcccl@{}}
    \toprule
    \rowcolor{TableHeader}
    \textbf{Dataset} & \textbf{Dim.} & \textbf{Size} & \textbf{Split} & \textbf{Frequency} & \textbf{Prediction Length} & \textbf{Domain} \\
    \midrule
    \textit{ETTh1}       & 7   & 14,400 & 6:2:2 & 1 hour  & \{96, 192, 336, 720\} & Temperature \\
    \textit{ETTh2}       & 7   & 14,400 & 6:2:2 & 1 hour  & \{96, 192, 336, 720\} & Temperature \\
    \textit{ETTm1}       & 7   & 57,600 & 6:2:2 & 15 min  & \{96, 192, 336, 720\} & Temperature \\
    \textit{ETTm2}       & 7   & 57,600 & 6:2:2 & 15 min  & \{96, 192, 336, 720\} & Temperature \\
    \textit{Weather}     & 21  & 52,696 & 7:1:2 & 10 min  & \{96, 192, 336, 720\} & Meteorology \\
    \textit{Electricity} & 321 & 26,304 & 7:1:2 & 1 hour  & \{96, 192, 336, 720\} & Electricity \\
    \textit{Exchange Rate} & 8   & 7,588  & 7:1:2 & 1 day   & \{96, 192, 336, 720\} & Finance \\
    \midrule
    \rowcolor{TableSection}
    \textit{ILI}         & 7   & 966    & 7:1:2 & 1 week  & \{24, 36, 48, 60\}    & Health \\
    \rowcolor{TableSection}
    \textit{ECG}     & 12  & 65,000 & 7:1:2 & 36 Hz   & \{36, 72, 144, 288\}  & Health \\
    \bottomrule
  \end{tabular*}
\end{table}

\noindent\textbf{Classification benchmarks.}
We use 30 multivariate UEA-style datasets \cite{bagnall2018uea}, together with \textit{HAR} \cite{anguita2013public} and \textit{Sleep-EDF} \cite{kemp2000analysis,goldberger2000physiobank}. The index below lists each benchmark's modality and prediction target.

\noindent\textbf{Anomaly-detection benchmarks.}
We use the task-specific train/validation/test splits of four multivariate telemetry datasets: \textit{SMD} \cite{su2019robust} contains server-machine metrics and labeled abnormal intervals; \textit{MSL} and \textit{SMAP} \cite{hundman2018detecting} contain spacecraft telemetry and anomaly events; and \textit{PSM} \cite{xu2021anomaly} contains production-server metrics with service-level anomalies. Each reconstruction model follows its prescribed training split and benchmark evaluation protocol, with the input window as the reconstruction target. Each backbone retains its native reconstruction head, and \method is applied to the trainable activation paths before that head. The benchmark-specific window composition, label usage, and thresholding follow the corresponding public benchmark protocol. Table~\ref{tab:anomaly_benchmark_details} summarizes the dimensionality, split sizes, and test-split anomaly rates of the four benchmarks.

\begin{table}[t]
  \vspace{-1em}
  \caption{Statistics of the four anomaly-detection benchmarks. Columns report dimensionality, training and test time steps, anomaly ratio in the test split, and domain. Evaluation follows the point-adjusted protocol defined in Appendix~\ref{app:evaluation_metrics}.}
  \label{tab:anomaly_benchmark_details}
  \centering
  \small
  \setlength{\tabcolsep}{4pt}
  \begin{tabular*}{\textwidth}{@{\extracolsep{\fill}}lccccl@{}}
    \toprule
    \rowcolor{TableHeader}
    \textbf{Dataset} & \textbf{Dim.} & \textbf{Train} & \textbf{Test} & \textbf{Anomaly Rate} & \textbf{Domain} \\
    \midrule
    \textit{SMD}  & 38 & 708,405 & 708,420 & 4.16\%  & Server machine \\
    \textit{MSL}  & 55 & 58,317  & 73,729  & 10.72\% & Spacecraft \\
    \textit{SMAP} & 25 & 135,183 & 427,617 & 13.13\% & Spacecraft \\
    \textit{PSM}  & 25 & 132,481 & 87,841  & 27.76\% & Server machine \\
    \bottomrule
  \end{tabular*}
\end{table}

\begin{table}[t]
\centering
\begingroup
\footnotesize
\renewcommand{\arraystretch}{1.08}
\setlength{\tabcolsep}{4pt}
\rowcolors{2}{TableSubheader}{white}
\noindent\begin{tabularx}{\linewidth}{>{\itshape}p{0.205\linewidth} X >{\itshape}p{0.205\linewidth} X}
\toprule
\rowcolor{TableHeader}
\multicolumn{1}{l}{\textbf{Dataset}} & \textbf{Sequence/target} & \multicolumn{1}{l}{\textbf{Dataset}} & \textbf{Sequence/target} \\
\midrule
ArticularyWordRecognition & Articulatory trajectories; spoken words & AtrialFibrillation & ECG; cardiac rhythm \\
BasicMotions & Wearable sensors; basic motions & CharacterTrajectories & Pen-tip trajectories; characters \\
Cricket & Motion sensors; umpire gestures & DuckDuckGeese & Audio features; bird calls \\
EigenWorms & Posture sequences; locomotion & Epilepsy & Wrist acceleration; seizure-like motions \\
EthanolConcentration & Spectral sequences; concentration & ERing & Ring-sensor trajectories; gestures \\
FaceDetection & Brain signals; face detection & FingerMovements & Brain signals; finger movements \\
HandMovementDirection & Brain signals; movement direction & Handwriting & Motion trajectories; characters \\
Heartbeat & Cardiac auscultation; heartbeat class & InsectWingbeat & Wingbeat signals; insect species \\
JapaneseVowels & Speech trajectories; speakers & Libras & Movement trajectories; sign language \\
LSST & Astronomical light curves; object class & MotorImagery & Brain signals; motor imagery \\
NATOPS & Motion sequences; handling signals & PenDigits & Pen trajectories; digits \\
PEMS-SF & Traffic sensors; traffic state & PhonemeSpectra & Speech spectra; phonemes \\
RacketSports & Wearable sensors; sport actions & SelfRegulationSCP1 & Cortical potentials; self-regulation \\
SelfRegulationSCP2 & Cortical potentials; self-regulation & SpokenArabicDigits & Speech sequences; digits \\
StandWalkJump & Multichannel sensors; activities & UWaveGestureLibrary & Acceleration; gestures \\
HAR & Smartphone inertial signals; activities & Sleep-EDF & Polysomnography; sleep stages \\
\bottomrule
\end{tabularx}
\endgroup
\end{table}

\subsection{Unified Experimental Protocol}
\label{app:experimental_protocol}

\noindent\textbf{Backbones and task settings.}
The forecasting study covers Informer, Crossformer, PatchTST, iTransformer, MultiPatchFormer, TimeMixer, WPMixer, TimesNet, and TimeFilter. Classification uses iTransformer, PatchTST, NSFormer, TimesNet, InceptionTime, and MiniROCKET; anomaly detection uses PatchTST, iTransformer, NSFormer, AnomTrans, TranAD, DCdetector, and TimesNet. For forecasting, the input length is $L=24$ on \textit{ILI} and $L=96$ on all other datasets. The evaluated horizons are $H\in\{24,36,48,60\}$ for \textit{ILI}, $H\in\{36,72,144,288\}$ for \textit{ECG}, and $H\in\{96,192,336,720\}$ for the remaining forecasting datasets. On \textit{Synth-12}, forecasts are evaluated against the clean target trajectory rather than its corrupted observation.

\noindent\textbf{Optimization and integration.}
All models are implemented in PyTorch, optimized with Adam, and run on NVIDIA A800 GPUs. We retain each backbone's default hyperparameters so that paired Raw and +\method variants differ only in the capacity modulation strategy. \method recursively replaces dropout modules on trainable paths; models without them are unchanged. All adaptive layers share the sample-level probability vector, but draw independent Bernoulli activation masks. The spectral residual scorer and capacity modulation module are active strictly during training and are completely bypassed during evaluation, so \method follows the same deterministic inference path as the Raw baseline with no additional computation or latency.

\noindent\textbf{Capacity modulation settings and controlled comparisons.}
Unless varied in the sensitivity study, we fix $[p_{\min},p_{\max}]=[0.05,0.50]$ and initialize the learnable spectral-mask steepness at $\alpha=10$, threshold bias at $\mathbf{b}_s=0$, and rate-mapping sensitivity at $\gamma=1$. In the fixed-dropout comparison, candidate rates are $p\in\{0,0.05,\ldots,0.50\}$ and $p^\star$ is selected by validation MSE only. The paired-seed study uses seeds 2022--2026 with matched initialization, data order, early stopping, and hyperparameters between Raw and +\method.

\noindent\textbf{Roles of the remaining hyperparameters.}
The fixed dropout bounds define the range over which the learned mapper allocates regularization. The lower bound $p_{\min}=0.05$ maintains mild regularization even at zero normalized residual score. The upper bound $p_{\max}=0.50$ keeps the expected retention fraction $1-p_i$ at least $0.50$ and the inverted-dropout scaling $1/(1-p_i)$ at most $2$. The attainable maximum also depends on $\gamma$. At its default initialization, a normalized score of $1$ gives $p_i=0.05+0.45\tanh(\operatorname{softplus}(1))\approx0.439$. These are properties of the rate mapper, rather than measured retention statistics.

The learnable scalar $\alpha$ controls spectral-mask sharpness. With $\alpha=10$, the sigmoid moves from $0.1$ to $0.9$ over a normalized-amplitude interval of $2\ln(9)/\operatorname{softplus}(10)\approx0.44$ around $\tau_c$, giving a smooth initial separation of spectral components. A smaller slope weakens this separation, while a larger slope approaches hard thresholding and reduces mask gradients away from $\tau_c$. Initializing $\mathbf{b}_s=0$ leaves the initial threshold logits determined by $w_{s,c}\mathrm{SFM}_c$. Each channel's weight and bias, together with $\alpha$ and $\gamma$, are updated through the task loss, so the thresholds and sharpness adapt during training.

The constant $\epsilon=10^{-8}$ stabilizes spectral normalization and logarithms. It also determines when nearly identical batch scores use the default normalized score $0.5$. Learning rate, batch size, and stopping settings follow each backbone's default configuration and are matched between paired Raw and +\method runs, as described above. The sensitivity experiment varies $\gamma$; the remaining settings describe the shared default configuration.

\noindent\textbf{Result notation.}
Raw denotes the unmodified backbone and +\method denotes the same backbone augmented with \method. For an error metric $e$ (lower is better) and a score metric $m$ (higher is better), relative improvement is computed as
\begin{equation}
    \Delta_{\mathrm{rel}} e = 100\frac{e^{\mathrm{Raw}}-e^{\mathrm{\method}}}{e^{\mathrm{Raw}}},
    \qquad
    \Delta_{\mathrm{rel}} m = 100\frac{m^{\mathrm{\method}}-m^{\mathrm{Raw}}}{m^{\mathrm{Raw}}}.
\end{equation}
A positive value therefore denotes an improvement in both cases. For accuracy and F1 reported as percentages, an absolute change is stated in percentage points (pp) rather than relative percent. W/T/L counts pairwise wins, ties, and losses under the metric named in the corresponding table caption, aggregated over the reported model--dataset pairs.

\subsection{Evaluation Metrics}
\label{app:evaluation_metrics}

\noindent\textbf{Forecasting.}
For $N_{\mathrm{f}}$ test windows, prediction horizon $H$, and $C$ target channels, let $y_{i,h,c}$ and $\hat{y}_{i,h,c}$ denote the target and prediction for sample $i$, horizon step $h$, and channel $c$. We report mean squared error (MSE) and mean absolute error (MAE):
\begin{align}
    \operatorname{MSE} &= \frac{1}{N_{\mathrm{f}}HC}
    \sum_{i=1}^{N_{\mathrm{f}}}\sum_{h=1}^{H}\sum_{c=1}^{C}
    \left(y_{i,h,c}-\hat{y}_{i,h,c}\right)^2, \\
    \operatorname{MAE} &= \frac{1}{N_{\mathrm{f}}HC}
    \sum_{i=1}^{N_{\mathrm{f}}}\sum_{h=1}^{H}\sum_{c=1}^{C}
    \left|y_{i,h,c}-\hat{y}_{i,h,c}\right|.
\end{align}
Lower values indicate better forecasting performance.

\noindent\textbf{Classification.}
For $N_{\mathrm{c}}$ test samples with ground-truth class $z_i$ and predicted class $\hat{z}_i$, classification accuracy is
\begin{equation}
    \operatorname{Accuracy} = \frac{1}{N_{\mathrm{c}}}
    \sum_{i=1}^{N_{\mathrm{c}}}\mathbb{I}(z_i=\hat{z}_i),
\end{equation}
where $\mathbb{I}(\cdot)$ is the indicator function.

\noindent\textbf{Anomaly Detection.}
Let $a_t\in\{0,1\}$ be the ground-truth label at time step $t$, and let $\tilde{a}_t\in\{0,1\}$ be the binary prediction from an anomaly score using the detector's evaluation threshold. For each contiguous ground-truth anomaly interval $I_m$, point adjustment marks the entire interval as detected when at least one point in that interval is predicted anomalous:
\begin{equation}
    \hat{a}^{\mathrm{PA}}_t =
    \begin{cases}
        1, & t\in I_m \text{ and } \sum_{u\in I_m}\tilde{a}_u>0
        \text{ for some }m, \\
        \tilde{a}_t, & \text{otherwise}.
    \end{cases}
\end{equation}
Using the adjusted predictions, the point-level counts are
\begin{align}
    \mathrm{TP} &= \sum_t a_t\hat{a}^{\mathrm{PA}}_t, &
    \mathrm{FP} &= \sum_t (1-a_t)\hat{a}^{\mathrm{PA}}_t, &
    \mathrm{FN} &= \sum_t a_t(1-\hat{a}^{\mathrm{PA}}_t).
\end{align}
We then report point-adjusted precision, recall, and F1 score:
\begin{align}
    \operatorname{Precision} &= \frac{\mathrm{TP}}{\mathrm{TP}+\mathrm{FP}}, &
    \operatorname{Recall} &= \frac{\mathrm{TP}}{\mathrm{TP}+\mathrm{FN}}, &
    \operatorname{F1} &= \frac{2\,\operatorname{Precision}\,\operatorname{Recall}}
    {\operatorname{Precision}+\operatorname{Recall}}.
\end{align}
For $K$ dataset--backbone configurations, reported macro averages and average gains use unweighted arithmetic means:
\begin{align}
    \operatorname{Macro}(m) &= \frac{1}{K}\sum_{j=1}^{K}m_j, &
    \operatorname{Avg.}\,\Delta m &= \frac{1}{K}\sum_{j=1}^{K}
    \left(m_j^{\mathrm{\method}}-m_j^{\mathrm{Raw}}\right),
\end{align}
where $m_j$ is the relevant metric for configuration $j$; percentage-point gains are reported when $m_j$ is expressed as a percentage.

\section{Complete Results}
\label{app:complete_results}

This section reports the horizon-level forecasting measurements underlying the averages in the main paper. Unless stated otherwise, each Raw/+\method pair uses the same backbone configuration and data split. Up and down arrows in Tables~\ref{tab:synthetic_forecasting_complete_results} and~\ref{tab:real_world_forecasting_complete_results} follow the relative-improvement convention in Appendix~\ref{app:experimental_protocol} and summarize horizon-averaged changes.

\subsection{Results of \textit{Synth-12}}

Table~\ref{tab:synthetic_forecasting_complete_results} expands the main-paper averages into MSE/MAE for each of five corruption scales $\sigma\in\{0.1,0.3,0.5,0.7,0.9\}$ and four horizons $H\in\{96,192,336,720\}$. Evaluation uses clean targets. The horizon-level entries expose variation with forecasting length that is hidden by the main-paper averages.

\begin{table}[t!]
    \centering
    \scriptsize
    \caption{Complete per-horizon forecasting results on Synth-12 for nine backbones. Each horizon reports Raw and \method (+\method) MSE/MAE; lower errors are bold and higher errors are underlined. Up and down arrows report relative improvements and regressions, respectively, averaged over horizons.}
    \label{tab:synthetic_forecasting_complete_results}
    \setlength{\tabcolsep}{1.2pt}
    \renewcommand{\arraystretch}{1.02}
    \resizebox{\linewidth}{!}{
\begin{tabular}{cccccccccccccccccccccccccccccccccccccc}
    \toprule
    \rowcolor{TableHeader}
    \multirow{3}{*}{\textbf{$\sigma$}} & \multirow{3}{*}{\textbf{$H$}} &
    \multicolumn{4}{c}{Informer}& \multicolumn{4}{c}{Crossformer} & \multicolumn{4}{c}{PatchTST}&
    \multicolumn{4}{c}{TimesNet}& \multicolumn{4}{c}{iTransformer} & \multicolumn{4}{c}{TimeMixer} &
    \multicolumn{4}{c}{WPMixer}& \multicolumn{4}{c}{TimeFilter} & \multicolumn{4}{c}{MultiPatchFormer}\\
    \rowcolor{TableSubheader}
    & &
    \multicolumn{4}{c}{\cite{zhou2021informer}}& \multicolumn{4}{c}{\cite{zhang2022crossformer}} & \multicolumn{4}{c}{\cite{nie2023patchtst}}&
    \multicolumn{4}{c}{\cite{wu2023timesnet}}& \multicolumn{4}{c}{\cite{liu2023itransformer}} & \multicolumn{4}{c}{\cite{wang2024timemixer}} &
    \multicolumn{4}{c}{\cite{murad2025wpmixer}}& \multicolumn{4}{c}{\cite{hutimefilter}} & \multicolumn{4}{c}{\cite{naghashi2025multiscale}}\\
    \cmidrule(lr){3-6}\cmidrule(lr){7-10} \cmidrule(lr){11-14}\cmidrule(lr){15-18} \cmidrule(lr){19-22} \cmidrule(lr){23-26} \cmidrule(lr){27-30}\cmidrule(lr){31-34} \cmidrule(lr){35-38}
    \rowcolor{TableSubheader}
    & &
    \multicolumn{2}{c}{Raw}& \multicolumn{2}{c}{\tableoursheader\textbf{+\method}}&
    \multicolumn{2}{c}{Raw}& \multicolumn{2}{c}{\tableoursheader\textbf{+\method}}&
    \multicolumn{2}{c}{Raw}& \multicolumn{2}{c}{\tableoursheader\textbf{+\method}}&
    \multicolumn{2}{c}{Raw}& \multicolumn{2}{c}{\tableoursheader\textbf{+\method}}&
    \multicolumn{2}{c}{Raw}& \multicolumn{2}{c}{\tableoursheader\textbf{+\method}} &
    \multicolumn{2}{c}{Raw}& \multicolumn{2}{c}{\tableoursheader\textbf{+\method}} &
    \multicolumn{2}{c}{Raw}& \multicolumn{2}{c}{\tableoursheader\textbf{+\method}} &
    \multicolumn{2}{c}{Raw}& \multicolumn{2}{c}{\tableoursheader\textbf{+\method}} &
    \multicolumn{2}{c}{Raw}& \multicolumn{2}{c}{\tableoursheader\textbf{+\method}}\\
\bottomrule
    \end{tabular}
    }
\end{table}

\subsection{Results of Open Benchmarks}
Table~\ref{tab:real_world_forecasting_complete_results} reports the individual forecasting horizons for nine backbones on nine real-world benchmarks. These entries complement the main-paper summary by showing where gains, ties, and regressions occur at particular horizons.

\begin{table}[t!]
    \centering
    \scriptsize
    \caption{Complete per-horizon real-world forecasting results for nine backbones. Each horizon reports Raw and \method (+\method) MSE/MAE; lower errors are bold and higher errors are underlined. Up and down arrows report relative improvements and regressions, respectively, averaged over horizons.}
    \label{tab:real_world_forecasting_complete_results}
    \setlength{\tabcolsep}{1.2pt}
    \renewcommand{\arraystretch}{0.96}
    \resizebox{\linewidth}{!}{
\begin{tabular}{cccccccccccccccccccccccccccccccccccccc}
    \toprule
    \rowcolor{TableHeader}
    \multirow{3}{*}{\textbf{Dataset}} & \multirow{3}{*}{\textbf{$H$}} &
    \multicolumn{4}{c}{Informer}& \multicolumn{4}{c}{Crossformer} & \multicolumn{4}{c}{PatchTST}&
    \multicolumn{4}{c}{TimesNet}& \multicolumn{4}{c}{iTransformer} & \multicolumn{4}{c}{TimeMixer} &
    \multicolumn{4}{c}{WPMixer}& \multicolumn{4}{c}{TimeFilter} & \multicolumn{4}{c}{MultiPatchFormer}\\
    \rowcolor{TableSubheader}
    & &
    \multicolumn{4}{c}{\cite{zhou2021informer}}& \multicolumn{4}{c}{\cite{zhang2022crossformer}} & \multicolumn{4}{c}{\cite{nie2023patchtst}}&
    \multicolumn{4}{c}{\cite{wu2023timesnet}}& \multicolumn{4}{c}{\cite{liu2023itransformer}} & \multicolumn{4}{c}{\cite{wang2024timemixer}} &
    \multicolumn{4}{c}{\cite{murad2025wpmixer}}& \multicolumn{4}{c}{\cite{hutimefilter}} & \multicolumn{4}{c}{\cite{naghashi2025multiscale}}\\
    \cmidrule(lr){3-6}\cmidrule(lr){7-10}\cmidrule(lr){11-14}\cmidrule(lr){15-18}\cmidrule(lr){19-22}\cmidrule(lr){23-26}\cmidrule(lr){27-30}\cmidrule(lr){31-34}\cmidrule(lr){35-38}
    \rowcolor{TableSubheader}
    & & \multicolumn{2}{c}{Raw}& \multicolumn{2}{c}{\tableoursheader\textbf{+\method}}&     \multicolumn{2}{c}{Raw}& \multicolumn{2}{c}{\tableoursheader\textbf{+\method}}&     \multicolumn{2}{c}{Raw}& \multicolumn{2}{c}{\tableoursheader\textbf{+\method}}&     \multicolumn{2}{c}{Raw}& \multicolumn{2}{c}{\tableoursheader\textbf{+\method}}&     \multicolumn{2}{c}{Raw}& \multicolumn{2}{c}{\tableoursheader\textbf{+\method}}&     \multicolumn{2}{c}{Raw}& \multicolumn{2}{c}{\tableoursheader\textbf{+\method}}&     \multicolumn{2}{c}{Raw}& \multicolumn{2}{c}{\tableoursheader\textbf{+\method}}&     \multicolumn{2}{c}{Raw}& \multicolumn{2}{c}{\tableoursheader\textbf{+\method}}&     \multicolumn{2}{c}{Raw}& \multicolumn{2}{c}{\tableoursheader\textbf{+\method}} \\
\bottomrule
    \end{tabular}
    }
\end{table}

\clearpage

\subsection{Paired-Seed Stability}
Table~\ref{tab:paired_seed_stability} reports a representative multi-seed check on three datasets and four backbones. Each Raw/+\method pair uses five matched seeds (2022--2026) with identical initialization, data order, early stopping, and hyperparameters. Mean MSE gains are positive in all twelve settings; several configurations reach $p<0.05$ under a paired comparison, while others remain directionally consistent but not significant at this sample size.

\begin{table}[!t]
    \centering
    \small
    \caption{Paired-seed stability across five random seeds (2022--2026). MSE entries show mean $\pm$ std. W/T/L denotes wins/ties/losses.}
    \label{tab:paired_seed_stability}
    \setlength{\tabcolsep}{3.2pt}
    \renewcommand{\arraystretch}{1.08}

    \resizebox{\textwidth}{!}{%
    \begin{tabular}{@{}l
        c >{\columncolor{TableOurs}}c c c c
        c >{\columncolor{TableOurs}}c c c c
        c >{\columncolor{TableOurs}}c c c c@{}}
        \toprule
        \rowcolor{TableHeader}
        & \multicolumn{5}{c}{\textbf{\textit{ETTh1}} ($L{=}96,H{=}96$)}
        & \multicolumn{5}{c}{\textbf{\textit{ILI}} ($L{=}24,H{=}24$)}
        & \multicolumn{5}{c}{\textbf{\textit{Exchange Rate}} ($L{=}96,H{=}96$)} \\
        \cmidrule(lr){2-6}\cmidrule(lr){7-11}\cmidrule(lr){12-16}
        \rowcolor{TableHeader}
        \textbf{Backbone}
        & \textbf{Raw MSE} & \tableoursheader\textbf{+\method MSE} & \textbf{Gain (\%)} & \textbf{W/T/L} & \textbf{$p$}
        & \textbf{Raw MSE} & \tableoursheader\textbf{+\method MSE} & \textbf{Gain (\%)} & \textbf{W/T/L} & \textbf{$p$}
        & \textbf{Raw MSE} & \tableoursheader\textbf{+\method MSE} & \textbf{Gain (\%)} & \textbf{W/T/L} & \textbf{$p$} \\
        \midrule
        Informer
        & $1.2658 \pm 0.1879$ & $\mathbf{1.0658} \pm 0.0453$ & \tablegain{$\uparrow$\,15.8} & 4/0/1 & 0.076
        & $6.2605 \pm 0.7293$ & $\mathbf{5.8442} \pm 0.2124$ & \tablegain{$\uparrow$\,6.7} & 3/0/2 & 0.301
        & $4.0671 \pm 0.8449$ & $\mathbf{1.2747} \pm 0.3771$ & \tablegain{$\uparrow$\,68.7} & 5/0/0 & \textbf{0.003} \\
        PatchTST
        & $0.3943 \pm 0.0047$ & $\mathbf{0.3897} \pm 0.0039$ & \tablegain{$\uparrow$\,1.2} & 4/0/1 & 0.056
        & $3.3997 \pm 0.3373$ & $\mathbf{2.8384} \pm 0.1953$ & \tablegain{$\uparrow$\,16.5} & 5/0/0 & \textbf{0.004}
        & $0.1040 \pm 0.0018$ & $\mathbf{0.1022} \pm 0.0014$ & \tablegain{$\uparrow$\,1.7} & 5/0/0 & \textbf{0.005} \\
        TimeMixer
        & $0.3946 \pm 0.0025$ & $\mathbf{0.3920} \pm 0.0020$ & \tablegain{$\uparrow$\,0.7} & 5/0/0 & \textbf{0.007}
        & $3.3904 \pm 0.1844$ & $\mathbf{3.3582} \pm 0.3473$ & \tablegain{$\uparrow$\,1.0} & 2/0/3 & 0.764
        & $0.1024 \pm 0.0005$ & $\mathbf{0.1017} \pm 0.0007$ & \tablegain{$\uparrow$\,0.7} & 4/0/1 & 0.166 \\
        TimeFilter
        & $0.3895 \pm 0.0010$ & $\mathbf{0.3891} \pm 0.0009$ & \tablegain{$\uparrow$\,0.1} & 4/0/1 & 0.104
        & $1.8652 \pm 0.1738$ & $\mathbf{1.8227} \pm 0.1565$ & \tablegain{$\uparrow$\,2.3} & 4/0/1 & 0.550
        & $0.1044 \pm 0.0014$ & $\mathbf{0.1033} \pm 0.0014$ & \tablegain{$\uparrow$\,1.0} & 5/0/0 & 0.124 \\
        \bottomrule
    \end{tabular}%
    }
\end{table}

\subsection{Complete Fixed-Dropout Sweep}
Table~\ref{tab:fixed_dropout_grid_results} reports test results for every candidate rate from 0 to 0.50 in increments of 0.05 on the representative horizons used in the main comparison. Validation MSE selects the bold cell $p^\star$, while test values are reported for every completed fixed-rate checkpoint; selection is therefore independent of test performance.

\begin{table}[H]
    \centering
    \scriptsize
    \caption{Fixed-dropout sweep at $H=96$ for ETTh1 and Exchange Rate and $H=48$ for ILI. Cells report test MSE/MAE; $p^\star$ is selected solely by validation MSE and marked in bold. Underlined cells mark the second-lowest test MSE in each row.}
    \label{tab:fixed_dropout_grid_results}
    \setlength{\tabcolsep}{2.2pt}
    \renewcommand{\arraystretch}{1.05}
    \resizebox{\textwidth}{!}{%
    \begin{tabular}{llc*{11}{c}}
        \toprule
        \rowcolor{TableHeader}
        \textbf{Dataset} & \textbf{Backbone} & $\boldsymbol{p^\star}$ & \multicolumn{11}{c}{\textbf{Fixed dropout rate $p$ (test MSE/MAE)}} \\
        \cmidrule(lr){4-14}
        \rowcolor{TableSubheader}
        & & & 0.00 & 0.05 & 0.10 & 0.15 & 0.20 & 0.25 & 0.30 & 0.35 & 0.40 & 0.45 & 0.50 \\
        \midrule
        ETTh1 & Informer & 0.50 & 1.478/0.872 & 1.508/0.879 & 1.367/0.835 & 1.425/0.860 & 1.360/0.836 & 1.432/0.858 & 1.362/0.844 & 1.506/0.884 & 1.314/0.824 & \underline{1.284/0.815} & \tableselectedcell\textbf{0.975/0.739} \\
        \addlinespace[2pt]
        Exchange Rate & Informer & 0.05 & \underline{4.024/1.589} & \tableselectedcell\textbf{2.754/1.169} & 4.107/1.601 & 4.227/1.626 & 4.261/1.625 & 4.192/1.617 & 4.214/1.619 & 4.277/1.625 & 4.213/1.612 & 4.170/1.601 & 4.274/1.621 \\
        \addlinespace[2pt]
        ILI & Informer & 0.45 & \underline{7.662/2.060} & 9.094/2.213 & 9.310/2.254 & 8.394/2.102 & 8.853/2.181 & 7.877/2.023 & 8.567/2.122 & 8.519/2.106 & 8.573/2.119 & \tableselectedcell\textbf{6.569/1.807} & 8.792/2.146 \\
        \addlinespace[2pt]
        ETTh1 & PatchTST & 0.05 & 0.392/0.392 & \tableselectedcell\textbf{\underline{0.386/0.390}} & 0.383/0.390 & 0.391/0.390 & 0.390/0.390 & 0.388/0.390 & 0.387/0.390 & 0.387/0.391 & 0.388/0.391 & 0.391/0.393 & 0.399/0.396 \\
        \addlinespace[2pt]
        Exchange Rate & PatchTST & 0.10 & 0.107/0.230 & \underline{0.102/0.225} & \tableselectedcell\textbf{0.108/0.231} & 0.101/0.223 & 0.101/0.225 & 0.103/0.226 & 0.103/0.226 & 0.103/0.227 & 0.104/0.228 & 0.105/0.229 & 0.106/0.230 \\
        \addlinespace[2pt]
        ILI & PatchTST & 0.00 & \tableselectedcell\textbf{2.956/1.111} & \underline{2.919/1.096} & 2.992/1.110 & 2.962/1.110 & 3.130/1.171 & 3.066/1.158 & 3.172/1.197 & 3.286/1.234 & 3.454/1.273 & 3.867/1.366 & 3.928/1.380 \\
        \addlinespace[2pt]
        ETTh1 & TimeFilter & 0.00 & \tableselectedcell\textbf{\underline{0.391/0.391}} & \underline{0.391/0.391} & \underline{0.391/0.391} & \underline{0.391/0.391} & \underline{0.391/0.391} & \underline{0.391/0.391} & \underline{0.391/0.391} & 0.390/0.391 & \underline{0.391/0.391} & \underline{0.391/0.391} & 0.390/0.391 \\
        \addlinespace[2pt]
        Exchange Rate & TimeFilter & 0.00 & \tableselectedcell\textbf{0.108/0.232} & 0.106/0.229 & 0.106/0.229 & 0.106/0.229 & 0.106/0.229 & 0.106/0.229 & 0.102/0.225 & 0.106/0.229 & \underline{0.103/0.226} & \underline{0.103/0.226} & 0.106/0.229 \\
        \addlinespace[2pt]
        ILI & TimeFilter & 0.10 & 2.619/1.012 & 2.560/1.019 & \tableselectedcell\textbf{2.619/1.012} & 2.574/1.020 & 2.514/1.010 & \underline{2.532/1.011} & 2.545/1.014 & 2.551/1.015 & 2.562/1.018 & 2.590/1.022 & 2.623/1.028 \\
        \addlinespace[2pt]
        ETTh1 & TimeMixer & 0.00 & \tableselectedcell\textbf{\underline{0.390/0.392}} & 0.391/0.390 & 0.394/0.390 & 0.391/0.388 & 0.391/0.388 & 0.391/0.388 & \underline{0.390/0.388} & \underline{0.390/0.388} & \underline{0.390/0.388} & \underline{0.390/0.388} & 0.389/0.388 \\
        \addlinespace[2pt]
        Exchange Rate & TimeMixer & 0.00 & \tableselectedcell\textbf{\underline{0.102/0.224}} & 0.101/0.223 & 0.101/0.223 & 0.101/0.223 & 0.101/0.223 & 0.101/0.223 & 0.101/0.223 & \underline{0.102/0.224} & \underline{0.102/0.224} & \underline{0.102/0.224} & \underline{0.102/0.224} \\
        \addlinespace[2pt]
        ILI & TimeMixer & 0.00 & \tableselectedcell\textbf{3.227/1.207} & 3.467/1.268 & 3.201/1.200 & 3.196/1.201 & 3.201/1.206 & \underline{3.187/1.205} & 3.203/1.210 & 3.211/1.214 & 3.230/1.220 & 3.256/1.225 & 3.280/1.233 \\
        \bottomrule
    \end{tabular}%
    }
\end{table}

\subsection{Detailed Ablation Results}
\label{appx:ablation_full}

\begin{table}[H]
    \centering
    \setlength{\tabcolsep}{2pt}
    \caption{Complete component ablation on Synth-12 across five corruption scales. Rows report MSE/MAE at each $\sigma$ and on average; $\Delta$ is relative to fixed dropout. OLS denotes global ordinary least-squares detrending. The full \method configuration is labeled ``Ours,'' and bold marks the lowest error in each result column.}
    \label{tab:component_ablation_complete_results}

    \resizebox{\linewidth}{!}{%
    \begin{tabular}{l ccc cc cc cc cc cc cc cc}
        \toprule
        \rowcolor{TableHeader}
        \multirow{2.5}{*}{\textbf{Method}} &
        \multicolumn{3}{c}{\textbf{Configuration}} &
        \multicolumn{2}{c}{$\sigma=0.1$} &
        \multicolumn{2}{c}{$\sigma=0.3$} &
        \multicolumn{2}{c}{$\sigma=0.5$} &
        \multicolumn{2}{c}{$\sigma=0.7$} &
        \multicolumn{2}{c}{$\sigma=0.9$} &
        \multicolumn{2}{c}{\textbf{Average}} &
        \multicolumn{2}{c}{$\Delta\%$ \scriptsize{(vs. Base)}} \\

        \cmidrule(lr){2-4}
        \cmidrule(lr){5-6} \cmidrule(lr){7-8} \cmidrule(lr){9-10} \cmidrule(lr){11-12} \cmidrule(lr){13-14}
        \cmidrule(lr){15-16} \cmidrule(lr){17-18}

        \rowcolor{TableSubheader}
        & \small{Detrend} & \small{LogNorm} & \small{log-SFM}
        & \scriptsize MSE & \scriptsize MAE
        & \scriptsize MSE & \scriptsize MAE
        & \scriptsize MSE & \scriptsize MAE
        & \scriptsize MSE & \scriptsize MAE
        & \scriptsize MSE & \scriptsize MAE
        & \small\textbf{MSE} & \small\textbf{MAE}
        & \scriptsize MSE & \scriptsize MAE \\
        \midrule

        Baseline & - & - & - & 1.228 & 0.862 & 1.189 & 0.846 & 1.180 & 0.843 & 1.140 & 0.831 & 1.060 & 0.798 & 1.159 & 0.836 & - & - \\
        \addlinespace[0.3em]

        Minimal Model & None & \ding{55} & \ding{55} & 1.698 & 1.034 & 1.907 & 1.092 & 1.668 & 1.010 & \textbf{0.633} & \textbf{0.650} & 1.665 & 1.015 & 1.514 & 0.960 & \tableloss{30\%\,$\downarrow$} & \tableloss{14\%\,$\downarrow$} \\
        w/o Detrend+Norm & None & \ding{55} & \ding{51} & 2.041 & 1.124 & 1.159 & 0.843 & 1.114 & 0.827 & 1.954 & 1.111 & 1.072 & 0.806 & 1.468 & 0.942 & \tableloss{26\%\,$\downarrow$} & \tableloss{12\%\,$\downarrow$} \\
        w/o Detrend & None & \ding{51} & \ding{51} & 1.672 & 1.017 & 1.863 & 1.093 & 1.721 & 1.051 & 1.442 & 0.950 & 1.513 & 0.960 & 1.642 & 1.014 & \tableloss{41\%\,$\downarrow$} & \tableloss{21\%\,$\downarrow$} \\
        Simple Detrend & Simple & \ding{51} & \ding{51} & 1.416 & 0.943 & 1.757 & 1.042 & 1.674 & 1.018 & 1.486 & 0.958 & 1.535 & 0.974 & 1.574 & 0.987 & \tableloss{35\%\,$\downarrow$} & \tableloss{18\%\,$\downarrow$} \\
        w/o Spectral Norm & OLS & \ding{55} & \ding{51} & 1.336 & 0.911 & 1.251 & 0.879 & 1.278 & 0.889 & 1.588 & 0.988 & 1.090 & 0.811 & 1.308 & 0.896 & \tableloss{12\%\,$\downarrow$} & \tableloss{7\%\,$\downarrow$} \\
        w/o log-SFM Anchor & OLS & \ding{51} & \ding{55} & 1.386 & 0.933 & 1.532 & 0.968 & 1.540 & 0.977 & 1.693 & 1.029 & 1.448 & 0.944 & 1.520 & 0.970 & \tableloss{31\%\,$\downarrow$} & \tableloss{16\%\,$\downarrow$} \\
        \midrule

        \rowcolor{TableOurs}
        \textbf{\method (Ours)} & OLS & \ding{51} & \ding{51} & \textbf{1.135} & \textbf{0.841} & \textbf{1.155} & \textbf{0.746} & \textbf{1.050} & \textbf{0.808} & 1.049 & 0.803 & \textbf{0.990} & \textbf{0.785} & \textbf{1.076} & \textbf{0.797} & \tablegain{7.2\%\,$\uparrow$} & \tablegain{4.7\%\,$\uparrow$} \\
        \bottomrule
    \end{tabular}
    }
\end{table}

Table~\ref{tab:component_ablation_complete_results} isolates the three preprocessing choices that turn a residual score into a reliable sample-wise capacity modulation signal: global OLS detrending, log-amplitude normalization, and log-SFM anchoring. The Baseline row is fixed-dropout training without adaptive scoring and is the reference for the final $\Delta\%$ columns. Reading the ablations as a pipeline rather than as independent switches clarifies which stages matter most.

Detrending is the dominant factor. Removing it entirely raises average MSE from $1.076$ to $1.642$ ($52.6\%$ relative degradation vs.\ the full method) and leaves every $\sigma$ well above both Baseline and Ours. Endpoint-style Simple Detrend is only a partial remedy: average MSE remains $1.574$, and at $\sigma=0.9$ it reaches $1.535$ (worse than no detrending at $1.513$), while the full OLS configuration obtains $0.990$. A slope estimated from two endpoints can be sensitive to corruption at those endpoints, whereas OLS uses the full window. The observed error difference is consistent with this motivation for global detrending.

The other two stages refine an already detrended residual. Dropping the log-SFM anchor raises average MSE from $1.076$ to $1.520$ ($41.3\%$ relative vs.\ the full method), nearly as severe as removing detrending, consistent with the spectral reference helping calibrate the mask. Removing log-amplitude normalization is milder but still material ($1.308$ average MSE, $21.6\%$ relative vs.\ the full method), supporting normalization before computing the residual score. Configurations that strip multiple stages at once (Minimal Model; w/o Detrend+Norm) stay in the $1.47$--$1.51$ average-MSE band and underperform fixed dropout on $\Delta\%$, showing that a poorly calibrated adaptive score can hurt more than a constant $p$.

The full configuration has the lowest average MSE and the lowest MSE at four of the five values of $\sigma$, and it is the only row with positive $\Delta\%$ against Baseline ($7.2\%$ MSE / $4.7\%$ MAE). At $\sigma=0.7$, Minimal Model records $0.633/0.650$, but its average MSE of $1.514$ remains above both the full method ($1.076$) and fixed dropout ($1.159$). Overall, the complete matrix supports a joint design. Removing any one of the three components increases average error, while the isolated advantage of the Minimal Model at $\sigma=0.7$ does not persist across scales. Together, the components provide the strongest average performance among the tested configurations on Synth-12.

\subsection{Batch-Size Sensitivity}
\label{appx:batch_size_sensitivity}

The rate mapper normalizes each residual score against the current mini-batch before mapping it to a dropout probability, so the batch size $B$ determines how the relative unreliability signal is estimated. Table~\ref{tab:batch_size_sensitivity} sweeps $B\in\{8,16,32,64,128\}$ for PatchTST on \textit{ETTh1} and WPMixer on \textit{Synth-12} at $\sigma=0.5$, both at $H=96$. Both settings show a smooth interior optimum at conventional batch sizes: MSE is lowest at $B=32$ for PatchTST ($0.3833$) and at $B=64$ for WPMixer ($0.2781$), and changes little across the remaining values of $B$. MAE follows the same pattern, with its minimum at $B=128$ for PatchTST ($0.3889$) and at $B=64$ for WPMixer ($0.3880$). Over the five batch sizes, the relative MSE range is $2.42\%$ and $3.11\%$, with a coefficient of variation of $0.97\%$ and $1.24\%$, respectively. The batch-relative mapping is therefore stable across conventional batch sizes and does not require per-dataset tuning of $B$. The min--max guard also makes the mapper well defined at the boundary. When the batch is too small to define a range, the dispersion condition is not met and the normalized score defaults to the neutral prior $\hat{s}_i=0.5$, returning a moderate default dropout probability instead of an ill-defined normalized score. This behavior keeps gradients well behaved for very small $B$, and it is consistent with the interior optima above: moderate batches provide enough samples for a meaningful min--max range while retaining the benefit of adaptive rates.

\begin{table}[H]
    \centering
    \small
    \setlength{\tabcolsep}{4pt}
    \renewcommand{\arraystretch}{1.1}
    \caption{Batch-size sensitivity at $H=96$ for two representative backbones. Because the rate mapper normalizes residual scores within each mini-batch, we vary $B\in\{8,16,32,64,128\}$ and report test MSE/MAE; bold marks the lowest MSE in each setting. All runs within a setting share a fixed control configuration, so the comparison isolates the effect of $B$. The last two columns summarize the MSE sweep: relative range is $(\max_B \mathrm{MSE}-\min_B \mathrm{MSE})/\min_B \mathrm{MSE}$, and CV is the population standard deviation across the five batch sizes divided by their mean.}
    \label{tab:batch_size_sensitivity}
    \resizebox{\linewidth}{!}{%
    \begin{tabular}{ll ccccc cc}
        \toprule
        \rowcolor{TableHeader}
        \multirow{2}{*}{\textbf{Backbone / dataset}} & \multirow{2}{*}{\textbf{Metric}} &
        \multicolumn{5}{c}{\textbf{Batch size} $B$} &
        \multicolumn{2}{c}{\textbf{MSE variation}} \\
        \cmidrule(lr){3-7}\cmidrule(lr){8-9}
        \rowcolor{TableSubheader}
        & & 8 & 16 & 32 & 64 & 128 & Range (\%) & CV (\%) \\
        \midrule
        \multirow{2}{*}{\makecell[l]{\textbf{PatchTST}\\\textit{ETTh1}}} & MSE
        & 0.390715 & 0.391424 & \textbf{0.383258} & 0.392519 & 0.384708 & \multirow{2}{*}{2.42} & \multirow{2}{*}{0.97} \\
        & MAE & 0.392903 & 0.392274 & 0.390247 & 0.391016 & 0.388892 & & \\
        \midrule
        \multirow{2}{*}{\makecell[l]{\textbf{WPMixer}\\\textit{Synth-12} ($\sigma{=}0.5$)}} & MSE
        & 0.280752 & 0.285856 & 0.286771 & \textbf{0.278124} & 0.279373 & \multirow{2}{*}{3.11} & \multirow{2}{*}{1.24} \\
        & MAE & 0.390588 & 0.397123 & 0.394220 & 0.387984 & 0.388969 & & \\
        \bottomrule
    \end{tabular}%
    }
\end{table}

\section{Component-wise Spectral Diagnostics}
\label{appx:visualizations}

This section isolates canonical signal components and corruption types to probe the spectral scorer qualitatively. These diagnostic panels are separate from the layered-corruption \textit{Synth-12} benchmark and show what the scorer treats as concentrated structure versus residual under different non-stationarities.

\noindent\textbf{Scope of the diagnostic reconstructions.}
The visualizations use hard coefficient retention to expose where spectral energy is located. The diagnostic threshold is fixed for visualization, whereas \method learns channel-adaptive thresholds and uses the resulting residual to dynamically modulate active capacity via sample-wise retention probabilities. The model retains the original input in data space $\mathcal{X}$, and the panels stress-test spectral concentration qualitatively.

\subsection{Stationary Periodic Signals}
Figure~\ref{fig:spectral_analysis_periodic} shows stationary periodic signals with time-invariant frequency and amplitude. In a finite sampled window, their spectra contain narrow peaks near the fundamental frequencies (e.g., 0.5 and 2.0 Hz). Under the tested corruptions, these peaks remain distinguishable from much of the broadband residual. The hard-threshold reconstructions therefore recover the dominant cycles with limited phase distortion, while the residual absorbs most of the additive high-frequency energy. This setting is the easiest for spectral scoring: energy is already sparse, so residual magnitude provides an effective proxy for sample unreliability.

\begin{figure}[!ht]
    \centering
    \includegraphics[width=\linewidth]{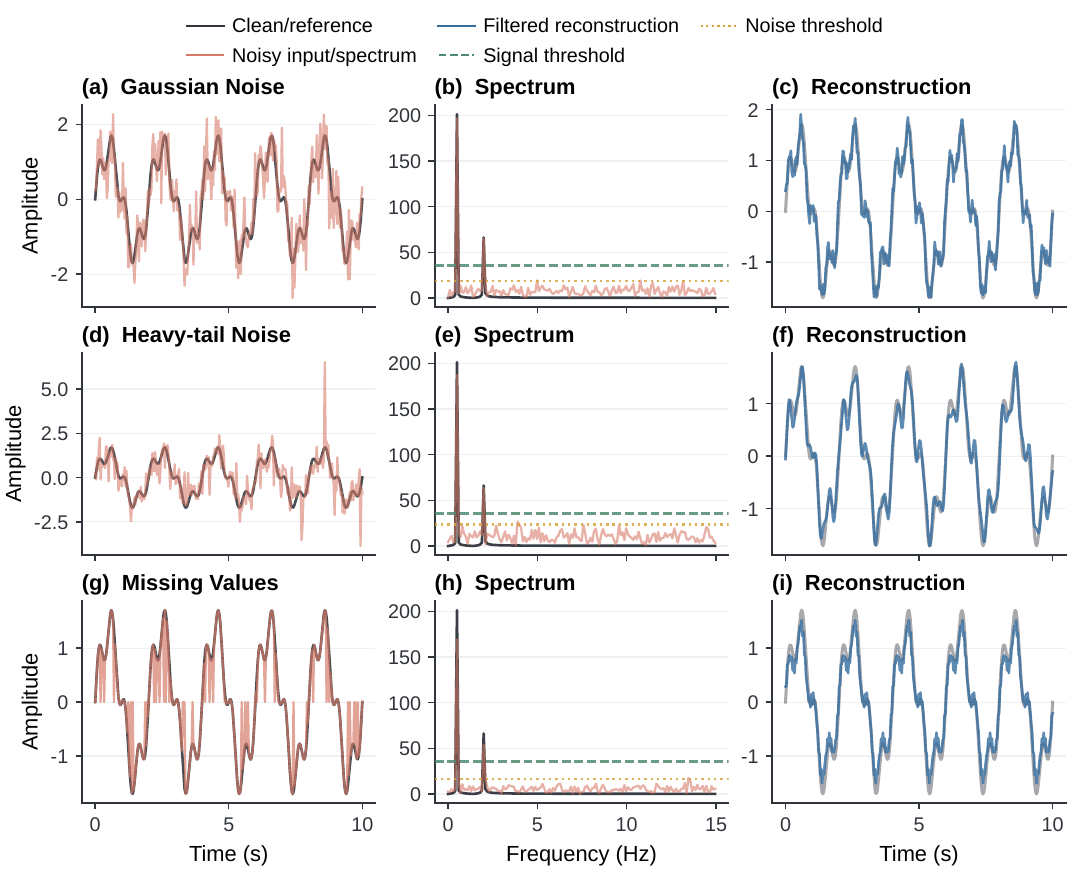}
    \caption{Stationary periodic signals under three isolated corruptions (Gaussian, heavy-tail, point-wise missing values). Columns: corrupted input, amplitude spectrum, and hard-threshold reconstruction. Narrow peaks remain separable from broadband residual energy, illustrating spectral separation in these examples.}
    \label{fig:spectral_analysis_periodic}
\end{figure}

\subsection{Non-stationary Mean (Trend \& Seasonality)}
Figure~\ref{fig:spectral_analysis_trend_seasonal} combines linear drift with seasonal oscillations. Before detrending, the drift concentrates energy near zero frequency and the seasonal component appears at harmonic peaks. Without this step, a hard retention threshold can either over-preserve the DC ramp or leak trend energy into neighboring bins after windowing. The scorer removes a global linear fit before spectral masking and restores it after reconstruction, reducing boundary leakage while retaining the main seasonal structure. The residual then better reflects corruption and residual non-trend fluctuation rather than the deterministic drift itself, which matches the role of OLS preprocessing in the main method.

\begin{figure}[!ht]
    \centering
    \includegraphics[width=\linewidth]{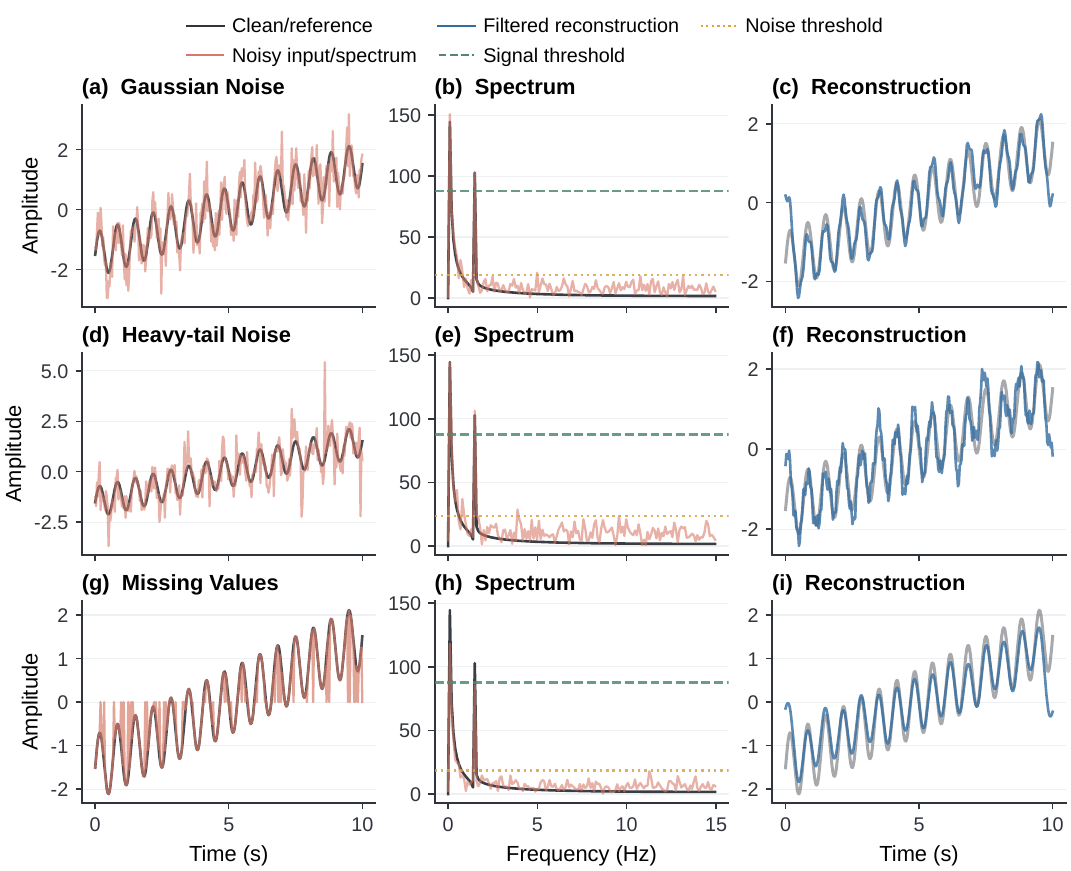}
    \caption{Trend-seasonal signals under three isolated corruptions. Columns: corrupted input, spectrum after global OLS detrending, and reconstruction with trend restoration. Detrending reduces DC leakage so the residual better reflects corruption rather than deterministic drift.}
    \label{fig:spectral_analysis_trend_seasonal}
\end{figure}

\subsection{Non-stationary Frequency (Chirp)}
Figure~\ref{fig:spectral_analysis_chirp} presents a linear chirp sweeping from 0.5 to 3.0 Hz. Its energy forms a structured band rather than isolated peaks, so a peak-only view of spectral sparsity is incomplete. In these examples, the applied threshold retains enough of the band to reconstruct the evolving oscillation and the instantaneous frequency trajectory remains visible after inversion. Corruptions raise the broadband floor around the band, which increases residual energy without fully erasing the chirp structure. This supports treating residual magnitude as a soft difficulty cue even when the latent signal is time-varying rather than purely tonal.

\begin{figure}[!ht]
    \centering
    \includegraphics[width=\linewidth]{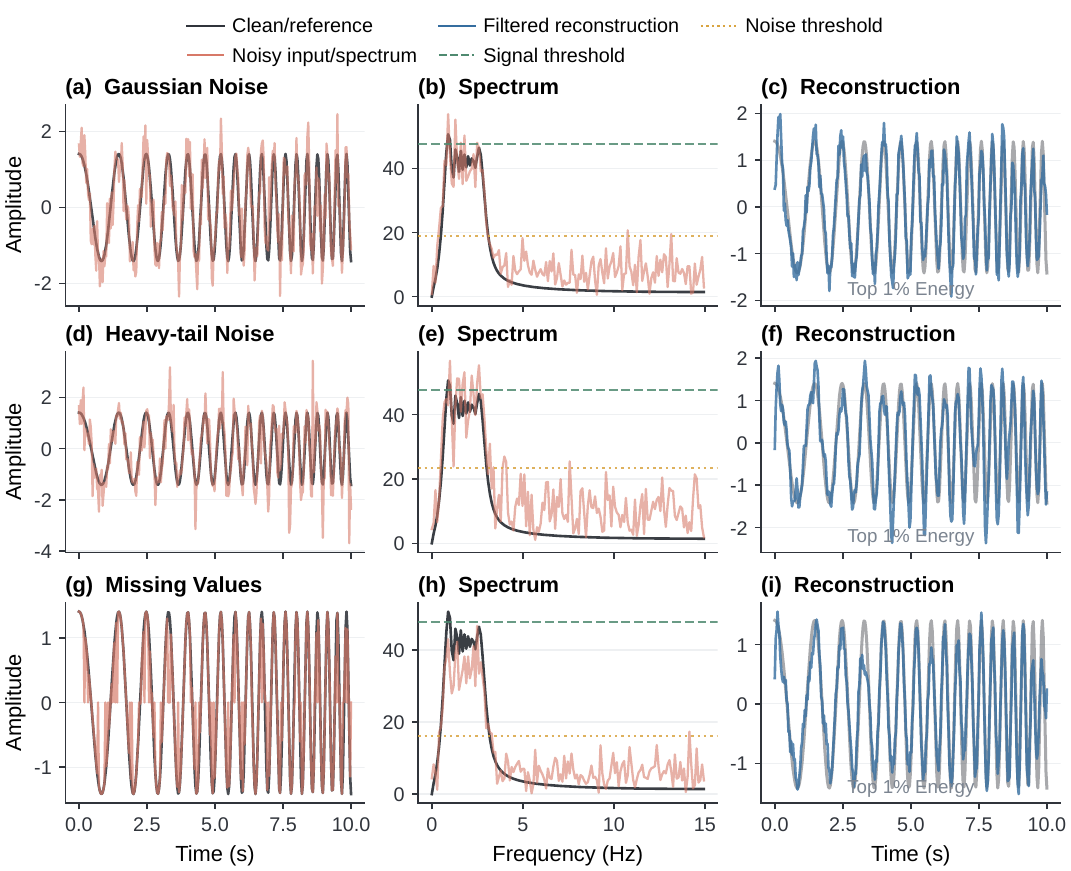}
    \caption{Linear chirps sweeping from 0.5 to 3.0 Hz under three isolated corruptions. Columns: corrupted input, amplitude spectrum, and hard-threshold reconstruction. Energy forms a structured band rather than isolated peaks; the retained band preserves the sweep while corruptions raise residual energy.}
    \label{fig:spectral_analysis_chirp}
\end{figure}

\subsection{Non-stationary Variance (Amplitude Modulation)}
Figure~\ref{fig:spectral_analysis_am} depicts amplitude-modulated signals, where a low-frequency envelope changes the carrier amplitude. The spectrum contains a carrier and modulation sidebands; the reconstructions preserve the principal amplitude variation in the tested cases. Because the envelope redistributes energy into nearby bins rather than destroying the carrier, hard retention can still recover the overarching modulation pattern while leaving fine-scale envelope noise in the residual. Residual sideband energy can include both corruption and low-amplitude modulation structure. This motivates retaining the original input for prediction while using the residual to guide regularization.

\begin{figure}[!ht]
    \centering
    \includegraphics[width=\linewidth]{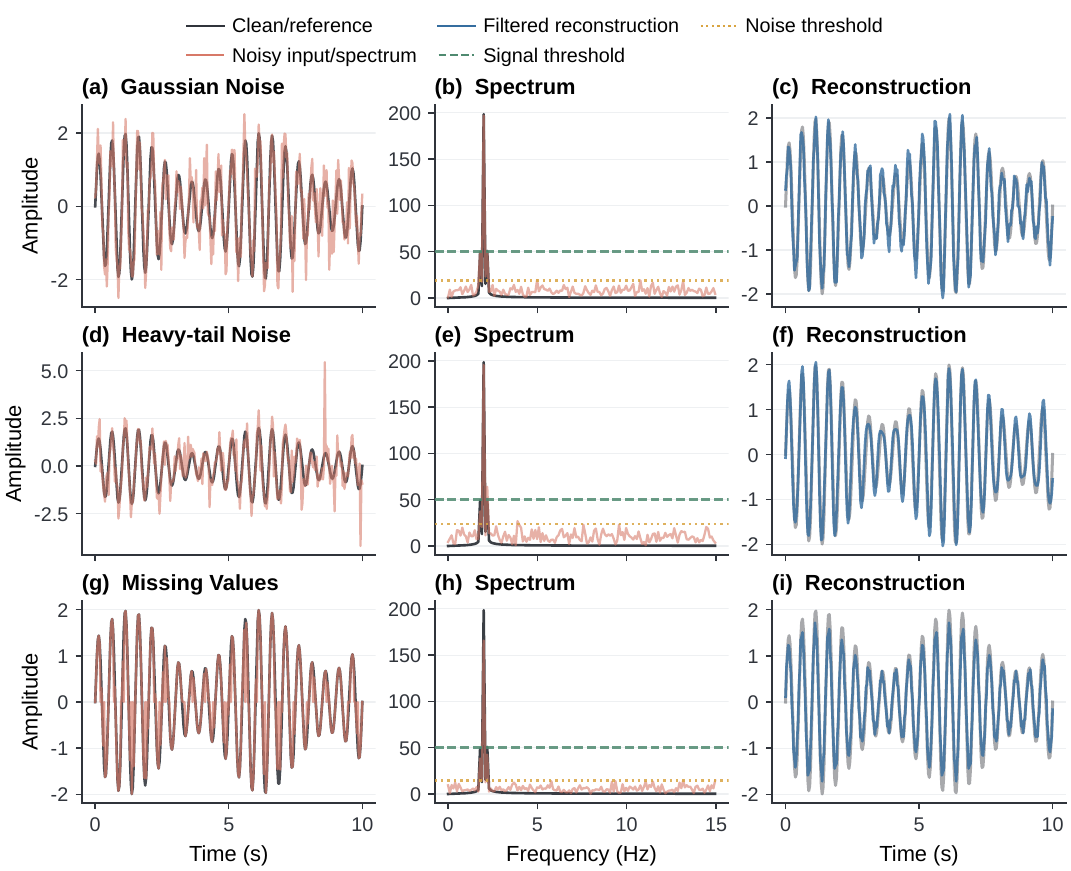}
    \caption{Amplitude-modulated signals (2.0 Hz carrier, 0.2 Hz envelope) under three corruptions. Columns: corrupted input, spectrum with carrier and sidebands, and hard-threshold reconstruction. The reconstruction retains the amplitude variation, while the residual can contain both corruption and modulation detail.}
    \label{fig:spectral_analysis_am}
\end{figure}

\subsection{Spectral Sparsity Analysis on Real-world Benchmarks}
\label{appx:real_world_sparsity}

Figure~\ref{fig:real_world_vis} retains the top 10\% of Fourier coefficients by amplitude on nine forecasting datasets and sets the rest to zero before inversion. Seven displayed correlations exceed 0.95, while \textit{Electricity} and \textit{ECG} have correlations of 0.868 and 0.769, respectively. Spectral concentration therefore varies across the examples: a small coefficient subset closely reconstructs some series but loses more structure on Electricity and ECG. The discarded coefficients can contain both nuisance fluctuations and useful signal, motivating the use of spectral residuals as a regularization cue while preserving the original input.

\begin{figure}[h]
    \centering
    \includegraphics[width=\linewidth]{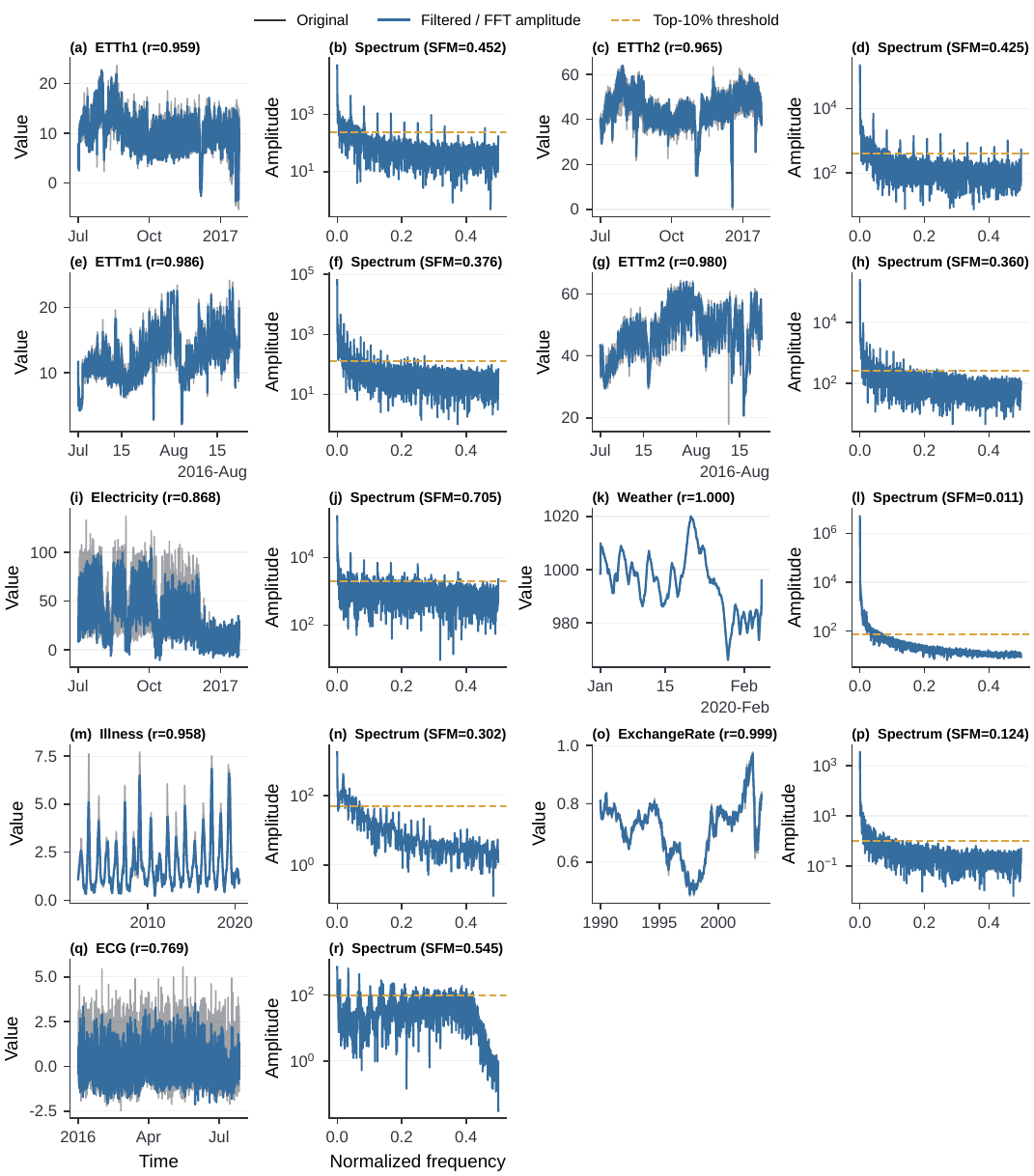}
    \caption{Top-10\% Fourier reconstruction on nine forecasting datasets. Each pair shows the original and reconstructed series with Pearson correlation, followed by the amplitude spectrum and retention threshold. Seven correlations exceed 0.95; Electricity and ECG yield 0.868 and 0.769. ``Illness'' denotes ILI. The hard mask is used for this diagnostic; \method uses a learned soft mask to compute a training-time regularization cue.}
    \label{fig:real_world_vis}
\end{figure}

\section{Impact of Spectral Leakage and Detrending}
\label{sec:detrend_ablation}

The discrete Fourier transform treats a finite observation window as one period of a periodically extended sequence. When a time series contains a trend, the join between the last time step $x_L$ and the first $x_1$ can therefore be discontinuous. This is a property of the finite-window DFT, not of the forecasting backbone itself, but it matters because \method scores residual energy in the spectral domain before modulating sample-wise active capacity.

This discontinuity can produce \textbf{spectral leakage}, spreading boundary-induced energy across the spectrum and introducing ringing artifacts in truncated spectral reconstruction. If left unaddressed, the spectral mask $\mathbf{M}$ may attenuate these leakage components, causing the reconstructed signal $\hat{\mathbf{x}}$ to exhibit large boundary errors. Edge-MAE can then reflect periodic-extension artifacts rather than stochastic corruption, which would inflate residual scores for well-structured samples and misallocate capacity. Global linear detrending reduces this confounder in the residual score, while residuals can still contain structure from non-stationary components beyond a linear trend. We compare three preprocessing strategies: (1) no detrending, (2) endpoint-based detrending, and (3) global linear OLS detrending across four stress tests (Figure~\ref{fig:detrend_analysis}). The panels isolate failure modes that appear when spectral scoring is applied to short windows with trends, outliers, or incomplete seasonal cycles.

\subsection{Theoretical Comparison}
\textbf{End-to-End Detrending} is a local approach that estimates the trend solely from the boundary values: $\hat{\mathbf{x}}_{trend}(t) = \mathbf{x}_1 + \frac{t-1}{L-1}(\mathbf{x}_L-\mathbf{x}_1)$ for $t\in\{1,\ldots,L\}$. While computationally trivial, it is highly sensitive to sensor noise at the endpoints ($t=1$ or $t=L$). Any spike, drop-out, or quantization error at either boundary fully determines the slope, so a single corrupted endpoint can dominate the residual everywhere.

\textbf{Global Linear Detrending (Ours)} estimates $\mathbf{w}^\ast,\mathbf{b}^\ast=\arg\min_{\mathbf{w},\mathbf{b}}\sum_{t=1}^L\|\mathbf{x}_t-(\mathbf{w}t+\mathbf{b})\|_2^2$ from all time points. Compared with endpoint interpolation, a single endpoint has less leverage, while ordinary least squares remains sensitive to outliers. In exchange, the fit is more stable under isolated boundary failures and better matches the role of detrending as a lightweight preprocessing step before soft spectral masking. The residual serves as a training-time difficulty cue while the backbone continues to use the original input.

\subsection{Analysis of Failure Modes}
We evaluate reconstruction fidelity near the sequence boundaries using \textbf{Edge-MAE}. Let $\mathcal{E}\subseteq\{1,\ldots,L\}$ denote the boundary indices highlighted in Figure~\ref{fig:detrend_analysis}. For $C$ channels,
\begin{equation}
    \operatorname{Edge\text{-}MAE}
    = \frac{1}{|\mathcal{E}|C}
    \sum_{t\in\mathcal{E}}\sum_{c=1}^{C}
    \left|x_{t,c}-\hat{x}_{t,c}\right|.
\end{equation}
Lower values indicate better boundary reconstruction. This diagnostic differs from the full-window residual score of \method: Edge-MAE stresses leakage at $t{=}1$ and $t{=}L$, whereas the training residual averages over the window after soft masking.

\textbf{Sensitivity to Outliers (Row 1).} 
In the ``Linear + Start Outlier'' scenario, a spike at $t=1$ directly determines the endpoint-interpolation slope. Global OLS distributes the fit over all points, mitigating outlier leverage without ignoring extreme observations. The practical consequence for capacity modulation is that endpoint detrending can produce an artificially inflated unreliability score from a single boundary outlier, whereas global OLS reduces the influence of that endpoint on the fitted trend.

\textbf{Phase Mismatch in Seasonality (Row 2).}
If the window does not contain an integer number of cycles, $x_1\neq x_L$ and the repeated-window interpretation introduces a boundary discontinuity. Both detrending variants reduce the resulting leakage in this example by removing the linear component of the join discontinuity before spectral truncation. Without detrending, residual energy near the edges can be dominated by periodic-extension artifacts even when the interior seasonal pattern is clean.

\textbf{Instability under Non-linearity (Row 3).}
For a quadratic trend, endpoint noise changes the interpolation slope. Global OLS gives the least-squares linear approximation, but no detrending has the lowest Edge-MAE in this row, showing that a linear trend model is not uniformly preferable under model mismatch. We retain global OLS as a default because it wins on the other three stress tests and is cheap to apply. Strongly curved trends may leave systematic residual structure that soft masking must absorb.

\textbf{Robustness to Sensor Failure (Row 4).}
A step-up regime shift is followed by a failed final observation. Endpoint interpolation connects the initial and failed final values and infers a nearly flat trend, erasing most of the high-state evidence. Global OLS uses the intervening high-state observations and produces a more suitable linear fit in this constructed case. The residual after global OLS therefore better reflects the failed endpoint rather than a systematically wrong slope across the whole window.

Taken together, the four rows favor global OLS as a default preprocessor for spectral residual scoring: it avoids endpoint leverage under outliers and failures (rows 1 and 4), reduces seasonal join discontinuities (row 2), and shows limitations under strong nonlinear trends where no linear model is optimal (row 3). Edge-MAE stresses boundary reconstruction rather than full-window forecasting error, so the figure diagnoses leakage-induced residual bias. In \method, the residual serves as a label-free unreliability proxy for capacity modulation; global OLS keeps this signal less confounded by boundary extension artifacts.

\begin{figure}[H]
    \centering
    \includegraphics[width=\linewidth]{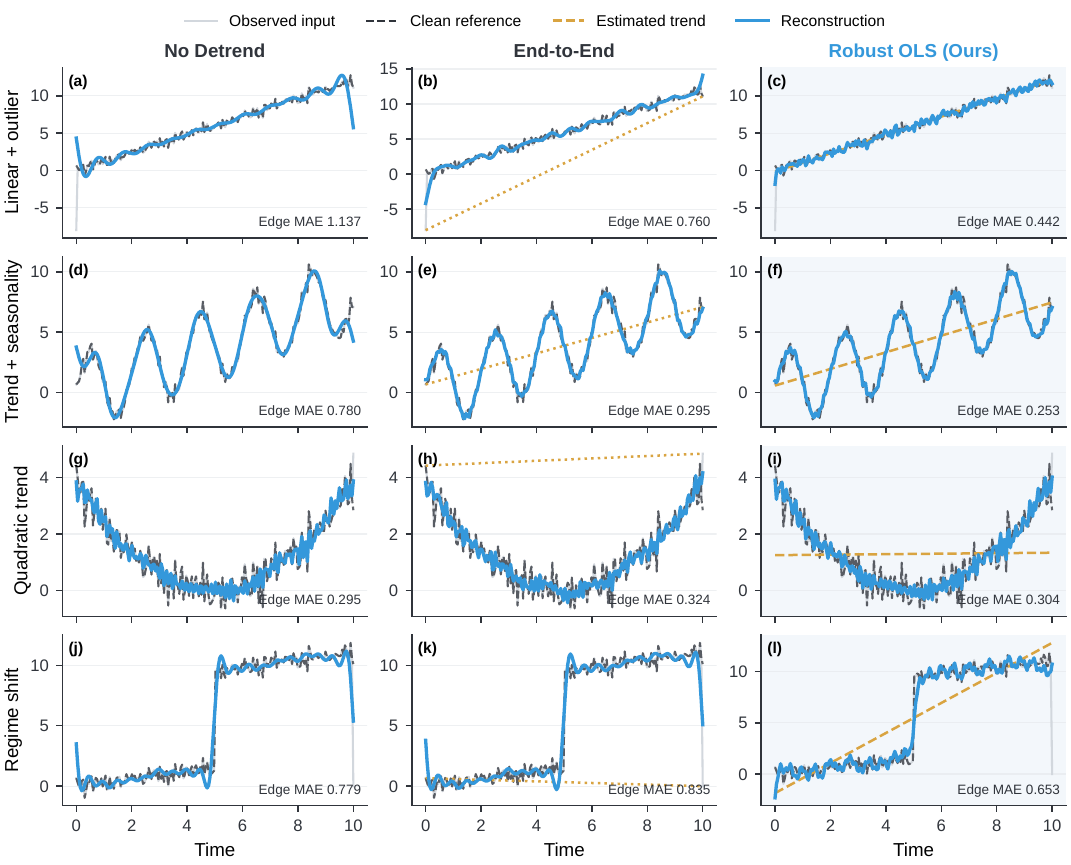}
    \caption{Comparison of no detrending, endpoint interpolation, and global OLS under four boundary stress tests, scored by reconstruction Edge-MAE near the window edges. Each row compares the three detrending strategies; panels show the observed input, clean reference, estimated trend, reconstruction, and Edge-MAE. \textbf{Row 1} (start-point outlier): endpoint interpolation inherits the spike as the full slope, while global OLS dilutes its leverage. \textbf{Row 2} (seasonal phase mismatch): $x_1\neq x_L$ induces a join discontinuity; both detrending variants reduce leakage relative to no detrending. \textbf{Row 3} (quadratic trend with endpoint noise): no detrending achieves the lowest Edge-MAE, showing that a linear trend model is not uniformly preferable under model mismatch. \textbf{Row 4} (regime shift then final-step sensor failure): endpoint interpolation flattens the high-state evidence, whereas global OLS recovers a more suitable linear fit. Global OLS wins rows 1, 2, and 4. The plot label ``Robust OLS refers to ordinary global OLS, not a robust-regression estimator.''}
    \label{fig:detrend_analysis}
\end{figure}

\section{Theory and Proofs}
\label{app:theory_proofs}

This appendix derives the local dropout approximation, characterizes the oracle allocation, and proves the excess-risk theorem for the stylized bias--variance surrogate introduced in the theory section.

\subsection{Local dropout approximation}
Condition on the input, mini-batch, and backbone parameters. Consider one hidden activation vector $\mathbf h$ and a mask $\mathbf B$ with independent Bernoulli entries of retention probability $q=1-p$. Set all other activation multipliers to one. In inverted dropout, $\widetilde{\mathbf h}=\mathbf B\odot\mathbf h/q$. Writing $\boldsymbol\delta=\widetilde{\mathbf h}-\mathbf h$, independence gives
\[
\mathbb E[\boldsymbol\delta]=\mathbf 0,
\qquad
\operatorname{Cov}(\boldsymbol\delta)=\frac{p}{q}\operatorname{diag}(h_1^2,\ldots,h_d^2).
\]
Let $\ell(\mathbf h)=\mathcal L(g_\theta(\mathbf h),\mathbf y)$, where $g_\theta$ is the downstream part of the backbone. Assume that $\ell$ is twice continuously differentiable on the line segments from $\mathbf h$ to every possible masked activation. A second-order Taylor expansion yields
\[
\mathbb E[\ell(\mathbf h+\boldsymbol\delta)]
 =\ell(\mathbf h)+\frac12\operatorname{tr}\!\left(\nabla^2_{\mathbf h}\ell(\mathbf h)\operatorname{Cov}(\boldsymbol\delta)\right)+R,
\]
where the first-order term vanishes and the exact remainder is
\[
R=\mathbb E\!\left[\int_0^1(1-t)\boldsymbol\delta^\top
\bigl(\nabla^2\ell(\mathbf h+t\boldsymbol\delta)-\nabla^2\ell(\mathbf h)\bigr)
\boldsymbol\delta\,\mathrm dt\right].
\]
If the Hessian is Lipschitz with constant $K_\ell$ along these segments, then
$|R|\le K_\ell\mathbb E[\|\boldsymbol\delta\|^3]/6$.
Thus a quadratic activation loss has $R=0$, while the approximation for a nonlinear loss depends on the size of this remainder. Defining
\[
\Omega(\mathbf x;\theta)=\sum_{j=1}^{d}\frac{\partial^2\ell}{\partial h_j^2}(\mathbf h)h_j^2
\]
and using $p/q=\lambda(p)$ gives
\[
\mathbb E_{\mathbf B}[\ell(\widetilde{\mathbf h})]
 =\ell(\mathbf h)+\frac12\lambda(p)\Omega(\mathbf x;\theta)+R,
\]
which is the dropout quadratic approximation stated in the main paper after collecting the higher-order terms in $R$. For a nonlinear loss, neglecting $R$ requires sufficiently small curvature variation over the mask perturbations. The quadratic term acts as a non-negative sensitivity penalty in regions where the relevant curvature term $\Omega(\mathbf x;\theta)$ is non-negative.

\subsection{Regularization mapping and oracle characterization}
For $p\in[0,1)$, $q=1-p$ and $\lambda=p/q$, so $q=(1+\lambda)^{-1}$ and
\[
\frac{\mathrm d\lambda}{\mathrm dp}=\frac{1}{(1-p)^2}>0.
\]
For the mapper in the adaptive-capacity subsection, letting $a=\operatorname{softplus}(\gamma)>0$ gives
\[
\frac{\mathrm dp}{\mathrm d\hat s}=(p_{\max}-p_{\min})a\,\operatorname{sech}^2(a\hat s)>0.
\]
For fixed mapper parameters, $p$ and $\lambda(p)$ are increasing in the normalized unreliability score within a mini-batch. The corresponding quadratic correction is a non-negative regularization term when $\Omega\ge0$. The bounds $[p_{\min},p_{\max}]$ define an allocation envelope; a mapper with finite $a$ attains a maximum rate $p_{\min}+(p_{\max}-p_{\min})\tanh(a)$ within that envelope. The expected number of retained activation coordinates is $\mathbb E[\sum_{j=1}^d B_j]=(1-p)d$ for a layer of width $d$. This quantity does not change the number of trainable parameters.

\subsection{Assumptions and oracle allocation}
For $C_1,C_2>0$ and unreliability level $\sigma\ge 0$, define
\[
\mathcal{E}(\lambda,\sigma)=C_1\lambda^2+C_2\sigma^2(1+\lambda)^{-2},\qquad \lambda\in\Lambda=[\lambda_{\min},\lambda_{\max}],
\]
where $0\le\lambda_{\min}<\lambda_{\max}<\infty$. The interval is induced by the dropout bounds,
\[
\lambda_{\min}=\frac{p_{\min}}{1-p_{\min}},\qquad
\lambda_{\max}=\frac{p_{\max}}{1-p_{\max}}.
\]
Assume $\mathbb{E}[\sigma(\mathbf X)^2]<\infty$ and set $g(\sigma)=\Pi_\Lambda(\lambda^\star(\sigma))$, where $\lambda^\star$ minimizes the surrogate over $[0,\infty)$. The surrogate coefficients $C_1$ and $C_2$ are fixed positive constants over the population; all sample heterogeneity enters through $\sigma(\mathbf X)$. More explicitly, assume a scalar prediction error $e_\lambda=b\lambda+\sqrt{C_2}\sigma\varepsilon/(1+\lambda)$ with $b^2=C_1$, $\mathbb E[\varepsilon\mid\sigma]=0$, and $\mathbb E[\varepsilon^2\mid\sigma]=1$. The conditional mean is $b\lambda$, and the conditional variance is $C_2\sigma^2/(1+\lambda)^2$. Their squared-mean-plus-variance decomposition gives the stated surrogate. The bias term assumes a linear response with zero baseline bias. For a smooth bias function $B$, the expansion $B(\lambda)=B(0)+B'(0)\lambda+o(\lambda)$ motivates this choice when $B(0)=0$ and $b=B'(0)\ne0$. The surrogate uses this first-order response throughout the comparison interval. The residual factor is obtained from minimizing $\frac12(v-\sigma\varepsilon)^2+\frac\lambda2v^2$, whose minimizer is $\sigma\varepsilon/(1+\lambda)$. This scalar penalty is a response assumption for the fitted predictor, rather than a consequence of the dropout Taylor expansion. The inverted-dropout multiplier itself has variance $\lambda$ and unit mean. The allocation theorem is conditional on these response assumptions and the common positive coefficients $C_1,C_2$.

\subsection{Proof of the main excess-risk theorem}
\begin{theorem}[Excess surrogate risk under reliability heterogeneity]
Let $\sigma(\mathbf X)\ge 0$ satisfy $\mathbb E[\sigma(\mathbf X)^2]<\infty$, let $\Lambda=[\lambda_{\min},\lambda_{\max}]$ be the feasible interval induced by $p\in[p_{\min},p_{\max}]$, and let $g(\sigma)=\Pi_\Lambda(\lambda^\star(\sigma))$ be the projected minimizer of $\mathcal E(\lambda,\sigma)$. If $\operatorname{Var}[g(\sigma(\mathbf X))]>0$, then for every fixed $\lambda_{\mathrm{fix}}\in\Lambda$,
\[
\mathbb E\!\left[\mathcal E(\lambda_{\mathrm{fix}},\sigma(\mathbf X))-\mathcal E(g(\sigma(\mathbf X)),\sigma(\mathbf X))\right]\ge C_1\operatorname{Var}[g(\sigma(\mathbf X))]>0.
\]
\end{theorem}

\begin{proof}
For fixed $\sigma$, differentiation gives
\[
\partial_\lambda\mathcal{E}=2C_1\lambda-2C_2\sigma^2(1+\lambda)^{-3},\qquad
\partial^2_{\lambda}\mathcal{E}=2C_1+6C_2\sigma^2(1+\lambda)^{-4}>0.
\]
Thus $\mathcal{E}(\cdot,\sigma)$ is strictly convex and has a unique minimizer on $\Lambda$, namely $g(\sigma)$. The stationarity equation for the minimizer over $[0,\infty)$ is $C_1\lambda^\star(1+\lambda^\star)^3=C_2\sigma^2$; its left-hand side is strictly increasing for $\lambda\ge0$, so $\lambda^\star(\sigma)$ and hence $g(\sigma)$ are nondecreasing in $\sigma$, with strict increase before projection clips the value at a boundary.

Let $m=\mathbb E[g(\sigma(\mathbf X))]$. Boundedness of $g$ and $\mathbb E[\sigma(\mathbf X)^2]<\infty$ ensure that the risks and squared allocation differences below have finite expectations. The first-order optimality condition on $\Lambda$ gives $\partial_\lambda\mathcal E(g(\sigma),\sigma)(\lambda-g(\sigma))\ge0$ for every $\lambda\in\Lambda$. Together with the uniform lower bound $\partial^2_{\lambda}\mathcal{E}\ge 2C_1$, this implies, for any fixed $\lambda_{\mathrm{fix}}\in\Lambda$,
\[
\mathcal{E}(\lambda_{\mathrm{fix}},\sigma)-\mathcal{E}(g(\sigma),\sigma)\ge C_1\bigl(\lambda_{\mathrm{fix}}-g(\sigma)\bigr)^2.
\]
Taking expectations and using $\mathbb E[(\lambda_{\mathrm{fix}}-g)^2]=\operatorname{Var}(g)+(\lambda_{\mathrm{fix}}-m)^2\ge\operatorname{Var}(g)>0$ proves the theorem inequality.
\end{proof}

The condition $\operatorname{Var}[g(\sigma(\mathbf X))]>0$ captures genuine allocation heterogeneity; when the projected oracle is constant, a uniform configuration coincides with the oracle. The theorem permits saturation at either endpoint of $\Lambda$.

\putbib[main]
\end{bibunit}

\end{document}